\documentclass[10pt]{article} 
\usepackage[preprint]{tmlr}

\usepackage{amsmath,amsfonts,bm}

\def\1{\bm{1}}

\DeclareMathAlphabet{\mathsfit}{\encodingdefault}{\sfdefault}{m}{sl}
\SetMathAlphabet{\mathsfit}{bold}{\encodingdefault}{\sfdefault}{bx}{n}

\DeclareMathOperator*{\argmin}{arg\,min}

\usepackage{hyperref}
\usepackage{url}

\usepackage{amssymb,amsthm,amsmath}

\newtheorem{theorem}{Theorem}
\newtheorem{lemma}{Lemma}

\newtheorem{assumption}{Assumption}
\newenvironment{proofsketch}{%
  \proof}{\endproof}

\newcommand{\rvar}[1]{\mathrm{#1}}        
\newcommand{\mat}[1]{{\boldsymbol #1}}    

\newcommand{\EE}[2][]{\mathbb{E}_{#1}\left[#2\right]} 

\newcommand{\A}{\mathcal{A}}
\newcommand{\B}{\mathcal{B}}

\newcommand{\E}{\mathcal{E}}

\newcommand{\M}{\mathcal{M}}

\renewcommand{\O}{\mathcal{O}} 

\renewcommand{\S}{\mathcal{S}}

\newcommand{\X}{\mathcal{X}}

\usepackage{array}
\usepackage{booktabs}
\newcolumntype{N}{>{\centering\arraybackslash}m{.6in}}
\newcolumntype{G}{>{\centering\arraybackslash}m{2in}}

\usepackage{pifont}

\usepackage{tikz}

\usepackage{multirow}

\title{Entropic Risk-Aware Monte Carlo Tree Search}

\author{\name Pedro P. Santos \email pedro.pinto.santos@tecnico.ulisboa.pt \\
      \name Jacopo Silvestrin \email jacopo.silvestrin@tecnico.ulisboa.pt \\
      \name Manuel Lopes \email manuel.lopes@tecnico.ulisboa.pt \\
      \name Francisco S. Melo \email fmelo@inesc-id.pt \\
      \addr Instituto Superior Técnico, University of Lisbon, and INESC-ID
      \AND
      \name Alberto Sardinha \email sardinha@inf.puc-rio.br\\
      \addr Pontifical Catholic University of Rio de Janeiro}

\def\month{MM}  
\def\year{YYYY} 
\def\openreview{\url{https://openreview.net/forum?id=XXXX}} 

\begin{document}

\maketitle

\begin{abstract}
We propose a provably correct Monte Carlo tree search (MCTS) algorithm for solving \textit{risk-aware} Markov decision processes (MDPs) with \textit{entropic risk measure} (ERM) objectives. We provide a \textit{non-asymptotic} analysis of our proposed algorithm, showing that the algorithm: (i) is \textit{correct} in the sense that the empirical ERM obtained at the root node converges to the optimal ERM; and (ii) enjoys \textit{polynomial regret concentration}. Our algorithm successfully exploits the dynamic programming formulations for solving risk-aware MDPs with ERM objectives introduced by previous works in the context of an upper confidence bound-based tree search algorithm. Finally, we provide a set of illustrative experiments to evaluate our risk-aware MCTS method, showing it successfully trades off between risk-neutral and risk-averse behavior.
\end{abstract}

\section{Introduction}
Monte Carlo Tree Search (MCTS) \citep{browne_2012,swiechowski_2022} is one of the most popular tree-based search algorithms for planning in large decision spaces. MCTS has been a cornerstone for solving complex sequential decision-making problems under the framework of Markov decision processes (MDPs), being successfully applied in different domains such as games \citep{silver_2017,silver_2018}, chemical synthesis \citep{segler_2018}, or planning and scheduling \citep{patra_2020}. MCTS iteratively builds a search tree that alternates between decision nodes, representing agent actions, and chance nodes, representing stochastic environment transitions. At each iteration, MCTS expands and refines the search tree by simulating a random trajectory, i.e., a random sequence of states and actions starting from the root node until a leaf node is reached. \citet{kocsis_2006,kocsis_2006_b} propose an action-selection algorithm at the decision nodes inspired by the upper confidence bound (UCB) algorithm \citep{auer_2002}, originally designed for multi-armed bandits \citep{lattimore_2020}. Previous works have analyzed different MCTS variants, providing convergence and concentration results on the performance of the algorithm \citep{kocsis_2006,shah_2020,comer_2025}.

MCTS aims to compute optimal \textit{risk-neutral} behavior, namely actions that minimize the \textit{expected} future sum of costs induced by interaction with the environment. However, optimizing for the expected behavior overlooks the variability in the random sum of costs \citep{shen_2014}: for a fixed policy (i.e., action-selection mechanism), the stochasticity of the environment leads to a \textit{distribution} over possible sums of costs. Therefore, optimizing the expectation of the induced distribution may fail to capture other important aspects of the policy's performance. A possible approach is to consider optimizing a specific \textit{risk measure} \citep{artzner_1999} of the induced distribution of cumulative costs. Risk measures allow the learned behavior to be tuned towards risk-averse or risk-seeking preferences. One of the most common risk measures is the \textit{entropic risk measure} (ERM) \citep{howard_matheson_1972}, which has received considerable attention in different domains such as finance and sequential decision-making \citep{follmer_2016,borkar_2002,pagnoncelli_2022,hau_2023,marthe2025efficientrisksensitiveplanningentropic,mortensen2025entropicriskoptimizationdiscounted}. The ERM enables trading off expected performance and risk aversion, allowing to compute optimal behavior that is robust to worst-case outcomes.

The study of sampling-based tree search algorithms for solving risk-aware MDPs with ERM objectives is of key importance. This is because such algorithms, like MCTS, are of particular interest in problems featuring large decision spaces, showing superior performance in comparison to dynamic programming approaches by focusing computation on the most important states, while only assuming access to a generative model of the environment. However, \citet{hayes_2023} are the only providing an MCTS-like algorithm capable of solving particular risk-aware MDPs. \citet{hayes_2023} propose a general MCTS-based algorithm that can be used to solve MDPs with non-linear utility functions. For particular choices of utility functions, the algorithm can be tuned towards more risk-seeking or risk-averse behaviors. In particular, as we show in Appendix~\ref{appendix:extended-related-work}, a specific choice of utility function can be used to solve risk-aware MDPs with ERM objectives, even though the authors do not provide results for such an utility function. Also, the authors do not provide any kind of theoretical guarantees on the performance of their proposed algorithm. We further elaborate on the connections between our work and that of \citet{hayes_2023} and others in Appendix~\ref{appendix:extended-related-work}.

In this work, we provide a novel and principled \textit{MCTS-based algorithm to solve risk-aware MDPs with ERM objectives}. We provide a \textit{non-asymptotic} analysis of our proposed algorithm, showing that the algorithm: (i) is \textit{correct} in the sense that the empirical ERM obtained at the root node converges to the optimal ERM; and (ii) enjoys \textit{polynomial regret concentration}. In contrast to \cite{hayes_2023}, our algorithm leverages the dynamic programming formulations of entropic risk-aware MDPs to design a UCB-based tree search algorithm with provable performance guarantees. Our algorithm mirrors risk-neutral MCTS by replacing the mean value estimators at decision nodes with entropic risk measure estimators, and by using UCB-style exploration bonuses to ensure sufficient exploration and convergence to optimal behavior. To establish our theoretical analysis, we borrow ideas from the results put forth by \citet{shah_2020} and \citet{comer_2025} in the context of risk-neutral MCTS. However, as discussed throughout our article, there are fundamental differences between our analysis and that of the aforementioned works, mainly because the mean and ERM estimators are inherently different. Hence, the analysis of our risk-aware MCTS algorithm does not follow trivially from the aforementioned articles, requiring novel proofs and additional results.

To establish our risk-aware MCTS algorithm and its respective theoretical analysis, we start by focusing our attention on solving non-stationary risk-aware bandits in Sec.~\ref{sec:bandits}. In Sec.~\ref{sec:bandits:non-stat-risk-aware-bandit}, we provide a UCB-based algorithm with polynomial bonuses to solve non-stationary risk-aware bandits. We show that the empirical ERM obtained by our proposed algorithm converges to the optimal ERM and provide polynomial concentration bounds on the deviation between the empirical ERM and the optimal ERM. Then, in Sec.~\ref{sec:bandits:non-stat-non-det-risk-aware-bandit}, we extend such results to the case of a non-deterministic non-stationary risk-aware bandit where action-selection triggers an immediate cost plus a second non-stationary cost associated with a random next state. In Sec.~\ref{sec:risk-aware-mcts}, we exploit the risk-aware bandits results from the previous section to provide a provably correct risk-aware MCTS algorithm for the ERM. We show that the results from Sec.~\ref{sec:bandits} can be recursively applied in the context of trees, hence showing that our risk-aware MCTS algorithms enjoys similar convergence and polynomial concentration properties as those we proved in the context of risk-aware bandits. Finally, in Sec.~\ref{sec:experiments}, we present illustrative experiments comparing our method against relevant baselines, including a bin-packing environment with an exponential number of states relative to the episode horizon.

\paragraph{Contributions}
As our key contribution, we propose a novel \textit{provably correct risk-aware MCTS algorithm} for the ERM, while providing a non-asymptotic analysis of its performance. We also provide novel convergence and concentration results for non-stationary risk-aware bandits.

\section{Background}

\subsection{Markov decision processes}
Discounted finite-horizon MDPs \citep{puterman2014markov} can be defined by the tuple $\M = \{\S, \A, \{\mat{P}^a : a \in \mathcal{A} \}, \{c_h : h \in \{0, \ldots, H\} \}, s_0, \gamma, H\}$ where: (i) $\mathcal{S}$ is the discrete state space; (ii) $\mathcal{A}$ is the discrete action space; (iii) $\{\mat{P}^a : a \in \mathcal{A} \}$ is a set of transition probability matrices $\mat{P}^a$ for each action $a \in \mathcal{A}$; (iv) $\{c_h : h \in \{0, \ldots, H\} \}$ is a set of bounded cost functions; (v) $s_0 \in \S$ is the initial state; (vi) $\gamma \in (0,1]$ is a discount factor; and (vii) $H \in \mathbb{N}$ is the horizon. Starting from the initial state $s_0 \in \S$, at each decision step $h \in \{0,\ldots, H-1\}$, the agent observes the state of the environment $\rvar{s}_h \in \mathcal{S}$ and chooses an action $\rvar{a}_h \in \mathcal{A}$. Depending on the action chosen by the agent, the MDP evolves to a new state $\rvar{s}_{h+1} \in \mathcal{S}$ with probability given by $P^{\rvar{a}_h}(\rvar{s}_h, \cdot)$, and the agent receives a bounded cost $c_h(\rvar{s}_h, \rvar{a}_h)$. The interaction repeats until a certain time horizon $H \in \mathbb{N}$ is reached, yielding a terminal bounded cost $c_H(\rvar{s}_H)$ to the agent. The framework of discounted finite-horizon MDPs is expressive enough to capture a variety of settings of interest. In particular, we highlight that setting $\gamma=1$ corresponds to the finite-horizon undiscounted setting. Also, for sufficiently large $H$, finite-horizon discounted MDPs can be used to compute approximate optimal policies for infinite-horizon discounted MDPs. We summarize the assumptions regarding our setting below, which are common among previous works \citep{kocsis_2006,shah_2020,comer_2025}. Generalization beyond fixed initial states follows straightforwardly from our analysis by extending the original MDP with an extra state to account for the initial random state.

\begin{assumption} \label{assumption:mdp_assumptions}
    The discounted finite-horizon MDP has discrete state $\S$ and action $\A$ spaces. The costs are bounded, i.e., $c_h(s, a) \in [-R,R]$ and $c_H(s) \in [-R,R] \; $ for any $s \in \S$ and $a \in \A$. The initial state $s_0 \in \S$ is fixed.
\end{assumption}

\noindent A decision rule $\pi_h$ specifies the action-selection mechanism at decision step $h$. In the context of finite-horizon MDPs, we are particularly interested in Markovian deterministic decision rules, which map the current state to a given action, i.e., $\pi_h : \mathcal{S} \rightarrow \A$. A policy $\pi = \{\pi_h\}_{h \in \{0, \ldots, H-1\}}$ is a sequence of decision rules, one for each decision step. It is known that the class of Markovian and deterministic policies, $\Pi_\text{M}^\text{D}$, suffices for optimality under finite-horizon MDPs \citep{puterman2014markov}.

For a given policy $\pi \in \Pi_\text{M}^\text{D}$, the interaction between the policy and the MDP induces a random sequence of states and actions $(\rvar{s}_0, \rvar{a}_0, \rvar{s}_1, \rvar{a}_1, \ldots, \rvar{s}_H)$. 
In the context of risk-neutral MDPs, the objective is to find a policy $\pi^*$ such that its induced expected sum of discounted costs is minimized,
\begin{equation}
    \pi^* \in \argmin_{\pi \in \Pi_\text{M}^\text{D}} \;\mathbb{E}\left[\sum_{h=0}^{H-1} \gamma^h c_h(\rvar{s}_h, \rvar{a}_h) + \gamma^H c_H(\rvar{s}_H)\right], \label{eq:optimal_policy_expected_objective}
\end{equation}
where the expectation is taken over the random sequence of states and actions induced by $\pi$.

\subsection{Risk-aware Markov decision processes}
In risk-aware MDPs \citep{NIPS2015_64223ccf,borkar_2002}, instead of aiming to find an optimal policy $\pi^*$ with respect to the expected discounted sum of costs, as defined in \eqref{eq:optimal_policy_expected_objective}, we aim to find an optimal policy with respect to a risk measure, $\rho$, of the discounted sum of costs,
\begin{equation}
    \pi^* \in \argmin_{\pi \in \Pi_\text{M}^\text{D}} \;\rho\left(\sum_{h=0}^{H-1} \gamma^h c_h(\rvar{s}_h, \rvar{a}_h) + \gamma^H c_H(\rvar{s}_H)\right).\label{eq:optimal_policy_risk_measure_objective}
\end{equation}
Risk measures \citep{artzner_1999} quantify an agent's attitude towards uncertainty and variability in outcomes, allowing for reasoning considering the entire distribution of discounted sum of costs. In this work, we focus our attention on the $\text{ERM}_\beta$ \citep{howard_matheson_1972}, defined for a random variable $\rvar{z}$ as
$
\text{ERM}_\beta(\rvar{z}) = \frac{1}{\beta} \ln\left( \mathbb{E}\left[\exp(\beta \rvar{z})\right]\right),
$
where parameter $\beta \in (0, \infty)$ trades off between expected performance (as $\beta \rightarrow 0$) and risk aversion (as $\beta$ increases). The $\text{ERM}_\beta$ belongs to the family of optimized certainty equivalents \citep{bental_1986}, being equivalently defined as
$\text{ERM}_\beta(\rvar{z}) = \min_{\lambda \in \mathbb{R}} \left\{ \lambda + \mathbb{E}\left[ u_\beta(\rvar{z} - \lambda) \right]\right\},$
with $u_\beta(x) = \frac{1}{\beta} \left( \exp(\beta x) - 1\right)$. In this work, we implicitly refer to the $\text{ERM}_\beta$ whenever we write $\rho$.

The $\text{ERM}_\beta$ is the only risk measure admitting a dynamic programming decomposition \citep{follmer_2016,marthe2024averagereturnmarkovdecision,hau_2023}. Furthermore, it is known that there always exists an optimal policy belonging to the $\Pi_\text{M}^\text{D}$ policy class \citep{hau_2023}. For any decision step $h \in \{0, \ldots, H\}$, we let $V^\pi_{h}(s) = \rho(\sum_{h'=h}^{H-1} c_{h'}(\rvar{s}_{h'}, \rvar{a}_{h'})  + c_H(\rvar{s}_H) | \rvar{s}_h = s)$ and $V^*_{h}(s) = \min_{\pi \in \Pi_\text{M}^\text{D}} V^\pi_h(s)$. \citet{hau_2023} show that the optimal value function $V^* = \{V^*_{h}\}_{h \in \{0, \ldots, H\}}$ and the optimal policy $\pi^* = \{\pi^*_h\}_{h \in \{0, \ldots, H-1\}}$ satisfy, for all $s \in \S$,
\begin{align}
    V_h^*(s) &= \min_{a \in \A}\left\{\text{ERM}_{\beta \gamma^h}\left( c_h(s,a) + \gamma V_{h+1}^*(\rvar{s}')\right) \right\}, \; \forall h \in \{0, \ldots, H-1\}, \; V_{H}^*(s) = c_H(s) \label{eq:erm_bellman_optimality equations}, \\
    \pi^*_h(s) &= \argmin_{a \in \A}\left\{\text{ERM}_{\beta \gamma^h}\left( c_h(s,a) + \gamma V_{h+1}^*(\rvar{s}')\right) \right\}, \; \forall h \in \{0, \ldots, H-1\}.
\end{align}

\section{Risk-aware multi-armed bandits} \label{sec:bandits}
In this section, we propose and investigate UCB-based algorithms for non-stationary risk-aware multi-armed bandits. In Sec.~\ref{sec:bandits:non-stat-risk-aware-bandit}, we study the case of risk-aware multi-armed bandits and then, in Sec.~\ref{sec:bandits:non-stat-non-det-risk-aware-bandit}, we extend our results to the case of non-deterministic multi-armed bandits where action-selection triggers an immediate cost plus a second non-stationary cost associated with a random next state. The results in this section form the foundation for the subsequent analysis of a provably correct MCTS algorithm for risk-aware MDPs.

\subsection{Non-stationary risk-aware multi-armed bandits}
\label{sec:bandits:non-stat-risk-aware-bandit}
Consider a bandit with $K$ arms, where each arm $i \in \{1, \ldots, K\}$ is associated with a random cost $\rvar{x}_i$ with support in $[-R,R]$. For each arm $i$, we let $\rho_i = \frac{1}{\beta} \ln\left( \mathbb{E}[ \exp(\beta \rvar{x}_{i})]\right)$ be the $\text{ERM}_\beta$ associated with arm $i$.
For a given arm $i$, we consider the following estimator for the $\text{ERM}_\beta$
\begin{equation}
    \hat{\rho}_{i,n} = \frac{1}{\beta} \ln\left(\frac{1}{n} \sum_{t=1}^n \exp(\beta \rvar{x}_{i,t})\right), \label{eq:entropic_risk_estimator}
\end{equation}
where $(\rvar{x}_{i,1}, \ldots, \rvar{x}_{i,n})$ is a sequence of $n$ random costs from arm $i$. We refer to Appendix~\ref{appendix:ERM-estimator-properties} for additional properties of the $\text{ERM}_\beta$ estimator defined above. In particular, we show in Appendix~\ref{appendix:ERM-estimator-properties} that this estimator belongs to the family of optimized certainty equivalents (OCEs) \citep{ben-tal-2007} and therefore admits an equivalent formulation that we exploit in our proofs.

We let $\mu_{i,n} = \mathbb{E}[\hat{\rho}_{i,n}]$. We now allow the distributions associated with each of the random variables $\rvar{x}_i$ to shift as a function of time. We make the following assumption.

\begin{assumption}
    \label{assumption:non-stationary-bandit-assumption}
    For any arm $i \in \{1, \ldots, K\}$, it holds that:
    \begin{enumerate}
        \item the limit $\lim_{n \rightarrow \infty} \mu_{i,n}$ exists and equals $\mu_i$.
        \item there exist constants $\theta > 1$, $\xi > 1 $, and $1/2 \le \eta < 1$ such that, for any $z \ge 1$ and $n \in \mathbb{N}$,
        \begin{equation}
        \mathbb{P}\left[n\hat{\rho}_{i,n} - n\mu_i \ge n^{\eta} z \right] \le \frac{\theta}{z^\xi}, \quad \text{and} \quad \mathbb{P}\left[n\hat{\rho}_{i,n} - n\mu_i \le - n^{\eta} z \right] \le \frac{\theta}{z^\xi}. \label{eq:bandit_concentration_assumption}
        \end{equation}
    \end{enumerate}
\end{assumption}

\noindent Essentially, the assumption above requires that: (i) the limit (in terms of $n$) of the expectation of the estimator $\hat{\rho}_{i,n}$ exists; and (ii) the estimator satisfies the polynomial concentration bounds \eqref{eq:bandit_concentration_assumption}. If no drift in the underlying distributions exists then we show in Appendix~\ref{appendix:ERM-estimator-properties} that (i) and (ii) are verified. Hence, Assumption \ref{assumption:non-stationary-bandit-assumption} states that the concentration bounds remain valid under the underlying shift in the distributions of the arms.

\paragraph{A UCB-based algorithm for the risk-aware non-stationary bandit}
Let $\rvar{a}_t \in \{1, \ldots, K\}$ denote the random variable encoding the arm selected at timestep $t$. We denote with $T_i(n) = \sum_{t=1}^{n} \mathbf{1}(\rvar{a}_{t}= i)$ the number of times arm $i$ was selected up to (including) timestep $n$. A UCB-based algorithm for the $\text{ERM}_\beta$ starts by selecting each of the arms once and then selects, at each timestep $t$, the arm with the best lower confidence bound, i.e.,
\begin{equation}
    \rvar{a}_t = \argmin_{i \in \{1,\ldots,K\}} \left\{ \hat{\rho}_{i,T_i(t)} - b_{t,T_i(t)} \right\}, \quad  \quad b_{t,s} = \frac{\theta^{1/\xi} t^{\alpha/\xi}}{s^{1-\eta}}, \label{eq:ucb_erm_algorithm}
\end{equation}
where $\alpha > 0$ is an hyperparameter that controls the exploration-exploitation trade-off. Without a loss of generality we assume there exists a unique optimal arm, denoted with $i^*$. We let $\mu^* = \min_{i \in \{1, \ldots, K\}} \mu_i$ be the limiting value of the expectation of the estimator for the optimal arm. Finally, let $\Delta_i = \mu_i - \mu^*$ and $\Delta_\text{min} = \min_{i \in \{1, \ldots, K\} \setminus i^* } \Delta_i$ be the gap between the optimal arm and the second-best arm. We state the following result (proof in Appendix~\ref{appendix:proof_theo_non_stat_bandit}).

\begin{theorem} \label{theo:non-stationary-bandit-results}
    Let $\bar{\rho}_n = \frac{1}{n} \sum_{i=1}^K T_i(n) \hat{\rho}_{i, T_i(n)}$ be the empirical $\text{ERM}_\beta$ when running algorithm \eqref{eq:ucb_erm_algorithm} under a non-stationary bandit satisfying Assumption~\ref{assumption:non-stationary-bandit-assumption}. Furthermore, for any $1/2 \le \eta < 1$, assume Assumption~\ref{assumption:non-stationary-bandit-assumption} holds with a sufficiently large $\xi > 1$ value such that there exists $\alpha > 2$ satisfying $\xi\eta(1-\eta) \le \alpha < \xi (1-\eta)$. Then,
    \begin{enumerate}
        \item it holds that $|\mu^* - \EE[]{\bar{\rho}_n}| = \O\left(\exp(2\beta R) n^{\frac{\alpha}{\xi(1-\eta)} -1}\right)$, i.e., $\lim_{n \rightarrow \infty} \EE[]{\bar{\rho}_n} = \mu^*$.
        \item there exist constants $\theta' >1$, $\xi' > 1$, and $1/2 \le \eta' < 1$ such that, for every $n \in \mathbb{N}$ and every $z \ge 1$, it holds that
        \begin{equation}
            \mathbb{P}\left[n\bar{\rho}_n - n\mu^* \ge n^{\eta'} z \right] \le \frac{\theta'}{z^{\xi'}}, \quad \text{ and } \quad \mathbb{P}\left[n\bar{\rho}_n - n\mu^* \le - n^{\eta'} z \right] \le \frac{\theta'}{z^{\xi'}},
        \end{equation}
        where $\eta' = \frac{\alpha}{\xi(1-\eta)}$, $\xi' = \alpha - 1$, and $\theta'$ depends on $R$, $K$, $\Delta_\text{min}$, $\beta$, $\theta$, $\xi$, $\alpha$, and $\eta$.
    \end{enumerate}
\end{theorem}

\noindent The result above shows that the expected empirical $\text{ERM}_\beta$ of algorithm \eqref{eq:ucb_erm_algorithm} satisfies a convergence rate of $\O\left(\exp(2\beta R) n^{\frac{\alpha}{\xi(1-\eta)} -1}\right)$. By letting $\eta=1/2$ and $\alpha = \xi\eta(1-\eta)$ the convergence rate simplifies to $\O(\exp(2\beta R) n^{-1/2})$. In such a case, the bias term in \eqref{eq:ucb_erm_algorithm} is of the form $\sqrt{\sqrt{t}/s}$. Our result recovers the convergence rate put forth by \cite{shah_2020} in the context of risk-neutral non-stationary bandits if we let $\beta \rightarrow 0$. As $\beta$ increases and we focus on the worst-case outcomes the convergence rate deteriorates. The exponential dependence on $\beta$ is similar to the result obtained by \citet[Theo. 1]{mortensen2025entropicriskoptimizationdiscounted} in the context of model-based risk-sensitive Q-value iteration. Finally, the assumption that there exists a sufficiently large $\xi > 1$ is not restrictive. In fact, if there exists no drift in the underlying distributions then we show in Appendix~\ref{appendix:ERM-estimator-properties} that any $\xi > 1$ satisfies Assumption~\ref{assumption:non-stationary-bandit-assumption} and, hence, we are free to choose a sufficiently large $\xi$ value.

However, it should be noted that the result above provides convergence and concentration results for $\bar{\rho}_n$, i.e., the empirical ERM when running algorithm  \eqref{eq:ucb_erm_algorithm} under a non-stationary bandit. While the result above is informative in the sense that it shows convergence of the empirical ERM to the optimal limiting value $\mu^*$, due to the non-linear nature of the $\text{ERM}_\beta$, it does not directly follow that the $\text{ERM}_\beta$ of the stream of costs obtained at the root node of the bandit also converges to $\mu^*$. To better understand this issue, let us introduce the following estimator, which computes the $\text{ERM}_\beta$ for the stream of costs obtained at the root of the bandit:
\begin{equation} \label{eq:entropic_risk_estimator_stream_estimator}
    \hat{\rho}_n^\text{stream} = \frac{1}{\beta} \ln\left( \frac{1}{n} \sum_{i=1}^K \sum_{t=1}^{T_i(n)} \exp(\beta \rvar{x}_{i,t})\right) = \frac{1}{\beta} \ln\left( \frac{1}{n} \sum_{i=1}^K T_i(n) \exp(\beta \hat{\rho}_{i, T_i(n)})\right),
\end{equation}
where the second equality above holds from algebraic manipulation and the definition of $\hat{\rho}_{i, n}$. Due to the non-linearity introduced by the $\ln$ and $\exp$ operators, $\hat{\rho}_n^\text{stream}$ differs in general from the empirical $\text{ERM}_\beta$, $\bar{\rho}_n$, considered in Theo.~\ref{theo:non-stationary-bandit-results}. This is in contrast to risk-neutral settings, as considered by previous works such as \cite{kocsis_2006,shah_2020,comer_2025}, where the empirical mean under the UCB algorithm coincides with the empirical mean of the stream of costs at the root node of the bandit. Nevertheless, in the following result (proof in Appendix~\ref{appendix:non-stationary-bandit-results-stream-estimator}), we build on Theo.~\ref{theo:non-stationary-bandit-results} and show that similar convergence and concentration properties hold for $\hat{\rho}_n^\text{stream}$.

\begin{theorem} \label{theo:non-stationary-bandit-results-stream-estimator}
    Let $\hat{\rho}_n^\text{stream}$, as defined in \eqref{eq:entropic_risk_estimator_stream_estimator}, be the empirical $\text{ERM}_\beta$ of the stream of costs obtained when running algorithm \eqref{eq:ucb_erm_algorithm} under a non-stationary bandit satisfying Assumption~\ref{assumption:non-stationary-bandit-assumption}. Furthermore, for any $1/2 \le \eta < 1$, assume Assumption~\ref{assumption:non-stationary-bandit-assumption} holds with a sufficiently large $\xi > 1$ value such that there exists $\alpha > 2$ satisfying $\xi\eta(1-\eta) \le \alpha < \xi (1-\eta)$. Then,
    \begin{enumerate}
        \item there exist constants $\theta'' >1$, $\xi'' > 1$, and $1/2 \le \eta'' < 1$ such that, for every $n \in \mathbb{N}$ and every $z \ge 1$, it holds that
        \begin{equation}
            \mathbb{P}\left[n\hat{\rho}_n^\text{stream} - n\mu^* \ge n^{\eta''} z \right] \le \frac{\theta''}{z^{\xi''}}, \quad \text{ and } \quad \mathbb{P}\left[n\hat{\rho}_n^\text{stream} - n\mu^* \le - n^{\eta''} z \right] \le \frac{\theta''}{z^{\xi''}},
        \end{equation}
        where $\eta'' = \frac{\alpha}{\xi (1-\eta)}$, $\xi'' = \alpha - 1$, and $\theta''$ depends on $R$, $K$, $\Delta_\text{min}$, $\beta$, $\theta$, $\theta'$, $\xi$, $\alpha$, and $\eta$.
        \item it holds that $|\mu^* - \hat{\rho}_n^\text{stream}| = \O\left(\theta'' n^{\frac{\alpha}{\xi(1-\eta)} -1}\right)$, i.e., $\lim_{n \rightarrow \infty} \mathbb{E}[\hat{\rho}_n^\text{stream}] = \mu^*$.
    \end{enumerate}
\end{theorem}

\noindent Thus, $\hat{\rho}_n^\text{stream}$ satisfies convergence and concentration properties similar to $\bar{\rho}_n$. Also, the convergence rate is similar (up to constants) to that of $\bar{\rho}_n$, and it is optimal by letting $\eta=1/2$ and $\alpha = \xi\eta(1-\eta)$, in which case the bonus is of the form $\sqrt{\sqrt{t}/s}$. Naturally, the convergence rate also depends on parameter $\beta$ since $\theta''$ depends on $\beta$. As in the case of $\bar{\rho}_n$, larger $\beta$ values deteriorate the convergence rate. We refer to Appendix~\ref{appendix:non-stationary-bandit-results-stream-estimator} for an explicit definition of $\theta''$, which highlights its dependence on $\beta$.

\subsection{Non-stationary non-deterministic risk-aware multi-armed bandits}
\label{sec:bandits:non-stat-non-det-risk-aware-bandit}
Consider a non-deterministic bandit with $K$ arms. After selecting a given arm $i \in \{1, \ldots, K\}$, a deterministic cost $c_i \in [-R_c,R_c]$ is obtained and a random next state $\rvar{s}' \in \S$ is sampled according to $P^i(\cdot)$, where $\sum_{s \in \S} P^i(s) = 1$. Let $\rvar{x}_{i}^{\rvar{s}'} \in [-R_\rvar{x},R_\rvar{x}]$ be the random cost obtained after choosing arm $i$ and obtaining the random next state $\rvar{s}'$. To simplify our notation, we let $R = R_c + R_\rvar{x}$, holding that $c_i \in [-R,R]$ and $\rvar{x}_{i}^{\rvar{s}'} \in [-R,R]$. Let
$$\rho_i = \frac{1}{\beta} \ln \left( \mathbb{E}_{\rvar{s}' \sim P^i, \rvar{x}_i^{\rvar{s}'}}[\exp(\beta ( c_i + \gamma \rvar{x}_i^{\rvar{s}'}))] \right) = \frac{1}{\beta} \ln \left( \sum_{s' \in \S} P^i(s') \mathbb{E}_{\rvar{x}_i^{s'}}[\exp(\beta (c_i + \gamma \rvar{x}_i^{s'}))] \right)$$
be the $\text{ERM}_\beta$ associated with arm $i$ under the non-deterministic bandit, where $\gamma \in (0,1]$.

We allow the sequence of costs $\rvar{x}_{i,t}^{s'}$ to be non-stationary over time. Let $\hat{\rho}_{i,n}^{s'}$ be the empirical $\text{ERM}_{\beta\gamma}$ estimator for the sequence of costs $\rvar{x}_{i,t}^{s'}$ associated with arm $i$ and next state $s'$, defined as $\hat{\rho}_{i,n}^{s'} = \frac{1}{\beta\gamma} \ln \left(\frac{1}{n}\sum_{t=1}^n \exp(\beta\gamma \rvar{x}_{i,t}^{s'}) \right)$. We highlight that we consider $\text{ERM}_{\beta\gamma}$ and not $\text{ERM}_{\beta}$ since the costs are obtained at the next timestep. Let $\mu_{i,n}^{s'} = \mathbb{E}[\hat{\rho}_{i,n}^{s'}]$. We make the following assumption.

\begin{assumption}
    \label{assumption:non-deterministic-bandit-assumption}
    For any arm $i \in \{1, \ldots, K\}$ and next state $s' \in \S$ it holds that
    \begin{enumerate}
        \item the limit $\lim_{n \rightarrow \infty} \mu_{i,n}^{s'}$ exists and equals $\mu_{i}^{s'}$.
        \item there exist constants $\theta > 1$, $\xi > 1 $, and $1/2 \le \eta < 1$ such that, for any $z \ge 1$ and $n \in \mathbb{N}$,
        \begin{equation}
        \mathbb{P}\left[n\hat{\rho}_{i,n}^{s'} - n\mu_i^{s'} \ge n^{\eta} z \right] \le \frac{\theta}{z^\xi}, \quad \text{and} \quad \mathbb{P}\left[n\hat{\rho}_{i,n}^{s'} - n\mu_i^{s'} \le - n^{\eta} z \right] \le \frac{\theta}{z^\xi}.
        \end{equation} 
    \end{enumerate}
\end{assumption}

\noindent For any given arm-selection algorithm, let $T_{i}(n)$ be the random variable encoding the number of times arm $i$ has been selected up to timestep $n$ and $T_{i}^{s'}(n)$ be the random variable encoding the number of times state $s'$ was sampled when arm $i$ has been selected $n$ times. For any arm $i$, let
\begin{equation}
    \hat{\rho}_{i,n} = \frac{1}{\beta} \ln \left( \frac{1}{n} \sum_{s' \in \S} \sum_{t=1}^{T_i^{s'}(n)} \exp(\beta ( c_i + \gamma \rvar{x}_{i,t}^{s'}) ) \right) \label{eq:entropic_risk_estimator_non_det_bandit}
\end{equation}
be the empirical $\text{ERM}_\beta$ associated with arm $i$ after it has been selected $n$ times. Intuitively, the estimator above defined calculates the empirical $\text{ERM}_\beta$ of the sequence of costs generated from arm $i$, where as opposed to Sec.~\ref{sec:bandits:non-stat-risk-aware-bandit}, now the estimator depends on the random number of times each possible next state $s' \in \S$ was sampled up to timestep $n$, $T_i^{s'}(n)$. However, we note that, if we redefine $\rvar{x}_{i,t}^{s'}$ to correspond to the cost sampled from next state $s'$ at the (global) timestep $t$, we can equivalently rewrite the estimator above as $\hat{\rho}_{i,n} = \frac{1}{\beta} \ln \left( \frac{1}{n} \sum_{t=1}^{n} \exp\left(\beta \rvar{z}_{i,t}\right) \right)$ where $\rvar{z}_{i,t} = c_i + \gamma \rvar{x}_{i,t}^{s_i'(t)}$, satisfying $\rvar{z}_{i,t} \in [-R,R]$, and $s'_i(t)$ is a random function mapping the timestep $t$ to a random next state $\rvar{s}' \in \S$. Thus, it can be seen that the estimator above has the same form as estimator \eqref{eq:entropic_risk_estimator}, where random variables $\rvar{z}_{i,t}$ play the role of random variables $\rvar{x}_{i,t}$ in \eqref{eq:entropic_risk_estimator}. The key difference is that here random variables $\rvar{z}_{i,t}$ are associated with a random next state; this is precisely the setting we consider in Sec.~\ref{sec:bandits:non-stat-risk-aware-bandit} since we allow the random variables $\rvar{x}_{i,t}$ to be non-stationary.

By definition, let $\mu_{i,n} = \mathbb{E}\left[\hat{\rho}_{i,n}\right]$. We state the following result (proof in Appendix~\ref{appendix:proof-non-deterministic-bandit-convergence-and-concentration}).

\begin{lemma}
    \label{lemma:non-deterministic-bandit-convergence-and-concentration}
    For any arm $i \in \{1, \ldots, K\}$, let
    $$\mu_i = \text{ERM}_\beta(c_i + \gamma \mu_i^{\rvar{s}'}) = \frac{1}{\beta} \ln \left(\mathbb{E}_{\rvar{s}' \sim P^i(\cdot)} \left[ \exp(\beta (c_i + \gamma \mu_i^{\rvar{s}'}) ) \right] \right).$$
    Then, under Assumption~\ref{assumption:non-deterministic-bandit-assumption}, it holds for any arm $i$ that
    \begin{enumerate}
        \item $\lim_{n \rightarrow \infty} \mu_{i,n} = \mu_i$.
        \item there exist constants $\theta^L > 1$, $\xi^L > 1 $, and $1/2 \le \eta^L < 1$ such that, for any $z \ge 1$ and $n \in \mathbb{N}$,
        \begin{equation}
        \mathbb{P}\left[n\hat{\rho}_{i,n} - n\mu_i \ge n^{\eta^L} z \right] \le \frac{\theta^L}{z^{\xi^L}}, \quad \text{and} \quad
        \mathbb{P}\left[n\hat{\rho}_{i,n} - n\mu_i \le - n^{\eta^L} z \right] \le \frac{\theta^L}{z^{\xi^L}},
        \end{equation}
        where $\eta^L = \eta$, $\xi^L = \xi$ and $\theta^L$ depends on $\xi$, $\theta$, $R$, $\beta$ and $|\S|$.
    \end{enumerate}
\end{lemma}

\noindent An important conclusion that can be derived from the Lemma above is that, in the context of a non-stationary non-deterministic bandit, estimator \eqref{eq:entropic_risk_estimator_non_det_bandit} satisfies Assumption~\ref{assumption:non-stationary-bandit-assumption}. Furthermore, we have that, for each arm $i$, the expectation of the estimator, $\mu_{i,n}$, converges to $\mu_i$, which corresponds to the true $\text{ERM}_\beta$  value due to the stochastic transitions for the limiting expected estimator values at the next states, $\mu^{s'}_i$. Since, as discussed above, estimator \eqref{eq:entropic_risk_estimator_non_det_bandit} has a similar form as estimator \eqref{eq:entropic_risk_estimator}, this will allow us to reuse Theo~\ref{theo:non-stationary-bandit-results} to prove convergence of a UCB-based algorithm for risk-aware non-stationary non-deterministic bandits below.

\paragraph{A UCB-based algorithm for the risk-aware non-stationary non-deterministic bandit}
A UCB-based algorithm for the $\text{ERM}_\beta$ under a non-stationary non-deterministic bandit replaces estimator \eqref{eq:entropic_risk_estimator} with estimator \eqref{eq:entropic_risk_estimator_non_det_bandit}. More precisely, the algorithm starts by selecting each of the arms once and then selects, at each timestep $t$, the arm with the best lower confidence bound, i.e.,
\begin{equation}
    \rvar{a}_t = \argmin_{i \in \{1,\ldots,K\}} \left\{ \hat{\rho}_{i,T_i(t)} - b_{t,T_i(t)} \right\}, \quad \quad b_{t,s} = \frac{(\theta^L)^{1/\xi^L} t^{\alpha/\xi^L}}{s^{1-\eta^L}}, \label{eq:ucb_erm_algorithm_2}
\end{equation}
where $\eta^L = \eta$, $\xi^L = \xi$, $\theta^L$ is given by Lemma~\ref{lemma:non-deterministic-bandit-convergence-and-concentration}, and $\alpha > 0$ is an hyperparameter.

Without loss of generality, we assume there exists a unique optimal arm. Let $\mu^* = \min_{i \in \{1, \ldots, K\}} \mu_i$, where $\mu_i$ is defined in Lemma~\ref{lemma:non-deterministic-bandit-convergence-and-concentration}, be the limiting value of the expectation of the estimator for the optimal arm. Similar to Sec.~\ref{sec:bandits:non-stat-risk-aware-bandit}, let $\Delta_i = \mu_i - \mu^*$ and $\Delta_\text{min} = \min_{i \in \{1, \ldots, K\} \setminus i^* } \Delta_i$ be the gap between the optimal arm and the second optimal arm.

\begin{theorem} \label{theo:non-stationary-non-deterministic-bandit-results}
    Let $\bar{\rho}_n = \frac{1}{n} \sum_{i=1}^K T_i(n) \hat{\rho}_{i, T_i(n)}$ be the empirical $\text{ERM}_\beta$ when running algorithm \eqref{eq:ucb_erm_algorithm_2} under a non-stationary non-deterministic bandit satisfying Assumption~\ref{assumption:non-deterministic-bandit-assumption}. Furthermore, for any $1/2 \le \eta < 1$, assume Assumption~\ref{assumption:non-deterministic-bandit-assumption} holds with a sufficiently large $\xi > 1$ value such that there exists $\alpha > 2$ satisfying $\xi\eta(1-\eta) \le \alpha < \xi (1-\eta)$. Then,
    \begin{enumerate}
        \item it holds that $|\mu^* - \EE[]{\bar{\rho}_n}| \le \O\left(\exp(2\beta R) n^{\frac{\alpha}{\xi(1-\eta)} -1}\right)$, i.e., $\lim_{n \rightarrow \infty} \EE[]{\bar{\rho}_n} = \mu^*$.
        \item there exist constants $\theta' >1$, $\xi' > 1$, and $1/2 \le \eta' < 1$ such that, for any $z \ge 1$ and $n \in \mathbb{N}$,
        \begin{equation}
            \mathbb{P}\left[n\bar{\rho}_n - n\mu^* \ge n^{\eta'} z \right] \le \frac{\theta'}{z^{\xi'}}, \quad \text{ and } \quad \mathbb{P}\left[n\bar{\rho}_n - n\mu^* \le - n^{\eta'} z \right] \le \frac{\theta'}{z^{\xi'}},
        \end{equation}
        where $\eta' = \frac{\alpha}{\xi(1-\eta)}$, $\xi' = \alpha - 1$, and $\theta'$ depends on $R$, $K$, $\Delta_\text{min}$, $\beta$, $\theta^L$, $\xi$, $\alpha$, and $\eta$.
    \end{enumerate}
\end{theorem}
\begin{proof}
    As previously discussed, estimator \eqref{eq:entropic_risk_estimator_non_det_bandit} has the same form as estimator \eqref{eq:entropic_risk_estimator}. Furthermore, action-selection \eqref{eq:ucb_erm_algorithm_2} is equivalent to action-selection \eqref{eq:ucb_erm_algorithm}, where the bonus term considers a different set of parameters. Given Assumption~\ref{assumption:non-deterministic-bandit-assumption}, from Lemma~\ref{lemma:non-deterministic-bandit-convergence-and-concentration} we know that estimator \eqref{eq:entropic_risk_estimator_non_det_bandit} (equiv., estimator \eqref{eq:entropic_risk_estimator}) satisfies, for any arm $i$: (i) $\lim_{n \rightarrow \infty} \mathbb{E}\left[\hat{\rho}_{i,n}\right] = \mu_i$; and (ii) there exist constants $\theta^L > 1$, $\xi^L > 1$, and $1/2 \le \eta^L < 1$ such that, for any $z \ge 1$ and $n \in \mathbb{N}$, $\mathbb{P}\left[n\hat{\rho}_{i,n} - n\mu_i \ge n^{\eta^L} z \right] \le \frac{\theta^L}{z^{\xi^L}},$ and $\mathbb{P}\left[n\hat{\rho}_{i,n} - n\mu_i \le - n^{\eta^L} z \right] \le \frac{\theta^L}{z^{\xi^L}}.$ We recall that $\xi^L = \xi$, $\eta^L = \eta$ and $\theta^L$ is given by Lemma~\ref{lemma:non-deterministic-bandit-convergence-and-concentration}. Hence, we meet Assumption~\ref{assumption:non-stationary-bandit-assumption} with parameters $\theta^L$, $\eta^L$, $\xi^L$ and the result follows from Theo.~\ref{theo:non-stationary-bandit-results}.
\end{proof}

\noindent Similar to Sec.~\ref{sec:bandits:non-stat-risk-aware-bandit}, the result above shows that the expected empirical $\text{ERM}_\beta$ of algorithm \eqref{eq:ucb_erm_algorithm_2} satisfies a convergence rate of $\O\left(\exp(2\beta R) n^{\frac{\alpha}{\xi(1-\eta)} -1}\right)$. By letting $\eta=1/2$ and $\alpha = \xi\eta(1-\eta)$ the convergence rate simplifies to $\O(\exp(2\beta R) n^{- 1/2})$. The result above extends the result of the previous section to the non-deterministic case. Similar to Sec.~\ref{sec:bandits:non-stat-risk-aware-bandit}, we can also consider estimator 
\begin{align}
    \hat{\rho}_n^\text{stream} &= \frac{1}{\beta} \ln\left( \frac{1}{n} \sum_{i=1}^K \sum_{s' \in \S} \sum_{t=1}^{T_i^{s'}(T_i(n))} \exp(\beta( c_i + \gamma \rvar{x}_{i,t}^{s'}))\right)  \label{eq:entropic_risk_estimator_non_det_bandit_stream_estimator} \\
    &= \frac{1}{\beta} \ln\left( \frac{1}{n} \sum_{i=1}^K T_i(n) \exp(\beta \hat{\rho}_{i,T_i(n)} ) \right) \nonumber ,
\end{align}
which computes the $\text{ERM}_\beta$ for the stream of costs obtained at the root of the bandit. The second equality above holds from algebraic manipulation and the definition of $\hat{\rho}_{i,n}$. Again, due to the non-linearity introduced by the $\ln$ and $\exp$ operators, $\hat{\rho}_n^\text{stream}$ differs in general from $\bar{\rho}_n$ considered in Theo.~\ref{theo:non-stationary-non-deterministic-bandit-results}. Nevertheless, we can exploit Theo.~\ref{theo:non-stationary-bandit-results-stream-estimator} and state the following result.

\begin{theorem} \label{theo:non-stationary-non-deterministic-bandit-results-stream-estimator}
    Let $\hat{\rho}_n^\text{stream}$, as defined in \eqref{eq:entropic_risk_estimator_non_det_bandit_stream_estimator}, be the empirical $\text{ERM}_\beta$ of the stream of costs at the root node of the bandit when running algorithm \eqref{eq:ucb_erm_algorithm_2} under a non-stationary non-deterministic bandit satisfying Assumption~\ref{assumption:non-deterministic-bandit-assumption}. Also, for $1/2 \le \eta < 1$, assume Assumption~\ref{assumption:non-deterministic-bandit-assumption} holds with a sufficiently large $\xi > 1$ value such that there exists $\alpha > 2$ satisfying $\xi\eta(1-\eta) \le \alpha < \xi (1-\eta)$. Then,
    \begin{enumerate}
        \item there exist constants $\theta'' >1$, $\xi'' > 1$, and $1/2 \le \eta'' < 1$ such that, for every $n \in \mathbb{N}$ and $z \ge 1$,
        \begin{equation}
            \mathbb{P}\left[n\hat{\rho}_n^\text{stream} - n\mu^* \ge n^{\eta''} z \right] \le \frac{\theta''}{z^{\xi''}}, \quad \text{ and } \quad \mathbb{P}\left[n\hat{\rho}_n^\text{stream} - n\mu^* \le - n^{\eta''} z \right] \le \frac{\theta''}{z^{\xi''}},
        \end{equation}
        where $\eta'' = \frac{\alpha}{\xi (1-\eta)}$, $\xi'' = \alpha - 1$, and $\theta'''$ depends on $R$, $K$, $\Delta_\text{min}$, $\beta$, $\theta$, $\theta''$, $\xi$, $\alpha$, and $\eta$.
        \item it holds that $|\mu^* - \hat{\rho}_n^\text{stream}| = \O\left(\theta'' n^{\frac{\alpha}{\xi(1-\eta)} -1}\right)$, i.e., $\lim_{n \rightarrow \infty} \mathbb{E}[\hat{\rho}_n^\text{stream}] = \mu^*$.
    \end{enumerate}
\end{theorem}
\begin{proof}
    The proof is similar to the proof of Theo.~\ref{theo:non-stationary-non-deterministic-bandit-results} and works by exploiting Theo.~\ref{theo:non-stationary-bandit-results-stream-estimator} given that, as highlighted in \eqref{eq:entropic_risk_estimator_non_det_bandit_stream_estimator}, estimator $\hat{\rho}_n^\text{stream}$ has precisely the same form as that introduced in \eqref{eq:entropic_risk_estimator_stream_estimator}. 
\end{proof}

\section{Entropic risk-aware Monte Carlo tree search} \label{sec:risk-aware-mcts}
In this section, we introduce a risk-aware MCTS algorithm for the $\text{ERM}_\beta$ that is provably correct. We exploit the results from the previous section to provide a non-asymptotic analysis of the proposed algorithm. Below, we provide our risk-aware MCTS algorithm and the main result of this paper, Theo.~\ref{theo:mcts-result}, which establishes a polynomial regret bound for our proposed algorithm.

The risk-aware MCTS algorithm for the $\text{ERM}_\beta$ we herein introduce incrementally constructs a search tree that alternates between decision nodes, where the agent can select between the different actions, and chance nodes, representing stochastic environment transitions. Each decision node $s_h \in \S$ of depth $h \in \{0, \ldots, H-1\}$ in the search tree stores $N(s_h)$ and $N(s_h,a)$, which, respectively, correspond to the number of times node $s_h$ has been visited and to the the number of times action $a \in \A$ has been selected while in node $s_h$. Furthermore, for each action $a \in \A$, decision node $s_h$ stores a list $\X_{(s_h,a)}$ containing the random discounted cumulative costs that have been sampled by starting from $s_h$ and taking action $a$ until a leaf node is reached. We denote with $\rvar{x}_{(s_h,a),t} = c_h(s_h,a) + \gamma c_{h+1}(\rvar{s}_{h+1}^t, \rvar{a}_{h+1}^t) + \ldots + \gamma^{H-h} c_H(\rvar{s}_H^t)$ the $t$-th sampled discounted cumulative cost starting from state-action pair $(s_h,a)$. We also let $\rvar{x}_{(s_h,a),t}^{s_{h+1}} = c_{h+1}(s_{h+1}^t, \rvar{a}_{h+1}^t) + \ldots + \gamma^{H-(h+1)} c_H(\rvar{s}_H^t)$ be the $t$-th random discounted cumulative cost starting from next state $s_{h+1} \in \S$ after taking action $a \in \A$ while in state $s_h \in \S$. Now, for each $s_h \in \S$ and $a \in \A$, let
\begin{equation}
    \hat{\rho}_{(s_h,a),n} = \frac{1}{\beta_h} \ln \left( \frac{1}{n} \sum_{t=1}^{n} \exp\left(\beta_h ( \underbrace{c_{h}(s_h,a) + \gamma \rvar{x}_{(\rvar{s}_h,a),t}^{\rvar{s}_{h+1}} }_{= \; \rvar{x}_{(\rvar{s}_h,a),t}} ) \right) \right) \label{eq:mcts_state_node_erm_estimator}
\end{equation}
be the $\text{ERM}_\beta$ for the discounted cumulative costs starting from state-action pair $(s_h,a)$, where $\beta_h = \beta \cdot \gamma^h$ is the entropic-risk parameter adjusted for depth $h$. Since $\rvar{x}_{(\rvar{s}_h,a),t} = c_{h}(s_h,a) + \gamma \rvar{x}_{(\rvar{s}_h,a),t}^{\rvar{s}_{h+1}}$, the estimator above essentially calculates the empirical $\text{ERM}_\beta$ for the sampled discounted cumulative sum of costs starting at state-action pair $(s_h,a)$. At each iteration of the algorithm, a sequence of actions is selected at the decision nodes starting from the root node $s_0$ until depth $H$ is reached, inducing a random sequence of states and actions $(\rvar{s}_0, \rvar{a}_0, \rvar{s}_1, \rvar{a}_1, \ldots, \rvar{s}_H)$ that traverses the search tree. Action selection at the decision nodes is given by
\begin{equation}
    \rvar{a}_{h} = \argmin_{a \in \A} \Bigg\{ \hat{\rho}_{(\rvar{s}_h,a),N(\rvar{s}_h,a)} - \frac{(\theta_{h+1}^L)^{1/{\xi_{h+1}}} N(\rvar{s}_h)^{\alpha_{h+1}/{\xi_{h+1}}}}{N(\rvar{s}_h,a)^{1-\eta_{h+1}}} \Bigg\} \label{eq:mcts_action-selection},
\end{equation}
where $\{\alpha_h\}_{h \in \{1, \ldots, H\}}$, $\{(\theta_h^L)\}_{h \in \{1, \ldots, H\}}$, $\{\eta_h\}_{h \in \{1, \ldots, H\}}$, and $\{\xi_h\}_{h \in \{1, \ldots, H\}}$ are parameters of the algorithm. We assume that if $N(\rvar{s}_h,a) = 0$ for some action $a \in \A$, then $a$ is selected. During backpropagation, the visitation counts $N(\rvar{s}_h)$ and $N(\rvar{s}_h,\rvar{a}_h)$ are incremented and the algorithm updates lists $\X_{(\rvar{s}_h,\rvar{a}_h)}$ for all nodes $\rvar{s}_h$ along the traversed path by storing the respective sampled discounted cumulative costs starting from $\rvar{s}_h$. We refer to Appendix~\ref{appendix:mcts_pseudocode} for the pseudocode of the risk-aware MCTS algorithm. We state the following result (full proof in Appendix~\ref{appendix:proof_of_mcts_result}) on the non-asymptotic performance of the risk-aware MCTS described above. 

\begin{theorem} \label{theo:mcts-result}
    For a discounted finite-horizon MDP satisfying Assumption~\ref{assumption:mdp_assumptions} and $1/2 \le \eta < 1$, let
    \begin{align*}
        \eta_h &= \eta, \; \forall h \in \{1, \ldots H\} \\
        \xi_h &= \alpha_{h+1} - 1, \; \forall h \in \{1, \ldots, H-1\}, \text{ and } \xi_{H} > 1 \text{ is arbitrary}, \\
        \alpha_h &= \eta(1-\eta)\xi_h, \; \forall h \in \{1, \ldots, H\}.
    \end{align*}
    Suppose $\xi_{H} > 1$ is chosen large enough\footnote{Choosing a large enough $\xi_H$ value is not a restrictive condition since $\xi_H$ can be set arbitrarily at the leaf nodes.} such that for every depth $h \in \{1, \ldots, H\}$ it holds that $\xi_h > 1$ and $\alpha_h > 2$. Then, there exist constants $\{\theta_h^L\}_{h \in \{1, \ldots, H\}}$ such that the risk-aware MCTS algorithm with action-selection given by \eqref{eq:mcts_action-selection} satisfies
    $\lim_{n \rightarrow \infty} \mathbb{E}[\hat{V}_{n}(s_0)] = V^*_0(s_0),$
    where $\hat{V}_{n}(s_0) = \frac{1}{\beta} \ln \left( \frac{1}{n} \sum_{a \in \A} \sum_{t=1}^{T_a(n)} \exp\left(\beta \rvar{x}_{(s_0,a),t}\right) \right)$ is the empirical $\text{ERM}_\beta$ obtained by the risk-aware MCTS algorithm at the root after $n$ iterations. Furthermore, there exist constants $\theta_0 > 1$, $\xi_0 > 1$ and $1/2 \le \eta_0 < 1$ such that, for every $n \in \mathbb{N}$ and $z \ge 1$ it holds that
    \begin{equation*}
        \mathbb{P}\left[n\hat{V}_{n}(s_0) - nV^*_0(s_0) \ge n^{\eta_0} z \right] \le \frac{\theta_0}{z^{\xi_0}}, \quad \text{ and } \quad \mathbb{P}\left[n\hat{V}_{n}(s_0) - nV^*_0(s_0) \le - n^{\eta_0} z \right] \le \frac{\theta_0}{z^{\xi_0}}.
    \end{equation*}
\end{theorem}
\begin{proofsketch}
    The proof follows by recursively applying Theo.~\ref{theo:non-stationary-non-deterministic-bandit-results-stream-estimator} in the context of trees. We start by verifying that Assumption~\ref{assumption:non-deterministic-bandit-assumption}, as required by Theo.~\ref{theo:non-stationary-non-deterministic-bandit-results-stream-estimator}, is satisfied at the leaf nodes $s_H$ (base case). This is because the ERM estimator for the stream of costs obtained at the leaves enjoys exponential concentration, implying it also satisfies the convergence and polynomial concentration properties stated in Assumption~\ref{assumption:non-deterministic-bandit-assumption}. Then, from the perspective of node $s_h$ at depth $h \in \{0, \ldots, H-1\}$, the action selection mechanism prescribed by the risk-aware MCTS algorithm, i.e., \eqref{eq:mcts_action-selection}, is equivalent to the action selection mechanism we present in Sec.~\ref{sec:bandits:non-stat-non-det-risk-aware-bandit} in the context of a non-stationary non-deterministic bandits, i.e., \eqref{eq:ucb_erm_algorithm_2}. Intuitively, the non-stationarity here is introduced by changes in action selection at the decision nodes below $s_h$ in the search tree. Therefore, as an induction step, we assume layer $h+1$ satisfies Assumption~\ref{assumption:non-deterministic-bandit-assumption} and apply Theo.~\ref{theo:non-stationary-non-deterministic-bandit-results-stream-estimator} to obtain convergence and concentration at the decision node $s_h$. We recursively apply Theo.~\ref{theo:non-stationary-non-deterministic-bandit-results-stream-estimator} until we establish convergence and polynomial concentration at the root node $s_0$. We also note that, at each decision node $s_h$, the limiting empirical $\text{ERM}_{\beta_h}$ satisfies the Bellman optimality equations \eqref{eq:erm_bellman_optimality equations} and, thus, the actions selected by the algorithm are optimal and the empirical $\text{ERM}_{\beta}$ for the stream of costs at the root node $s_0 \in \S$ converges to $V^*_0(s_0)$. Finally, we also show that there exists a set of parameters $\{\alpha_h\}_{h \in \{1, \ldots, H\}}, \{\theta_h\}_{h \in \{1, \ldots, H\}}, \{\eta_h\}_{h \in \{1, \ldots, H\}}, \text{ and } \{\xi_h\}_{h \in \{1, \ldots, H\}}$ satisfying all required assumptions at each level of the tree, as well as the recursive relations imposed by Theo.~\ref{theo:non-stationary-non-deterministic-bandit-results-stream-estimator}.
\end{proofsketch}

\noindent Theo.~\ref{theo:mcts-result} shows that our proposed risk-aware MCTS algorithm not only provably solves risk-aware MDPs with ERM objectives, but also enjoys polynomial regret concentration. Theo.~\ref{theo:mcts-result} heavily relies on the results from Sec.~\ref{sec:bandits:non-stat-non-det-risk-aware-bandit}. Thus, from the perspective of a decision node $s_h \in \S$ in the tree, the convergence rate satisfies $\O\left(n^{\frac{\alpha_{h+1}}{\xi_{h+1} (1-\eta_{h+1})} -1}\right)$, which can be simplified to $\O\left(n^{\eta - 1}\right)$ given our choice of parameters. Letting $\eta=1/2$ yields an optimal convergence rate of $\O\left(n^{-1/2}\right)$, in which case the bonus terms are of the form $\sqrt{\sqrt{N(s_h)}/N(s_h,a)}$. Also, as discussed in Sec.~\ref{sec:bandits:non-stat-non-det-risk-aware-bandit}, larger $\beta$ values are detrimental to the convergence rate. However, a decision node $s_h \in \S$ at depth $h \in \{0, \ldots, H-1\}$ uses the adjusted parameter $\beta_h = \beta \cdot \gamma^h$. Hence, we expect $\beta$ to have a stronger impact on the convergence rate at shallower nodes of the tree.

\section{Illustrative experiments} \label{sec:experiments}

\begin{figure}[t]
\centering

\begin{minipage}{0.32\linewidth}
\centering
\includegraphics[height=3.5cm]{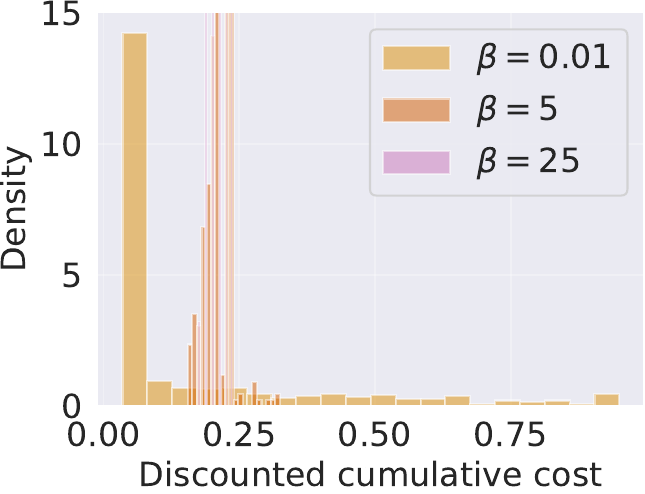}

\vspace{0.3em}
\footnotesize (a) MDP-4.
\end{minipage}
\hfill
\begin{minipage}{0.32\linewidth}
\centering
\includegraphics[height=3.5cm]{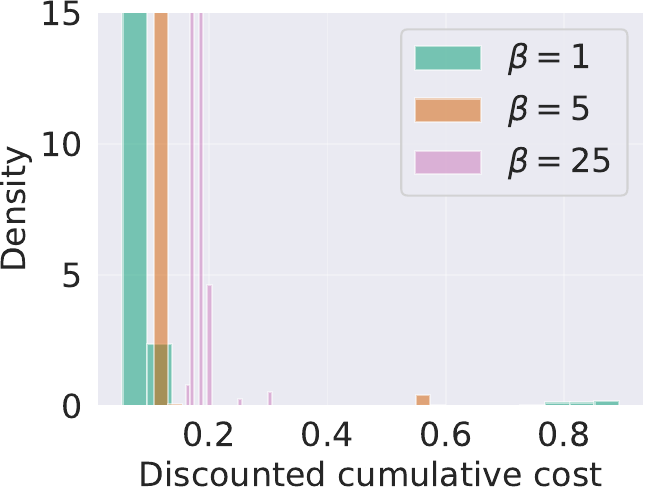}

\vspace{0.3em}
\footnotesize (b) Grid-MDP.
\end{minipage}
\hfill
\begin{minipage}{0.32\linewidth}
\centering
\includegraphics[height=3.5cm]{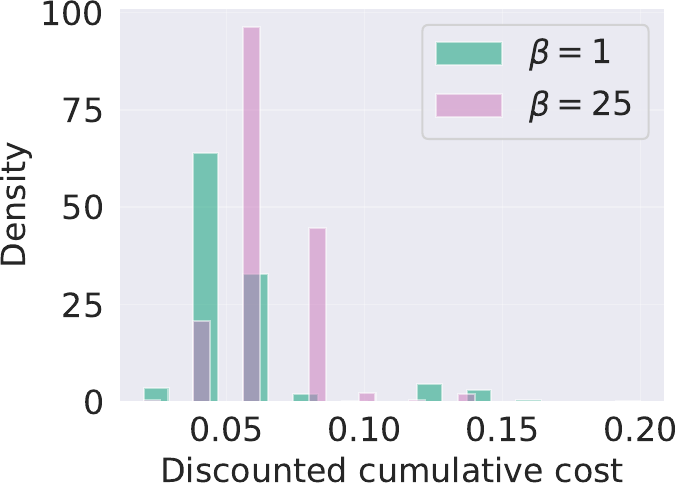}

\vspace{0.3em}
\footnotesize (c) Bin-packing.
\end{minipage}

\caption{%
Distribution of discounted costs obtained by \textsc{ERM-MCTS} for different $\beta$ values.
}
\label{fig:experimental-results-distributions}
\end{figure}

\begin{table}[t]
    \centering
    \small
    \renewcommand{\arraystretch}{1.2}
    \begin{tabular}{|c|c||c|c|c|}
    \hline
    $\beta$ & Algorithm & Crush & Spill & Total \\ \hline \hline
    \multirow{2}{*}{1}  & \textsc{Acc-MCTS}        & 19 [12, 26] & 21 [14, 29] & 40 [30, 51] \\ \cline{2-5}
                        & \textsc{ERM-MCTS} (Ours) & 17 [11, 24] & 23 [15, 31] & 40 [30, 50] \\ \hline
    \multirow{2}{*}{25} & \textsc{Acc-MCTS}        & 0 [0, 0]    & 5 [2, 9]    & 5 [2, 9]    \\ \cline{2-5}
                        & \textsc{ERM-MCTS} (Ours) & 1 [0, 3]    & 6 [2, 10]   & 7 [3, 11]   \\ \hline
    \end{tabular}
    \caption{(Bin-Packing) Total number of accident events over $500$ episodes, by event type, with $[5\%, 95\%]$ bootstrapped confidence intervals in brackets. Results for $n=5\;000$ MCTS iterations and $K_\text{UCB} = 0.3$.}
    \label{table:binpacking-accidents}
\end{table}

We present a set of illustrative experiments comparing our algorithm, \textsc{ERM-MCTS}, against: (i) an oracle baseline, \textsc{ERM-BI}, that computes the optimal policy via backward induction using \eqref{eq:erm_bellman_optimality equations} as introduced by \citet{hau_2023}; and (ii) our adaptation of the algorithm proposed by \citet{hayes_2023}, \textsc{Acc-MCTS}, to the ERM case. For the purpose of our illustrative experiments, we implement \textsc{ERM-MCTS} as described in the previous section with $\eta = 1/2$, but let $\theta^L_h = K_\text{UCB}^{\xi_h}$ so the bonuses are constant with value $K_\text{UCB}$, as commonly used in the MCTS literature. We let $K_\text{UCB} = 0.3$ for the experimental results herein presented. We obtained similar results for other $K_\text{UCB}$ constants, and we refer to Appendix~\ref{appendix:empirical-results} for the complete results. Our code is available \href{https://github.com/JacopoSilvestrin/risk-aware-mcts}{here}.

We consider three environments. The MDP-4 environment consists of a four-state MDP where the agent needs to tradeoff, at the initial state, between a safe action that leads to a medium-cost state and a risky action that can lead to both low- or high-cost states. The agent resets to the initial state with a given probability. The Grid-MDP is a $5 \times 3$ grid environment with an additional absorbing \textit{pitfall} state. From the initial state, the agent must choose between a safer and a riskier path to the target state, trading off a lower expected cost against an increased probability of transitioning to the pitfall state. The Bin-packing environment consists of a stochastic sequential environment that models the task of a robot packing grocery items into a bag at a checkout counter. At each timestep in the episode, the agent observes the current item to be packed, together with the state of all bags, and selects a bag into which to place the item. The agent aims to use as few bags as possible while avoiding crushing
and spilling events. The environment is designed to exhibit meaningful tail risk: most packing sequences incur only mild costs due to packing inefficiency, but a small fraction result in high-cost accident events (e.g., crushed fragile items or spilled liquids). The number of possible states in the Bin-packing environment is of the order $O(115^H)$.

We set horizon $H=20$ for MDP-4, $H=15$ for Grid-MDP, and $H=5$ for the Bin-packing environment. We run 500 seeds with discount factor $\gamma=1$ for each experimental setup and report the bootstrapped ERM confidence intervals in the tables and figures. In every environment, costs are rescaled so that the cumulative cost of an episode lies in $[0,1]$, all reported ERM values therefore refer to normalized costs. We refer to Appendix~\ref{appendix:empirical-results} for more details regarding our environments and the complete experimental results.

In Figure~\ref{fig:experimental-results-distributions}, we report the distributions of cumulative discounted costs produced by \textsc{ERM-MCTS} for the three environments after $5\;000$ MCTS iterations. As the risk sensitivity parameter $\beta$ increases, the induced distributions progressively shift away from high-cost outcomes, indicating that \textsc{ERM-MCTS} exhibits increasingly risk-averse behavior. For the Bin-Packing environment, Table~\ref{table:binpacking-accidents} reports the number of accident events (crush or spill) incurred by \textsc{ERM-MCTS} and \textsc{Acc-MCTS} for different values of the risk-sensitivity parameter $\beta$. As shown, the number of accidents decreases substantially as $\beta$ increases, indicating that both algorithms induce increasingly conservative behavior to avoid adverse events.

\begin{figure}[t]
\centering

\begin{minipage}{0.32\linewidth}
\centering
\includegraphics[height=3.5cm]{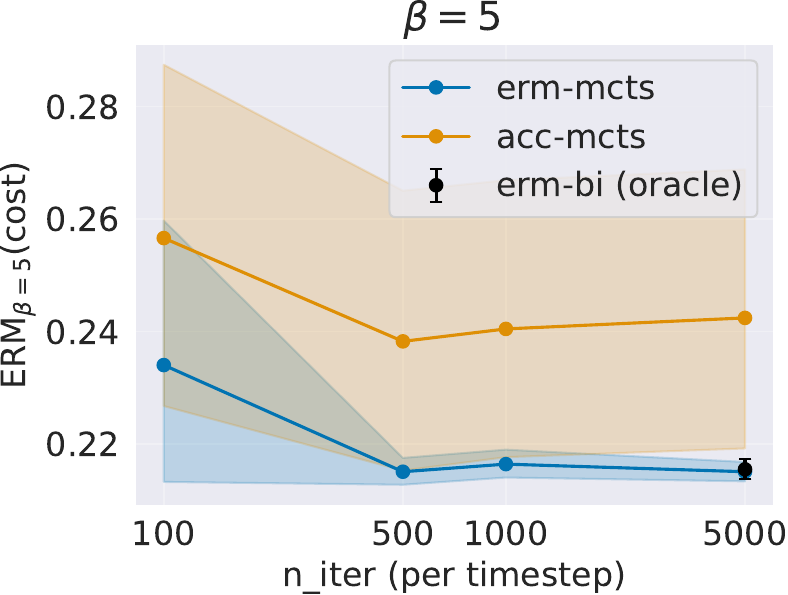}

\vspace{0.3em}
\footnotesize (a) MDP-4.
\end{minipage}
\hfill
\begin{minipage}{0.32\linewidth}
\centering
\includegraphics[height=3.5cm]{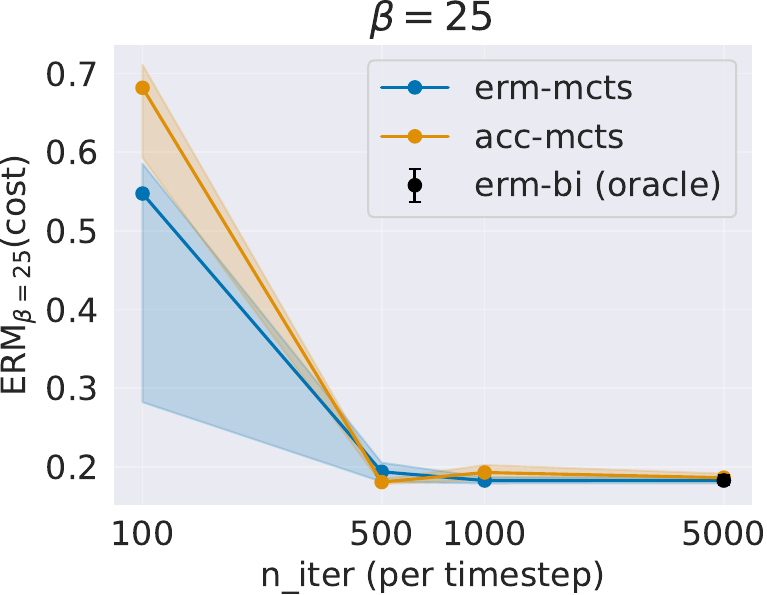}

\vspace{0.3em}
\footnotesize (b) Grid-MDP.
\end{minipage}
\hfill
\begin{minipage}{0.32\linewidth}
\centering
\includegraphics[height=3.5cm]{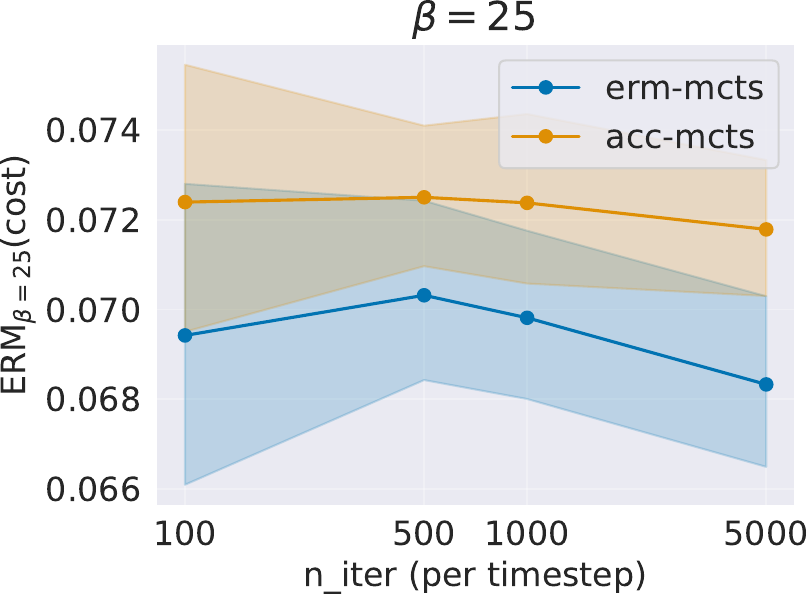}

\vspace{0.3em}
\footnotesize (c) Bin-packing.
\end{minipage}

\caption{%
$\text{ERM}_\beta$ for different algorithms as a function of the number of MCTS iterations ($n$).
}
\label{fig:experimental-learning-curves}
\end{figure}

Figure~\ref{fig:experimental-learning-curves} compares the $\text{ERM}_{\beta}$ values obtained by different algorithms as a function of the number of MCTS iterations. First, we focus our attention on the MDP-4 and Grid-MDP environments, under which we can run the backward induction-based oracle algorithm ($\textsc{\text{ERM}-BI}$) as the number of states is manageably small. As seen in Figure~\ref{fig:experimental-learning-curves} (a) and (b), \textsc{ERM-MCTS} achieves a similar $\text{ERM}_{\beta}$ value to that of the oracle baseline as the number of iterations increases. Also, as reported in the Appendix, for all tested 19 configurations of $\beta$ and $K_\text{UCB}$ values, \textsc{ERM-MCTS} achieves a similar $\text{ERM}_{\beta}$ value to that of the oracle baseline in 18 out of 19 experimental configurations. Second, considering all three environments, and according to our results in Figure~\ref{fig:experimental-learning-curves}, \textsc{ERM-MCTS} is also at least as good as \textsc{Acc-MCTS} baseline.

Our illustrative results show that \textsc{ERM-MCTS}, which exploits the dynamic programming decompositions of the $\text{ERM}_{\beta}$, can successfully trade off risk-neutral and risk-averse behavior while achieving quantitative performance comparable to that of relevant baselines. Furthermore, the results obtained in the high-dimensional Bin-packing environment demonstrate that \textsc{ERM-MCTS} can effectively handle large state-action spaces.

\section{Conclusion}
We propose a theoretically grounded MCTS algorithm to solve risk-aware MDPs with ERM objectives. We demonstrate that our algorithm is correct and satisfies polynomial regret concentration. We present illustrative experiments showing that our method successfully trades off between risk-neutral and risk-averse behaviors, while comparing its performance against other baselines. Future work should better compare, both theoretically and empirically, the sample complexity of different approaches, such as  \textsc{ERM-MCTS} and \textsc{Acc-MCTS}, to solve MDPs with ERM objectives.

\subsubsection*{Broader Impact Statement}
This paper presents work whose goal is to advance the field of Machine Learning. There are many potential societal consequences of our work, none of which we feel must be specifically highlighted here.


\bibliography{main}
\bibliographystyle{tmlr}

\appendix
\section{Extended related work}
\label{appendix:extended-related-work}
In this section, we discuss previous works that are related to our study, reviewing works on risk-neutral MCTS, risk-aware multi-armed bandits, and risk-aware and safe MCTS. 

\subsection{Risk-neutral MCTS}
Seminal works explored the extension of the UCB \citep{auer_2002} algorithm to sequential decision-making settings \citep{kearns_1999,chang_2005,kocsis_2006,kocsis_2006_b}. We highlight the works of \citet{kocsis_2006,kocsis_2006_b}, which proposed a UCB algorithm for trees named UCT and argued for its convergence to optimal behavior. UCT considers logarithmic bonuses of the form $\sqrt{\ln(N(s_h))/N(s_h,a)}$. The work of \citep{munos_2014} provides a broader discussion beyond UCT, investigating how optimistic planning can be used to solve sequential decision-making problems.

Subsequent works \citep{shah_2020,comer_2025} provide non-asymptotic analyses of MCTS with polynomial bonuses similar to the bonuses we consider in our work. In particular, \citet{shah_2020} claim that the analysis provided by \citet{kocsis_2006,kocsis_2006_b} is incorrect since it relies on independence assumptions that do not hold in practice and, therefore, MCTS with logarithmic bonuses does not enjoy exponential concentration. With this in mind, the authors provide convergence and concentration results for an MCTS algorithm with polynomial bonuses under MDPs with deterministic dynamics. Later, \citet{comer_2025} extend the results of \citet{shah_2020} to the case of MDPs with non-deterministic dynamics. Our work borrows ideas from the results and proofs put forth by both \citet{shah_2020} and \citet{comer_2025}. We follow a similar line of reasoning as that of \citet{shah_2020,comer_2025} given that we first prove convergence and concentration results for non-stationary bandits and then exploit such results to prove convergence and concentration in the context of trees. However, as discussed thoughout our article, there are fundamental differences between our analysis and that provided by the aforementioned works, majorly because the mean estimator and the ERM estimator are inherently different (e.g., the mean estimator is linear whereas the ERM estimator is non-linear, the ERM is not positively-homogeneous while the mean estimator is, etc.). Hence, the extension of the results put forth by \citet{shah_2020,comer_2025} to the ERM case is non-trivial and forced us to either substantially adapt the proofs of the aforementioned articles, or to prove additional results (e.g., Theo.~\ref{theo:non-stationary-bandit-results-stream-estimator} or Theo.~\ref{theo:non-stationary-non-deterministic-bandit-results-stream-estimator}). Finally, we emphasize that we are explicit in all our proofs whenever we reuse sub-results from the previously cited articles.

We refer to the surveys on risk-neutral MCTS \citep{browne_2012,swiechowski_2022} for a detailed discussion of previous works studying MCTS and applications.



\subsection{Risk-aware multi-armed bandits}
Previous works investigate risk-aware bandits \citep{sani_2012,kagrecha_2019,cassel_2018,chang_2022,farsang_2023,liang_2021,tan_2022_survey}. We refer to \citet{tan_2022_survey} for a survey of such works. In particular, as discussed by \citet{tan_2022_survey}, previous works considered both regret minimization, where the objective is to minimize the expected regret defined as a function of the risk measure of the distribution associated with each of the arms, as well as pure exploration settings, where the objective is to find the arm associated with the optimal risk measure value.  \citet{tan_2022_survey} present a risk UCB algorithm for regret minimization that, while assuming that the distributions of the arms are sub-Gaussian and the risk measure is Lipschitz, achieves a sub-linear regret guarantee with a regret bound similar to that of risk-neutral UCB. Other works, such as \citep{cassel_2018}, provide a systematic approach to analyzing risk-aware stochastic multi-armed bandits, focusing on UCB-like policies. \citet{chang_2022} explore Thompson sampling-based solutions to risk-averse multi-armed bandits. \citet{farsang_2023} focus on a risk-averse contextual bandits. Regarding the ERM case, \citet{liang_2021} propose a Thompson sampling-based algorithm for entropic risk bandits. The authors provide an analysis of the asymptotic regret of the proposed algorithm, showing that the algorithm is asymptotically optimal.

In our work, we investigate, in Sec.~\ref{sec:bandits}, UCB-based algorithms for ERM bandits. We show that our proposed algorithm is asymptotically optimal and that it satisfies polynomial regret concentration. Our main objective with Sec.~\ref{sec:bandits} is to provide the necessary results to then prove convergence of our MCTS algorithm to solve ERM risk-aware MDPs. Nevertheless, the results we present in Sec.~\ref{sec:bandits} contribute with a complementary analysis of UCB-based algorithms for ERM bandits to that presented in the aforementioned works. In particular, we highlight that we are the first to provide convergence and concentration results in the context of non-stationary ERM bandits.

\subsection{Risk-aware MCTS}
\citet{hayes_2023} propose an MCTS-based algorithm to solve expected scalarized objectives with non-linear utility functions. More precisely, the authors aim to solve
\begin{equation}
    \pi^* \in \argmin_{\pi} \mathbb{E}\left[u\left(\sum_{t=0}^\infty \gamma^t c(\rvar{s}_t, \rvar{a}_t)\right)\right], \label{eq:hayes_ESR_objective}
\end{equation}
where $u$ is a utility function. The authors focus on both scalar and vectorized (multi-objective setting) cost functions. In the context of scalar cost functions, the authors study risk-aware sequential decision making by setting the utility function $u$ to particular risk-seeking or risk-averse functions. \citet{hayes_2023} propose two MCTS-based algorithms to solve the problem above, one based on UCB-like exploration strategies with logarithmic bonuses and the other based on Thompson sampling. The authors experiment with both MCTS variants in the context of a stock exchange environment put forth by \cite{shen_2014}.

We now highlight some key connections and differences between our work and that of \citet{hayes_2023}. We start by noting that the risk-aware formulations considered by \citet{hayes_2023} differ from our work, as we consider the case of the ERM and \citet{hayes_2023} consider particular utility functions that promote risk-seeking or risk-averse behaviors that differ from the ERM formulation. Nevertheless, since \citet{hayes_2023} allow for any non-linear utility function $u$, we can let $u(x) = \exp(\beta x)$ and thus recover the ERM setting given that solving \eqref{eq:hayes_ESR_objective} with $u(x) = \exp(\beta x)$ is equivalent to solving (due to the fact that the $\ln$ operator is monotonically non-decreasing)
$$\pi^* \in \argmin_{\pi} \frac{1}{\beta} \ln\left( \mathbb{E}\left[u\left(\sum_{t=0}^\infty \gamma^t c(\rvar{s}_t, \rvar{a}_t)\right)\right] \right).$$
Therefore, in theory, the MCTS algorithms proposed by \citet{hayes_2023} can be used to solve risk-aware settings under entropic risk measures, even though the authors did not explore this research direction in their work. While \citet{hayes_2023} did not provide any kind of convergence or concentration guarantees for their proposed algorithms, we argue that their proposed MCTS variants, in particular NLU-MCTS that is based on UCB-exploration, is no different from standard MCTS with logarithmic bonuses. This is because the approach proposed by the authors to solve \eqref{eq:hayes_ESR_objective} can be equivalently seen as applying a standard MCTS algorithm to a particular MDP, derived from the original problem, where the agent keeps track of the accrued cumulative costs/rewards up to any timestep (i.e., the state space is extended to accommodate this information) and where the cost/reward function is zero except at the terminal timestep where it equals $u$ evaluated at the accrued costs/rewards throughout the entire trajectory. Therefore, NLU-MCTS enjoys similar convergence and concentration properties as standard MCTS algorithms.

However, we note that the concentration properties are associated with solving a rather sparse cost/reward finite-horizon MDP derived from the original risk-aware MDP (i.e., the concentration properties are not directly associated with solving the original risk-aware MDP). On the other hand, our approach directly solves the original risk-aware MDP, providing a principled MCTS algorithm that exploits the dynamic programming decompositions of the the ERM put forth by \cite{hau_2023}. We also theoretically analyze our algorithm while borrowing ideas from previous results regarding the risk-neutral MCTS algorithm \citep{kocsis_2006,shah_2020,comer_2025}, successfully extending the proofs and results of previous articles to the risk-aware setting. We close by highlighting that the main goal of \citet{hayes_2023} is to propose a practical MCTS algorithm to solve expected scalarized objectives, and not to provide a theoretical analysis of their approach (as it is the main goal of our work in the context of risk-aware MDPs with ERM objectives); in fact, it is likely that the approach proposed by \citet{hayes_2023} lacks theoretical guarantees since the authors consider logarithmic bonuses and previous works were only able to provide robust analyzes in the context of polynomial bonuses \citep{munos_2014,shah_2020,comer_2025}.

\subsection{Safe/constrained MCTS}
Other works consider the use of MCTS to solve constrained MDPs \citep{altman_1999}. Constrained MDPs differ from risk-aware MDPs since, in constrained MDPs, the objective is to minimize/maximize the expected cumulative discounted sum of costs/rewards, while meeting
additional constraints related to the cumulative discounted sum of costs/rewards under other cost/reward functions. There exist different variants of MCTS algorithms for constrained MDPs, depending on how the constraints are encoded and enforced by the MCTS algorithm \citep{parthasarathy_2024,kurecka_2025,zhang_2025}.

\citet{zhang_2025} propose a safety-aware MCTS algorithm under the context of constrained MDPS. In particular, the authors aim to solve a constrained MDP where the constraints are defined in terms of the Conditional Value-at-Risk (CVaR) of the distribution of discounted cumulative rewards. \citet{zhang_2025} introduce CVaR-MCTS, which integrates the CVaR constraints with MCTS via Lagrangian duality and dual stochastic gradient descent/ascent updates. The authors provide a bound on the regret of the proposed algorithm and empirically compare their algorithm against other baselines. Contrary to \citet{zhang_2025}, we do not consider constraints and, instead, focus on optimizing the (unconstrained) ERM of the distribution of discounted cumulative costs. We also highlight that we focus on the case of the ERM, whereas \citet{zhang_2025} focus on the case of the CVaR. Furthermore, the proofs of our theoretical results follow significantly different lines of reasoning. In particular, we start by proving convergence and concentration results for risk-aware non-stationary bandits that are then recursively apply them in the context of trees. On the other hand, \citet{zhang_2025} first focus on a particular node in the search tree and then sum up the regret over all nodes. 

\newpage

\section{Properties of the ERM and ERM estimators}
\label{appendix:ERM-estimator-properties}
For a given random variable $\rvar{x} \in [-R,R]$, the $\text{ERM}_\beta$ associated with the random variable $\rvar{x}$ is
\begin{equation*}
    \frac{1}{\beta} \ln\left( \mathbb{E}[ \exp(\beta \rvar{x})]\right) = \min_{\lambda \in \mathbb{R}} \left\{ \lambda + \mathbb{E}\left[ u_\beta(\rvar{x} - \lambda) \right]\right\},
\end{equation*}
where the equality above follows from the Optimized Certainty Equivalent (OCE) formulation of the ERM with $u_\beta(x) = \frac{1}{\beta} \left( \exp(\beta x) - 1\right)$. Furthermore, let $(\rvar{x}_{1}, \ldots, \rvar{x}_{n})$ be a sequence of independent and identically distributed (i.i.d.) random variables taking values in $[-R,R]$. For any $\beta > 0$, let also $\hat{\rho}_{n} = \min_{\lambda \in \mathbb{R}} \left\{ \lambda + \frac{1}{n} \sum_{t=1}^n u_\beta(\rvar{x}_{t} - \lambda) \right\}$. It holds that
\begin{equation*}
    \hat{\rho}_{n} = \min_{\lambda \in \mathbb{R}} \left\{ \lambda + \frac{1}{n} \sum_{t=1}^n u_\beta(\rvar{x}_{t} - \lambda) \right\} \overset{(a)}{=} \frac{1}{\beta} \ln\left(\frac{1}{n} \sum_{t=1}^n \exp(\beta \rvar{x}_{t})\right) ,
\end{equation*}
where (a) above follows from the Lemma below. Therefore, the ERM can be equivalently written in its standard form or as an OCE. Both formulations are equivalent given the Lemma below, and, therefore, we use them interchangeably in our article.

\begin{lemma}[Properties of the ERM estimator]
    \label{lemma:estimator_properties}
    For any $\beta > 0$, $n \in \mathbb{N}$, and sequence of i.i.d. random variables $(\rvar{x}_{1}, \ldots, \rvar{x}_{n})$ satisfying $\rvar{x}_t \in [-R,R]$, let $h(\lambda) = \lambda + \frac{1}{n} \sum_{t=1}^{n} u_\beta(\rvar{x}_{t} - \lambda)$. It holds that:
    \begin{enumerate}
        \item $h$ is a convex function of $\lambda$ and attains a minimum on $\mathbb{R}$.
        \item if $\lambda^* \in \argmin_{\lambda \in \mathbb{R}} h(\lambda)$, then $\lambda^* \in [-R,R]$ and $\lambda^* = \frac{1}{\beta} \ln \left( \frac{1}{n} \sum^n_{t=1} \exp (\beta \rvar{x}_{t} ) \right)$.
        \item $h(\lambda^*) \in  [-R,R]$ with $h(\lambda^*) = \frac{1}{\beta} \ln \left( \frac{1}{n} \sum^n_{t=1} \exp ( \beta \rvar{x}_{t} ) \right)$.
    \end{enumerate} 
\end{lemma}
\begin{proof}
    Consider the second order derivative of $h(\lambda)$
    \begin{equation*}
        h''(\lambda) = \frac{\beta}{n} \sum^n_{t=1} \exp \left( \beta(\rvar{x}_{t}-\lambda \right) \ge 0, \quad \forall \lambda \in \mathbb{R}.
    \end{equation*}
    The convexity of $h(\lambda)$ follows from the second-order derivative criterion as $h''(\lambda) \ge 0$. We now show that the function attains the minimum in $\mathbb{R}$. Consider the first order derivative
    \begin{equation*}
        h'(\lambda) = 1 - \frac{1}{n} \sum^n_{t=1} \exp \left( \beta (\rvar{x}_{t}-\lambda) \right).
    \end{equation*}
    Solving for $h'(\lambda) = 0$ to find the minimizer $\lambda^*$ gives
    \begin{equation*}
        \lambda^* = \frac{1}{\beta} \ln \left( \frac{1}{n} \sum^n_{t=1} \exp (\beta \rvar{x}_{t} ) \right)\in \mathbb{R},
    \end{equation*}
    for $\rvar{x}_{t} \in [-R,R]$ bounded. Notice that $\lambda^*$ is monotonically increasing in $\rvar{x}_{t}$, and obtains its highest value when $\rvar{x}_{t} = R, \forall t\in\{1,n\}$, which gives
    \begin{equation*}
        \lambda^*_{\text{max}} = \frac{1}{\beta} \cdot \beta R = R.
    \end{equation*}
    Likewise, the minimum value of $\lambda^*$ is obtained for $\rvar{x}_{t} = -R \quad \forall t\in\{1,n\}$, which gives
    \begin{equation*}
        \lambda^*_{\text{min}} = \frac{1}{\beta} \cdot (-\beta R) = -R.
    \end{equation*}
    This proves that $\lambda^* \in [-R,R]$.\\

    \noindent We now show that $h(\lambda^*) \in [-R,R]$. Solving for $\lambda = \lambda^*$ we get

    \begin{align*}
        h(\lambda^*) =& \frac{1}{\beta} \ln \left( \frac{1}{n} \sum_{t=1}^n \exp (\beta \rvar{x}_{t} ) \right) + \frac{n}{n \beta} \left( \sum^n_{t=1} \exp (\beta \rvar{x}_{t}) \right)^{-1} \left( \sum^n_{t=1} \exp (\beta \rvar{x}_{t}) \right) - \frac{n}{\beta n} \\
        =& \frac{1}{\beta} \ln \left( \frac{1}{n} \sum^n_{t=1} \exp ( \beta \rvar{x}_{t} ) \right),
    \end{align*}
    which is monotonically increasing in $\rvar{x}_{t}$, and obtains its highest value when $\rvar{x}_{t} = R \quad \forall t \in \{ 1,n \}$.

    \begin{equation*}
        h(\lambda^*)_{\text{max}} = \frac{1}{\beta} \cdot (\beta R) = R.
    \end{equation*}

    \noindent Likewise, the minimum is obtained for $\rvar{x}_{t} = -R \quad \forall t \in \{ 1,n \}$.

    \begin{equation*}
        h(\lambda^*)_{\text{min}} = \frac{1}{\beta} \cdot (-\beta R) = -R.
    \end{equation*}

    \noindent We conclude that $h(\lambda^*) \in [-R,R]$.
\end{proof}

\begin{lemma}[Exponential concentration of the ERM estimator]
   Let $(\rvar{x}_{1}, \ldots, \rvar{x}_{n})$ be a sequence of i.i.d. random variables taking values in $[-R,R]$. For any $\beta > 0$ and $z>0$ it holds that
    \begin{align}
        \mathbb{P}\left[|\hat{\rho}_{n} - \rho(\rvar{x}_1)| \ge z \right] &\le 2 \exp\left(-Cnz^2\right), \label{eq:entropic_risk_concentration_bound_two_sided} \\
        \mathbb{P}\left[\hat{\rho}_{n} - \rho(\rvar{x}_1) \ge z \right] &\le 2 \exp\left(-Cnz^2\right), \label{eq:entropic_risk_concentration_bound_1} \\
        \mathbb{P}\left[\hat{\rho}_{n} - \rho(\rvar{x}_1) \le - z \right] &\le 2 \exp\left(-Cnz^2\right),
        \label{eq:entropic_risk_concentration_bound_2}
    \end{align}
    where $C = \frac{\omega^2}{8 L_u^2 \sigma^2}$ with $\omega = \beta \exp(\beta (-2R))$, $L_u = \beta \exp(\beta 2R)$, and $\sigma = R^2$.
\end{lemma}
\begin{proof}
    Theorem 3.2 in \cite{ghosh2024concentrationboundsoptimizedcertainty} states that if: (i) $u_\beta$ is $\omega$-strongly convex; (ii) $u_\beta$ is $L_u$-smooth, i.e., $|u_\beta(y) - (u_\beta(x) + u_\beta'(x)(y-x))| \le L_u/2 (x-y)^2$ for any $x, y \in \mathbb{R}$; (iii) $u_\beta'$ is continuously differentiable and the collection of random variables $\{u_\beta''(\rvar{x} - \lambda) : \lambda \in \mathbb{R}\}$ is uniformly integrable; and (iv) the random variables $\rvar{x}_{t}$ are sub-Gaussian with parameter $\sigma$; then
    $$\mathbb{P}\left[ |\hat{\rho}_{n} - \rho(\rvar{x}_1)| \ge z \right] \le 2\exp\left(-\frac{n \omega^2 z^2}{8 L_u^2 \sigma^2}\right).$$
    We have that $u_\beta'(x) = \exp(\beta x)$ and $u_\beta''(x) = \beta \exp(\beta x)$. From Lemma \ref{lemma:estimator_properties} it holds that $\lambda^* \in [-R,R]$ since the random variables $\rvar{x}_{t}$ are bounded. Thus, we restrict our attention to $\lambda \in [-R,R]$ and, hence, $\rvar{x}_{t} - \lambda \in [-2R, 2R]$. Thus, we now show that the assumptions (i) to (iv) are satisfied when we restrict the domain of $u_\beta$ to the $[-2R, 2R]$ interval. Regarding (i), it holds that $u_\beta''(x) = \beta \exp(\beta x) \ge \beta \exp(\beta (-2R)) $ for any $x \in [-2R, 2R]$; thus, with $\omega = \beta \exp(\beta (-2R))$, $u_\beta$ is $\omega$-strongly convex over the $[-2R, 2R]$ interval. Regarding (ii), we have, for any $x,y \in [-2R, 2R]$, $$u_\beta(y) = u_\beta(x) + u_\beta'(x)(y-x) + \frac{u_\beta''(x)}{2} (y-x)^2.$$ Thus, $$|u_\beta(y) - u_\beta(x) + u_\beta'(x)(y-x) | = \frac{|u_\beta''(x)|}{2} (y-x)^2 \le  \frac{\beta \exp(\beta 2R)}{2} (y-x)^2,$$ and, hence, letting $L_u = \beta \exp(\beta 2R)$ yields a valid smoothness constant over the $[-2R, 2R]$ interval. Assumption (iii) holds given that $u_\beta'(x) = \exp(\beta x)$ and $u_\beta''(x) = \beta \exp(\beta x)$. Finally, assumption (iv) is also verified because, since the random variables are bounded in $[-R,R]$ it implies that they are also sub-Gaussian with $\sigma = R^2$.
    
    \noindent The one-sided concentration bounds can be inferred from the fact that
    \begin{align*}
        \mathbb{P}\left[ \hat{\rho}_{n} - \rho(\rvar{x}_1) \ge z \right] &\le \mathbb{P}\left[ \left\{ \hat{\rho}_{n} - \rho(\rvar{x}_1) \ge z \right\} \cup \left\{ \hat{\rho}_{n} - \rho(\rvar{x}_1) \le -z \right\} \right]  \le \mathbb{P}\left[ |\hat{\rho}_{n} - \rho(\rvar{x}_1)| \ge z \right], \\
        \mathbb{P}\left[ \hat{\rho}_{n} - \rho(\rvar{x}_1) \le -z \right] &\le \mathbb{P}\left[ \left\{ \hat{\rho}_{n} - \rho(\rvar{x}_1) \ge z \right\} \cup \left\{ \hat{\rho}_{n} - \rho(\rvar{x}_1) \le -z \right\} \right]  \le \mathbb{P}\left[ |\hat{\rho}_{n} - \rho(\rvar{x}_1)| \ge z \right].
    \end{align*}
\end{proof}

Given \eqref{eq:entropic_risk_concentration_bound_two_sided} and since $\lim_{n \rightarrow \infty} 2 \exp\left(-Cnz^2\right) = 0$ for any $z \in \mathbb{R}^+$ and $\beta > 0$, estimator $\hat{\rho}_{n}$ converges in probability to the true value $\rho(\rvar{x}_1)$, i.e., $\lim_{n \rightarrow \infty} \mathbb{P}\left[|\hat{\rho}_{n} - \rho(\rvar{x}_1)| \ge z \right] = 0$. Hence, $\hat{\rho}_{n}$ is a weakly consistent estimator. Note also that, since $\sum_{n=1}^\infty \mathbb{P} \left[|\hat{\rho}_{n} - \rho(\rvar{x}_1)| \ge z \right] \le \sum_{n=1}^\infty 2 \exp\left(-Cnz^2\right) < \infty$ for any $z \in \mathbb{R}^+$ and $\beta > 0$, from the Borel-Cantelli lemma it holds that $\mathbb{P}\left[ \lim_{n \rightarrow \infty} |\hat{\rho}_{n} - \rho(\rvar{x}_1)| \ge z \right] = 0$, i.e., the probability that infinitely many events $\{|\hat{\rho}_{n} - \rho(\rvar{x}_1)| \ge z\}$ occur is 0. Since this holds for any $z \in \mathbb{R}^+$, it implies that $\mathbb{P}\left[\lim_{n \rightarrow \infty} \hat{\rho}_{n} = \rho(\rvar{x}_1)\right] = 1$, i.e., $\hat{\rho}_{n}$ converges almost surely to $\rho(\rvar{x}_1)$. Thus, estimator $\hat{\rho}_{n}$ is also strongly consistent.

Estimator $\hat{\rho}_{n}$ is biased as $\mathbb{E}[\hat{\rho}_{n}] \neq \rho(\rvar{x}_1)$ in general. However, since $\hat{\rho}_{n}$ converges almost surely to $\rho(\rvar{x}_1)$, from the bounded convergence theorem \cite[Theo. 1.6.7]{durrett_2019} (the estimator is bounded as the random variables $\rvar{x}_t$ belong to the interval $[-R,R]$) it holds that $\lim_{n \rightarrow \infty} \mathbb{E}[\hat{\rho}_{n}] = \rho(\rvar{x}_1)$, i.e., the estimator is asymptotically unbiased.

\begin{lemma}[Polynomial concentration of the ERM estimator]
    \label{lemma:polynomial-concentration-erm-estimator}
    Let $(\rvar{x}_{1}, \ldots, \rvar{x}_{n})$ be a sequence of i.i.d. random variables taking values in $[-R,R]$. Let also $\mu = \lim_{n \rightarrow \infty} \mathbb{E}[\hat{\rho}_n] = \rho(\rvar{x}_1)$. Then, there exist constants $\theta > 1$, $\xi > 1 $, and $1/2 \le \eta < 1$ such that, for any $z \ge 1$ and $n \in \mathbb{N}$,
    \begin{align}
    \mathbb{P}\left[n\hat{\rho}_{n} - n\mu \ge n^{\eta} z \right] &\le \frac{\theta}{z^\xi}, \\
    \mathbb{P}\left[n\hat{\rho}_{n} - n\mu \le - n^{\eta} z \right] &\le \frac{\theta}{z^\xi}.
    \end{align}
\end{lemma}
\begin{proof}
    It holds, for any $z \ge 1$ and $n \in \mathbb{N}$, that
    \begin{align*}
        \mathbb{P}\left[\hat{\rho}_{n} - \mu \ge n^{\eta - 1} z \right] &\overset{(a)}{=} \mathbb{P}\left[\hat{\rho}_{n} - \rho(\rvar{x}_1) \ge n^{\eta - 1} z \right]\\
        &\overset{(b)}{\le} 2 \text{exp}\left(- C (n^{\eta-1}z)^2 n\right) \\
        &= 2 \text{exp}\left(-C n^{2\eta-1} z^2 \right) \\
        &\overset{(c)}{\le} 2 \text{exp}\left(-Cz^2 \right) \\
        &\overset{(d)}{\le} 2 \left(\frac{\xi}{2 C}\right)^{\xi/2} \text{exp}(-\xi/2) z^{-\xi}
    \end{align*}
    where (a) follows from (i), (b) follows from the fact that $\mathbb{P}\left[\hat{\rho}_{n} - \rho(\rvar{x}_1) \ge z' \right] \le 2 \exp\left(-C(z')^2n\right)$ for any $z'>0$ (from \eqref{eq:entropic_risk_concentration_bound_1}) and thus it also holds by letting $z' = n^{\eta - 1} z$ for any $z \ge 1$, $n \in \mathbb{N}$, and $1/2 \le \eta < 1$. Step (c) holds because $1/2 \le \eta < 1$ and, therefore, $n^{2\eta-1} \ge 1$ for any $n \in \mathbb{N}$. Step (d) holds because (for any $z \ge 1$ and $\xi > 1$)
    \begin{equation*}
        z^\xi \text{exp}\left(- C z^2\right) \le \sup_{z \in \mathbb{R}} \left\{ z^\xi \text{exp}\left(- C z^2\right) \right\}= \left(\frac{\xi}{2 C}\right)^{\xi/2} \text{exp}(-\xi/2).
    \end{equation*}
    Therefore, the polynomial concentration bound is valid by picking any $\xi > 1$ and $1/2 \le \eta < 1$, and letting $\theta = \max\left\{1, 2 \left(\frac{\xi}{2 C}\right)^{\xi/2}\text{exp}(-\xi/2)\right\}$. Similar steps can be followed to prove the other direction using \eqref{eq:entropic_risk_concentration_bound_2}.
\end{proof}

\section{Proof of Theorem~\ref{theo:non-stationary-bandit-results}}
\label{appendix:proof_theo_non_stat_bandit}
We split the proof in two parts. The first part deals with the convergence result. The second part deals with the concentration result. To simplify the notation, quantities related to the optimal arm $i^*$ are simply denoted with superscript $*$, e.g., $T_{i^*}(n) = T^*(n)$.

\paragraph{Proof of part 1 of Theorem~\ref{theo:non-stationary-bandit-results} (convergence result):}
We start by presenting the following Lemmas, which are due to \citet{shah_2020}.

\begin{lemma} \label{lemma:lemma-1-shah-paper}
    Let $i \in \{1, \ldots, K\} \setminus i^*$ be a sub-optimal arm and 
    $$A_i(t) = \min_{u\in \mathbb{N}} \left\{\frac{u^\eta \left(\frac{\theta}{t^{-\alpha}}\right)^{1/\xi}}{u} \le \frac{\Delta_i}{2}\right\} = \left\lceil \left( \frac{2}{\Delta_i} \theta^{1/\xi} t^{\alpha/\xi}\right)^{\frac{1}{1-\eta}} \right\rceil.$$
    For each $s \in \mathbb{N}$ and $t \in \mathbb{N}$ such that $A_i(t) \le s \le t$ we have
    $$\mathbb{P}[\hat{\rho}_{i,s} - b_{t,s} < \mu^* ] \le t^{-\alpha}.$$
\end{lemma}
\begin{proof}
    The proof of this result is similar to the proof of Lemma 1 in \cite{shah_2020}. The key differences are that: (i) our parameter $\theta$ corresponds to parameter $\beta$ in \cite{shah_2020}; and (ii) \cite{shah_2020} consider rewards whereas we consider costs and, thus, some quantities (e.g., the exploration bonus $b_{t,s}$ and $\Delta_i$) are symmetrically defined (hence, most of the inequalities in the proof should be flipped but the same sequence of reasoning is valid).
\end{proof}

\begin{lemma} \label{lemma:lemma-2-shah-paper}
    Let $i \in \{1, \ldots, K\} \setminus i^*$ be a sub-optimal arm. It holds that
    $$\EE[]{T_i(n)} \le \left( \frac{2 \theta^{1/\xi}}{\Delta_i}\right)^\frac{1}{1-\eta} n^{\frac{\alpha}{\xi(1-\eta)}} + \frac{2}{\alpha - 2} + 1.$$
\end{lemma}
\begin{proof}
    The proof of this result is similar to the proof of Lemma 2 in \cite{shah_2020}. The key differences are that: (i) our parameter $\theta$ corresponds to parameter $\beta$ in \cite{shah_2020}; and (ii) \cite{shah_2020} consider rewards whereas we consider costs and, thus, some quantities (e.g., the exploration bonus $b_{t,s}$ and $\Delta_i$) are symmetrically defined (hence, most of the inequalities in the proof should be flipped but the same sequence of reasoning is valid).
\end{proof}

\begin{proof}
    To prove the first part of the Theorem~\ref{theo:non-stationary-bandit-results} (convergence result), we start by noting that
    \begingroup
    \allowdisplaybreaks
    \begin{equation}
        |\mu^* - \EE[]{\bar{\rho}_n}| \le |\mu^* - \mu^*_n| + |\mu^*_n - \EE[]{\bar{\rho}_n}|. \label{eq:non-stat-bandit-theorem-appendix-proof-decomposition}
    \end{equation}
    \endgroup

    \noindent Regarding the term $|\mu^* - \mu^*_n|$ in \eqref{eq:non-stat-bandit-theorem-appendix-proof-decomposition}, we have that
    \begingroup
    \allowdisplaybreaks
    \begin{align*}
        |\mu^* - \mu^*_n| &= |\mu^* -  \EE[]{\hat{\rho}_{i^*,n}}| \\
        &\le \EE[]{ | \hat{\rho}_{i^*,n} - \mu^* | } \\
        &= \frac{1}{n} \EE[]{ | n\hat{\rho}_{i^*,n} - n\mu^* | } \\
        &\overset{(a)}{=} \frac{1}{n} \int_{0}^\infty \mathbb{P}\left[ | n\hat{\rho}_{i^*,n} - n\mu^* | > u \right] du \\
        &= \frac{1}{n} \int_{0}^{n^\eta} \mathbb{P}\left[ | n\hat{\rho}_{i^*,n} - n\mu^* | > u \right] du + \frac{1}{n} \int_{n^\eta}^\infty \mathbb{P}\left[ | n\hat{\rho}_{i^*,n} - n\mu^* | > u \right] du \\
        &\overset{(b)}{\le} \frac{1}{n} \int_{0}^{n^\eta} 1 du + \frac{2 \theta }{n} n^{\eta\xi} \int_{n^\eta}^\infty u^{-\xi} du \\
        &= \frac{1}{n} n^\eta + \frac{2 \theta}{n} n^{\eta\xi} \frac{(n^\eta)^{1-\xi}}{\xi - 1} \\
        &= n^{\eta - 1} \left(1 + \frac{2\theta}{\xi - 1} \right),
    \end{align*}
    \endgroup
    where (a) holds since $| n\hat{\rho}_{i^*,n} - n\mu^* |$ is non-negative. Step (b) follows from the fact that, under Assumption~\ref{assumption:non-stationary-bandit-assumption}, we can infer that $\mathbb{P}\left[ | n\hat{\rho}_{i^*,n} - n\mu^* | > n^\eta z \right] \le \frac{2\theta}{z^\xi}$, for any $z \ge 1$ and $n \in \mathbb{N}$. Therefore, letting $z = u/n^\eta$ (note that $u \ge n^\eta$ since we split the integral and thus $z \ge 1$), it holds that $\mathbb{P}\left[ | n\hat{\rho}_{i^*,n} - n\mu^* | > n^\eta z \right] \le 2\theta \left(\frac{n^\eta}{u}\right)^\xi$. \\
    
    \noindent Regarding the term $|\mu^*_n - \EE[]{\bar{\rho}_n}|$ in \eqref{eq:non-stat-bandit-theorem-appendix-proof-decomposition} we have that
    \begingroup
    \allowdisplaybreaks
    \begin{align}
        n|\mu^*_n - \EE[]{\bar{\rho}_n}| &= \left| n \EE[]{\hat{\rho}_{i^*,n}} - \EE[]{\sum_{i=1}^K T_i(n) \hat{\rho}_{i, T_i(n)}} \right| \nonumber \\
        &\le \left| n \EE[]{\hat{\rho}_{i^*,n}} - \EE[]{T^*(n) \hat{\rho}_{i^*, T^*(n)}} \right| + \left|\EE[]{\sum_{i=1, i \neq i^*}^K T_i(n) \hat{\rho}_{i, T_i(n)}} \right|. \label{eq:non_stat_bandit_proof_eq1}
    \end{align}
    \endgroup
    The first term in \eqref{eq:non_stat_bandit_proof_eq1} can be upper-bounded as follows:
    \begingroup
    \allowdisplaybreaks
    \begin{align*}
        \left| n \EE[]{\hat{\rho}_{i^*,n}} - \EE[]{T^*(n) \hat{\rho}_{i^*, T^*(n)}} \right| &= \left| \EE[]{n\hat{\rho}_{i^*,n} - T^*(n) \hat{\rho}_{i^*, T^*(n)} } \right| \\
        &\overset{(a)}{=} \left| \mathbb{E} \left[ n \min_{\lambda \in \mathbb{R}} \{h_n(\lambda)\} - T^*(n) \min_{\lambda \in \mathbb{R}} \{h_{T^*(n)}(\lambda)\} \right] \right| \\
        &= \left| \mathbb{E} \left[ \min_{\lambda \in \mathbb{R}} \{n h_n(\lambda)\} - \min_{\lambda \in \mathbb{R}} \{ T^*(n) h_{T^*(n)}(\lambda)\} \right] \right| \\
        &\le \mathbb{E} \left[ \left|\min_{\lambda \in \mathbb{R}} \{n h_n(\lambda)\} - \min_{\lambda \in \mathbb{R}} \{ T^*(n) h_{T^*(n)}(\lambda)\} \right| \right] \\
        &\overset{(b)}{\le} \mathbb{E} \left[ (n- T^*(n)) \left( R + \frac{\exp(2 \beta R) - 1}{\beta} \right) \right] \\
        &= \left( R + \frac{\exp(2 \beta R) - 1}{\beta} \right) \mathbb{E} \left[ (n-T^*(n)) \right]\\
        &= \left( R + \frac{\exp(2 \beta R) - 1}{\beta} \right) \mathbb{E} \left[ \sum_{i=1, i \neq i^*}^K T_i(n)\right], \\
    \end{align*}
    \endgroup
    where in (a) we let $h_n(\lambda) = \lambda + \frac{1}{n} \sum_{t=1}^{n} u_\beta(\rvar{x}_{i^*, t} - \lambda)$. For step (b), let $g(\lambda) = n h_n(\lambda) - T^*(n) h_{T^*(n)}(\lambda) = (n- T^*(n)) \lambda + \sum_{t=T^*(n) + 1}^{n} u_\beta(\rvar{x}_{i^*, t} - \lambda)$. Let also $\lambda^*_n = \argmin_\lambda n h_n(\lambda) = \argmin_\lambda h_n(\lambda)$ and $\lambda^*_{T^*(n)} = \argmin_\lambda T^*(n) h_{T^*(n)}(\lambda) = \argmin_\lambda h_{T^*(n)}(\lambda)$. We note that
    \begingroup
    \allowdisplaybreaks
    \begin{align*}
    \min_{\lambda \in \mathbb{R}} \{n h_n(\lambda)\} - \min_{\lambda \in \mathbb{R}} \{ T^*(n) h_{T^*(n)}(\lambda)\} &\le n h_n(\lambda^*_{T^*(n)}) - \min_\lambda \{ T^*(n) h_{T^*(n)}(\lambda)\} \\
    &= n h_n(\lambda^*_{T^*(n)}) - T^*(n) h_{T^*(n)}(\lambda^*_{T^*(n)}) \\
    &= g(\lambda^*_{T^*(n)})
    \end{align*}
    \endgroup
    Also, it holds that
    \begingroup
    \allowdisplaybreaks
    \begin{align*}
    \min_{\lambda \in \mathbb{R}} \{n h_n(\lambda)\} - \min_{\lambda \in \mathbb{R}} \{ T^*(n) h_{T^*(n)}(\lambda)\} &\ge \min_{\lambda \in \mathbb{R}} \{n h_n(\lambda)\} - T^*(n) h_{T^*(n)}(\lambda_n^*) \\
    &= n h_n(\lambda^*_{n}) - T^*(n) h_{T^*(n)}(\lambda_n^*) \\
    &= g(\lambda^*_{n}).
    \end{align*}
    \endgroup
    Hence $$g(\lambda_n^*) \le \min_{\lambda \in \mathbb{R}} \{n h_n(\lambda)\} - \min_{\lambda \in \mathbb{R}} \{ T^*(n) h_{T^*(n)}(\lambda)\} \le g(\lambda^*_{T^*(n)}),$$
    implying that (since $a \le x \le b$ implies that $x \le \max\{|a|, |b|\}$)
    \begingroup
    \allowdisplaybreaks
    \begin{align*}
    |\min_{\lambda \in \mathbb{R}} \{n h_n(\lambda)\} - \min_{\lambda \in \mathbb{R}} \{ T^*(n) h_{T^*(n)}(\lambda)\} | &\le \max\{|g(\lambda^*_n)|, |g(\lambda^*_{T^*(n)})|\} \\
    &\overset{(a)}{\le} (n- T^*(n)) R + \sum_{t=T^*(n) + 1}^{n} \frac{\exp(2 \beta R) - 1}{\beta} \\
    &= (n- T^*(n)) \left( R + \frac{\exp(2 \beta R) - 1}{\beta} \right).
    \end{align*}
    \endgroup
    In (a) we note that $\lambda^*_{T^*(n)} \in [-R,R]$ and $\lambda^*_{n} \in[-R,R]$ from point 2 in Lemma~\ref{lemma:estimator_properties}, as well as $\rvar{x}_{i^*,t} \in [-R,R]$. Thus, for any $\lambda \in [-R,R]$,
    \begingroup
    \allowdisplaybreaks
    \begin{align*}
    |g(\lambda)| &\le (n- T^*(n))|\lambda| + \frac{1}{\beta} \sum_{t=T^*(n) + 1}^{n} \left|\exp(\beta (\rvar{x}_{i^*,t} - \lambda) ) - 1\right| \\
    &\le (n- T^*(n))R + \frac{1}{\beta}\sum_{t=T^*(n) + 1}^{n} \left(\exp(2 \beta R) - 1\right),
    \end{align*}
    \endgroup
    where the first inequality follows from the triangular inequality and the second from the fact that: (i) $|\lambda|\le R$; and (ii) $\beta (\rvar{x}_{i^*,t} - \lambda) \in [-2\beta R, 2\beta R]$ and $|\exp(x) - 1| \le \exp(2 \beta R) - 1$ for any $x \in [-2\beta R, 2\beta R]$.\\
    
    \noindent The second term in \eqref{eq:non_stat_bandit_proof_eq1} can be upper-bounded as follows:
    \begingroup
    \allowdisplaybreaks
    \begin{equation*}
     \left| \EE[]{ \sum_{i=1, i \neq i^*}^K T_i(n) \hat{\rho}_{i, T_i(n)} } \right| \le \EE[]{ \sum_{i=1, i \neq i^*}^K \left| T_i(n) \hat{\rho}_{i, T_i(n)} \right| } \overset{(a)}{\le} R \EE[]{ \sum_{i=1, i \neq i^*}^K T_i(n) }, \\
    \end{equation*}
    \endgroup
    where (a) follows from point 3 in Lemma~\ref{lemma:estimator_properties}.
    Therefore, it holds that
    \begin{align*}
       |\mu^*_n - \EE[]{\bar{\rho}_n}| &\le \frac{ \left( 2R + \frac{\exp(2 \beta R) - 1}{\beta} \right) \EE[]{\sum_{i=1, i \neq i^*}^K T_i(n)}}{n} \\
        &\overset{(a)}{\le} \frac{ \left( 2R + \frac{\exp(2 \beta R) - 1}{\beta} \right) (K-1) \left( \left( \frac{2 \theta^{1/\xi}}{\Delta_\text{min}} \right)^{\frac{1}{1-\eta}} n^{\frac{\alpha}{\xi(1-\eta)}} + \frac{2}{\alpha - 2} + 1\right)}{n},
    \end{align*}
    where (a) follows from Lemma~\ref{lemma:lemma-2-shah-paper}. \\
    
    \noindent To obtain the final expression, we replace the upper bounds we derived for terms $|\mu^* - \mu^*_n|$ and $|\mu^*_n - \EE[]{\bar{\rho}_n}|$ into \eqref{eq:non-stat-bandit-theorem-appendix-proof-decomposition}, obtaining
    \begin{align*}
        |\mu^* - \EE[]{\bar{\rho}_n}| \le n^{\eta - 1} \left(1 + \frac{2\theta}{\xi - 1} \right) + \frac{ \left( 2R + \frac{\exp(2 \beta R) - 1}{\beta} \right) (K-1) \left( \left( \frac{2 \theta^{1/\xi}}{\Delta_\text{min}} \right)^{\frac{1}{1-\eta}} n^{\frac{\alpha}{\xi(1-\eta)}} + \frac{2}{\alpha - 2} + 1\right)}{n}.
    \end{align*}

    \noindent The rate of convergence can be simplified to $\O\left(\exp(2\beta R) n^{\frac{\alpha}{\xi(1-\eta)}-1}\right)$ since the first term in the upper bound above is $\O(n^{\eta-1})$ and the second term is $\O\left(\exp(2\beta R) n^{\frac{\alpha}{\xi(1-\eta)} - 1}\right)$. Given that $\xi\eta(1-\eta) \le \alpha$, we have that $\eta \le \frac{\alpha}{\xi(1-\eta)}$ and thus $\eta -1 \le \frac{\alpha}{\xi(1-\eta)} - 1$; since both exponents are negative (note that $\alpha < \xi(1-\eta)$), we have that the decay rate is governed by the slowest decaying term, $\O\left(\exp(2\beta R) n^{\frac{\alpha}{\xi(1-\eta)}-1}\right)$.
\end{proof}    

\paragraph{Proof of the second part of Theorem~\ref{theo:non-stationary-bandit-results} (concentration result):}
Our proof builds on the concentration result of Theorem 3 in \citet{shah_2020}, requiring key adaptations because the ERM estimator \eqref{eq:entropic_risk_estimator} differs from the mean estimator analyzed in \citet{shah_2020}. \\

\noindent Let
\begin{equation}
    A(t) = \max_{i \in \{1, \ldots, K\}} A_i(t) = \left\lceil \left( \frac{2}{\Delta_\text{min}} \theta^{1/\xi} t^{\alpha/\xi}\right)^{\frac{1}{1-\eta}} \right\rceil. \label{eq:A(t)_def}
\end{equation}

\noindent Replacing $\theta$ with any larger number still makes concentration bounds \eqref{eq:bandit_concentration_assumption} valid. Hence, we let $\theta$ be large enough so that $\frac{2 \theta^{1/\xi}}{\Delta_\text{min}} > 1$. Let also
\begin{equation}
    N_p = \min\{t \in \mathbb{N} : t \ge A(t)\}. \label{eq:Np_definition}
\end{equation}

\noindent We state the following Lemma, akin to Lemma 3 in \cite{shah_2020}.

\begin{lemma} \label{lemma:non-stationary-bandit-concentration-lemma}
    For any $n \ge N_p$ and $x \ge 1$, let $r_0 = n^\eta + 2 \left( R + \frac{\exp(2 \beta R) - 1}{2\beta} \right) (K-1) \left( 3 + A(n)\right)$. It holds that
    \begin{align}
    \mathbb{P}\left[n\bar{\rho}_{n} - n\mu^* \ge r_0 x \right] &\le \frac{\theta}{x^\xi} + \frac{2(K-1)}{(\alpha -1)\left((1+ A(n))x\right)^{\alpha - 1}}, \label{eq:non-stationary-bandit-concentration-lemma-eq1} \\
    \mathbb{P}\left[n\bar{\rho}_{n} - n\mu^* \le - r_0 x \right] &\le \frac{\theta}{x^\xi} + \frac{2(K-1)}{(\alpha -1)\left((1+ A(n))x\right)^{\alpha - 1}} \label{eq:non-stationary-bandit-concentration-lemma-eq2}.
    \end{align}
\end{lemma}
\begin{proof}
    The proof of this Lemma builds on the proof of Lemma 3 in \cite{shah_2020}. The proof of \cite{shah_2020} had to be adapted to replace the mean estimator with the ERM estimator \eqref{eq:entropic_risk_estimator}. The difference in estimator particularly impacts the first part of the proof where the objective is to upper bound $\mathbb{P}\left[n\mu^* - n \bar{\rho}_n \ge r_0 x\right]$, i.e., the steps to reach \eqref{eq:non_stat_bandit_proof_eq3}. From that point onwards, the proof majorly follows that of Lemma 3 in \cite{shah_2020}. \\

    \noindent We start by proving the upper bound \eqref{eq:non-stationary-bandit-concentration-lemma-eq2}. It holds that
    \begingroup
    \allowdisplaybreaks
    \begin{equation}
     n\mu^* - n \bar{\rho}_n = n\mu^* - \sum_{i=1}^K T_i(n) \hat{\rho}_{i, T_i(n)} = n\mu^* - T^*(n) \hat{\rho}_{i^*, T^*(n)} - \sum_{i=1, i \neq i^*}^K T_i(n) \hat{\rho}_{i, T_i(n)} \label{eq:non_stat_bandit_proof_eq2}
    \end{equation}
    \endgroup 
    
    \noindent To upper bound the term $- T^*(n) \hat{\rho}_{i^*, T^*(n)}$ in \eqref{eq:non_stat_bandit_proof_eq2} we let $h_\X(\lambda) = |\X| \lambda + \sum_{t \in \X} u_\beta(\rvar{x}_{i^*,t } - \lambda)$, where $\X$ denotes an arbitrary set of timesteps. Let also $\A = \{1, \ldots, T^*(n) \}$, $\B = \{T^*(n) + 1, \ldots, n \}$ and $\lambda^*_\A = \argmin_\lambda h_\A(\lambda)$. Now, it holds that
    \begingroup
    \allowdisplaybreaks
    \begin{align*}
        n \hat{\rho}_{i^*, n} &= \min_\lambda h_{\A \cup \B} (\lambda) \\
        &\le h_{\A \cup \B} (\lambda_\A^*) \\
        &= (n - T^*(n)) \lambda_\A^* + \sum_{t \in \B} u_\beta(\rvar{x}_{i^*, t} - \lambda_\A^*) + T^*(n)  \lambda_\A^* + \sum_{t \in \A} u_\beta(\rvar{x}_{i^*, t} - \lambda_\A^*) \\
        &= (n - T^*(n)) \lambda_\A^* + \sum_{t = T^*(n) + 1}^{n} u_\beta(\rvar{x}_{i^*, t} - \lambda_\A^*) + \min_\lambda h_\A(\lambda) \\
        &= (n - T^*(n)) \lambda_\A^* + \sum_{t = T^*(n) + 1}^{n} u_\beta(\rvar{x}_{i^*, t} - \lambda_\A^*) + T^*(n) \hat{\rho}_{i^*, T^*(n)} \\
        &\overset{(a)}{\le} (n - T^*(n)) R + \sum_{t = T^*(n) + 1}^{n} \frac{\exp(2 \beta R) - 1}{\beta} + T^*(n) \hat{\rho}_{i^*, T^*(n)} \\
        &= \left( R + \frac{\exp(2 \beta R) - 1}{\beta} \right) (n - T^*(n)) + T^*(n) \hat{\rho}_{i^*, T^*(n)} \\
        &= \left( R + \frac{\exp(2 \beta R) - 1}{\beta} \right) \sum_{i=1, i \neq i^*}^K T_i(n) + T^*(n) \hat{\rho}_{i^*, T^*(n)},
    \end{align*}
    \endgroup 
    where (a) follows from point 2 in Lemma \ref{lemma:estimator_properties}, the fact that $\rvar{x}_{i^*,t} \in [-R,R]$, and by noting that $\frac{1}{\beta} \left(\exp(\beta x) - 1\right)$ is monotonically increasing in $x$. Equivalently,
    $$-T^*(n) \hat{\rho}_{i^*, T^*(n)} \le -n \hat{\rho}_{i^*, n} + \left( R + \frac{\exp(2 \beta R) - 1}{\beta} \right) \sum_{i=1, i \neq i^*}^K T_i(n).$$

    \noindent The term $- \sum_{i=1, i \neq i^*}^K T_i(n) \hat{\rho}_{i, T_i(n)}$ in \eqref{eq:non_stat_bandit_proof_eq2} can be upper bounded as
    \begingroup
    \allowdisplaybreaks
    \begin{align*}
        - \sum_{i=1, i \neq i^*}^K T_i(n) \hat{\rho}_{i, T_i(n)} &\le \left|\sum_{i=1, i \neq i^*}^K T_i(n) \hat{\rho}_{i, T_i(n)} \right| \\
        &\le \sum_{i=1, i \neq i^*}^K |T_i(n) \hat{\rho}_{i, T_i(n)}| \\
        &\le \sum_{i=1, i \neq i^*}^K T_i(n) |\hat{\rho}_{i, T_i(n)}|\\
        &\overset{(a)}{\le} R \sum_{i=1, i \neq i^*}^K T_i(n)
    \end{align*}
    \endgroup 
    where (a) follows from point 3 in Lemma \ref{lemma:estimator_properties}.\\
    
    \noindent Hence, \eqref{eq:non_stat_bandit_proof_eq2} can be upper bounded as
    $$n\mu^* - n \bar{\rho}_n \le n\mu^* - n \hat{\rho}_{i^*, n} + 2\left( R + \frac{\exp(2 \beta R) - 1}{2\beta} \right) \sum_{i=1, i \neq i^*}^K T_i(n).$$

    \noindent Therefore, it holds that
    \begingroup
    \allowdisplaybreaks
    \begin{align}
        \mathbb{P}\left[n\mu^* - n \bar{\rho}_n \ge r_0 x\right] &\le \mathbb{P}\left[n\mu^* - n \hat{\rho}_{i^*, n} + 2 \left( R + \frac{\exp(2 \beta R) - 1}{2\beta} \right) \sum_{i=1, i \neq i^*}^K T_i(n) \ge r_0 x \right] \nonumber \\
        &\overset{(a)}{\le} \mathbb{P}\left[n\mu^* - n \hat{\rho}_{i^*, n} \ge n^\eta x\right] + \sum_{i=1, i \neq i^*}^K \mathbb{P}\left[T_i(n) \ge (3+A(n))x\right] \label{eq:non_stat_bandit_proof_eq3}.
    \end{align}
    \endgroup
    Step (a) follows from a union bound and the fact that (we let $L = R + \frac{1}{2\beta} \left(\exp(2 \beta R) - 1\right) $ below)
    $$\underbrace{\left\{ n\mu^* - n \hat{\rho}_{i^*, n} + 2L \sum_{i \neq i^*} T_i(n) \ge r_0 x \right\} }_{=E} \subseteq \underbrace{\{n\mu^* - n \hat{\rho}_{i^*, n} \ge n^\eta x\} }_{=E_0} \cup \bigcup_{i \neq i^*} \underbrace{\{T_i(n) \ge (3+A(n))x\}}_{=E_i},$$
    where we let $E$, $E_0$ and $E_i$ be the different events defined above. To see why the relation between the different events holds we note that the relation above is logically equivalent to $E \implies E_0 \vee \bigvee_{i \neq i^*} E_i$, also equivalent to $\neg \left( E_0 \vee \bigvee_{i \neq i^*} E_i\right) = \neg E_0 \wedge \bigwedge_{i \neq i^*}\neg E_i  \implies \neg E$. This last implication is true since, if $\neg E_0$ (i.e., $n\mu^* - n \hat{\rho}_{i^*, n} < n^\eta x$) and, for all $i \neq i^*$, $\neg E_i$ (i.e., $T_i(n) < (3+A(n))x$) we have that
    $$n\mu^* - n \hat{\rho}_{i^*, n} + 2L \sum_{i \neq i^*} T_i(n) < n^\eta x + 2L\sum_{i \neq i^*} (3+A(n))x = r_0 x.$$
    The first term in \eqref{eq:non_stat_bandit_proof_eq3} can be bounded using \eqref{eq:bandit_concentration_assumption}, yielding
    \begin{equation}
        \mathbb{P}\left[n\mu^* - n \hat{\rho}_{i^*, n} \ge n^\eta x\right] \le \frac{\theta}{x^\xi}. \label{eq:non_stat_bandit_proof_eq6}
    \end{equation}
    
    \noindent To bound the second term in  \eqref{eq:non_stat_bandit_proof_eq3} we proceed in a similar fashion to the proof of Lemma 3 in \cite{shah_2020}. Fix $n$ and a sub-optimal arm $i$. Let $u \in \mathbb{N}$ be an integer satisfying $A(n) \le u \le n$. For any $\tau \in \mathbb{R}$, consider the following two events:
    \begin{align*}
        E_1 &= \{ \hat{\rho}_{i,u} - b_{t,u} \ge \tau, \forall t \in \{u, \ldots, n\}\} \\
        E_2 &= \{ \hat{\rho}_{i^*,s} - b_{u+s,s} < \tau, \forall s \in \{1, \ldots, n - u\} \}
    \end{align*}
    As a first step, we want to show that $E_1 \cap E_2 \implies T(i) \le u$. Note that $b_{t,s}$ is non-decreasing with respect to $t$. For each $s$ such that $1 \le s \le n-u$ and $t$ such that $u+s \le t \le n$, it holds that
    \begin{equation}
        \hat{\rho}_{i^*,s} - b_{t,s} \le \hat{\rho}_{i^*,s} - b_{u+s,s} < \tau \le \hat{\rho}_{i,u} - b_{t,u} \label{eq:non_stat_bandit_proof_eq4}
    \end{equation}
    This implies that $T_i(n) \le u$. To prove such an inequality, suppose the opposite, i.e., $T_i(n) > u$ and let $t' = \min\{t: t \le n, T_i(t) = u\}$. It holds that, $T^*(t') \le t' - u$ because $t' = T_i(t') + T^*(t') + \sum_{j \neq i, j \neq i^*}T_j(t') \ge T_i(t') + T^*(t') = u + T^*(t') $. For any $t \ge t'$ it also holds that $T^*(t) \le t - u$, or equivalently, $u + T^*(t) \le t$. Since $t \le n$, it also holds that $T^*(t) \le n - u$. Therefore, letting $s = T^*(t)$ in \eqref{eq:non_stat_bandit_proof_eq4} (note that such a choice verifies $1 \le s \le n-u$ and $u+s \le t \le n$) implies that $\hat{\rho}_{i^*,T^*(t)} - b_{t,T^*(t)} < \hat{\rho}_{i,u} - b_{t,u}$, i.e., arm $i^*$ will be selected from timestep $t$ onwards and, hence, arm $i$ will never be played the $(u+1)$-th time, contradicting $T_i(n) > u$. Hence, it holds that $T_i(n) \le u$. \\

    \noindent Since $E_1 \cap E_2 \implies T_i(n) \le u$, we have that
    \begin{align*}
        \{T_i(n) &> u\} \subset (E_1^c \cup E_2^c) \\
        &= \{ \exists t \in \{u, \ldots, n \} \text{ s.t. } \hat{\rho}_{i,u} - b_{t,u} < \tau \} \cup \{ \exists s \in \{1, \ldots, n - u\} \text{ s.t. } \hat{\rho}_{i^*,s} - b_{u+s,s} \ge \tau \}.
    \end{align*}

    \noindent Using a union bound it holds, for any $u$ satisfying $A(n) \le u \le n$, that
    $$\mathbb{P}\left[T_i(n) > u\right] \le \sum_{t=u}^n \mathbb{P} [\hat{\rho}_{i,u} - b_{t,u} < \tau] + \sum_{s=1}^{n-u}\mathbb{P} [\hat{\rho}_{i^*,s} - b_{u+s,s} \ge \tau].$$

    \noindent We note that the upper bound above is also valid in case $A(n) \le n < u$ since $\mathbb{P}\left[T_i(n) > u\right] = 0$. Hence, we can choose $u$ and $\tau$ as long as $u \ge A(n)$. Let $u = \lfloor (1+A(n)) x \rfloor + 1 $ and $\tau = \mu^*$. Since $A_i(n) \le A(n) \le u$ it holds that
    \begin{equation*}
        \sum_{t=u}^n \mathbb{P} [\hat{\rho}_{i,u} - b_{t,u} < \mu^*] \overset{(a)}{\le} \sum_{t=u}^n t^{-\alpha} \overset{(b)}{\le} \frac{\left((1+A(n))x\right)^{1-\alpha}}{\alpha - 1},
    \end{equation*}
    where (a) follows from Lemma~\ref{lemma:lemma-1-shah-paper} and (b) from the proof of Lemma 3 in \cite{shah_2020}. It also holds that
    \begin{equation*}
        \sum_{s=1}^{n-u} \mathbb{P} [\hat{\rho}_{i^*,s} - b_{u+s,s} \ge \mu^*] \overset{(a)}{\le} \sum_{s=1}^{n-u} (u+s)^{-\alpha} \overset{(b)}{\le} \frac{\left((1+A(n))x\right)^{1-\alpha}}{\alpha - 1},
    \end{equation*}
    where (a) follows from Lemma~\ref{lemma:lemma-1-shah-paper} and (b) from the proof of Lemma 3 in \cite{shah_2020}. \\

    \noindent Combining both inequalities it holds that (noting that $(3 + A(n))x > \lfloor (1+ A(n))x\rfloor + 1$)
    \begin{equation}
        \mathbb{P}\left[T_i(n) \ge (3+A(n))x\right] \le \mathbb{P}\left[T_i(n) > u \right] \le \frac{2\left((1+A(n))x\right)^{1-\alpha}}{\alpha - 1}. \label{eq:non_stat_bandit_proof_eq5}
    \end{equation}

    \noindent Concluding, we can upper bound both terms in \eqref{eq:non_stat_bandit_proof_eq3} using \eqref{eq:non_stat_bandit_proof_eq6} and \eqref{eq:non_stat_bandit_proof_eq5}, yielding
    $$\mathbb{P}\left[n\mu^* - n \bar{\rho}_n \ge r_0 x\right] \le \frac{\theta}{x^\xi} + \sum_{i \neq i^*} \frac{2\left((1+A(n))x\right)^{1-\alpha}}{\alpha - 1} = \frac{\theta}{x^\xi} + \frac{2(K-1)}{(\alpha - 1) \left((1+A(n))x\right)^{\alpha-1}}.$$

    \noindent To prove the other upper bound \eqref{eq:non-stationary-bandit-concentration-lemma-eq1}, we start by noting that
    \begingroup
    \allowdisplaybreaks
    \begin{equation}
     n \bar{\rho}_n - n\mu^* = \sum_{i=1}^K T_i(n) \hat{\rho}_{i, T_i(n)} - n\mu^* = T^*(n) \hat{\rho}_{i^*, T^*(n)} - n\mu^* + \sum_{i=1, i \neq i^*}^K T_i(n) \hat{\rho}_{i, T_i(n)}. \label{eq:non_stat_bandit_proof_eq8}
    \end{equation}
    \endgroup 
    To upper bound the term $T^*(n) \hat{\rho}_{i^*, T^*(n)}$ in \eqref{eq:non_stat_bandit_proof_eq8} we let $h_\X(\lambda) = |\X| \lambda + \sum_{t \in \X} u_\beta(\rvar{x}_{i^*,t } - \lambda)$, where $\X$ denotes an arbitrary set of timesteps. Let also $\A = \{1, \ldots, T^*(n)\}$, $\B = \{T^*(n) + 1, \ldots, n\}$. Now, it holds that
    \begingroup
    \allowdisplaybreaks
    \begin{align*}
        n \hat{\rho}_{i^*, n} &= \min_\lambda h_{\A \cup \B} (\lambda) \\
        &= \min_\lambda \left\{ h_{\A} (\lambda) + h_{\B} (\lambda) \right\}\\
        &\ge \min_\lambda \left\{ h_{\A} (\lambda) \right\} + \min_\lambda \left\{ h_{\B} (\lambda) \right\} \\
        &= T^*(n) \hat{\rho}_{i^*, T^*(n)} + (n - T^*(n)) \lambda_\B^* + \sum_{t = T^*(n)+1}^{n} u_\beta(\rvar{x}_{i^*, t} - \lambda_\B^*),
    \end{align*}
    \endgroup
    Equivalently,
    \begingroup
    \allowdisplaybreaks
    \begin{align*}
    T^*(n) \hat{\rho}_{i^*, T^*(n)} &\le  n \hat{\rho}_{i^*, n} - \left( (n - T^*(n)) \lambda_\B^* + \sum_{t = T^*(n) + 1}^{n} u_\beta(\rvar{x}_{i^*, t} - \lambda_\B^*) \right)\\
    &\le  n \hat{\rho}_{i^*, n} + \left| (n - T^*(n)) \lambda_\B^* + \sum_{t = T^*(n) + 1}^{n} u_\beta(\rvar{x}_{i^*, t} - \lambda_\B^*) \right|\\
    &\le n \hat{\rho}_{i^*, n} + (n - T^*(n)) |\lambda_\B^*| + \sum_{t = T^*(n) + 1}^{n} | u_\beta(\rvar{x}_{i^*, t} - \lambda_\B^*)| \\
    &\overset{(a)}{\le} n \hat{\rho}_{i^*, n} + \left( R + \frac{\exp(2 \beta R)}{\beta} \right) \sum_{i=1, i \neq i^*}^K T_i(n),
    \end{align*}
    \endgroup
    where in (a) we followed similar steps as those used in the proof of upper bound \eqref{eq:non-stationary-bandit-concentration-lemma-eq2}. \\

    \noindent Term $\sum_{i=1, i \neq i^*}^K T_i(n) \hat{\rho}_{i, T_i(n)}$ in \eqref{eq:non_stat_bandit_proof_eq8} can be upper bounded as
    \begin{equation*}
    \sum_{i=1, i \neq i^*}^K T_i(n) \hat{\rho}_{i, T_i(n)} \overset{(a)}{\le} R \sum_{i=1, i \neq i^*}^K T_i(n)
    \end{equation*}
    where (a) follows from point 3 in Lemma~\ref{lemma:estimator_properties}. Hence, \eqref{eq:non_stat_bandit_proof_eq8} can be upper bounded as
    $$n \bar{\rho}_n - n\mu^* \le n \hat{\rho}_{i^*, n} -n \mu^* + 2 \left( R + \frac{\exp(2 \beta R) - 1}{2\beta} \right) \sum_{i=1, i \neq i^*}^K T_i(n),$$
    implying that 
    \begingroup
    \allowdisplaybreaks
    \begin{align*}
    \mathbb{P}\left[ n \bar{\rho}_n - n\mu^* \ge r_0 x\right] &\le \mathbb{P}\left[n \hat{\rho}_{i^*, n} -n \mu^* + 2 \left( R + \frac{\exp(2 \beta R) - 1}{2\beta} \right) \sum_{i=1, i \neq i^*}^K T_i(n) \ge r_0 x\right].
    \end{align*}
    \endgroup
    The proof can be completed by exploiting \eqref{eq:bandit_concentration_assumption} and following the same line of reasoning as we did to prove \eqref{eq:non-stationary-bandit-concentration-lemma-eq2}.
\end{proof}

\noindent To complete the proof of the concentration result we follow similar steps as those considered by \cite{shah_2020} in the proof of Theorem 3. To simplify the notation, let $R' =  R + \frac{\exp(2 \beta R) - 1}{2\beta}$. Let also $N'_p = \min\left\{ t \in \mathbb{N} : t \ge A(t), 2 R' A(t) \ge t^\eta + 2R'(4K-3)\right\}$, which is guaranteed to exist. Note that $N'_p \ge N_p$. For each $n \ge N'_p$,
\begingroup
\allowdisplaybreaks
\begin{align*}
2R' K \left( \frac{2}{\Delta_\text{min}} \theta^{1/\xi}\right)^{\frac{1}{1-\eta}} n^{\frac{\alpha}{\xi(1-\eta)}} &= 2 R' K \left[ \left( \frac{2}{\Delta_\text{min}} \theta^{1/\xi}\right)^{\frac{1}{1-\eta}} n^{\frac{\alpha}{\xi(1-\eta)}} + 1 - 1\right] \\
&\ge 2R' K A(n) - 2 R' K \\
&= 2 R' (K-1) A(n) + 2 R' A(n) - 2R'K \\
&\ge 2 R' (K-1) A(n) + n^\eta + 2R'(4K-3) - 2R'K \\
&= 2 R' (K-1) (A(n) + 3) + n^\eta = r_0.
\end{align*}
\endgroup

\noindent Now, from Lemma~\ref{lemma:non-stationary-bandit-concentration-lemma} we have, for any $n \ge N'_p$ and $x \ge 1$, that
\begingroup
\allowdisplaybreaks
\begin{align}
\mathbb{P}\left[ \bar{\rho}_n - n\mu^* \ge 2R' K \left( \frac{2}{\Delta_\text{min}} \theta^{1/\xi}\right)^{\frac{1}{1-\eta}} n^{\frac{\alpha}{\xi(1-\eta)}} x \right] &\le \mathbb{P}\left[ n \bar{\rho}_n - n\mu^* \ge r_0 x\right] \nonumber \\
&\le \frac{\theta}{x^\xi} + \frac{2(K-1)}{(\alpha-1)\left((1 + A(n))x\right)^{\alpha-1}} \nonumber \\
&\le \frac{2 \max\left(\theta, \frac{2(K-1)}{(\alpha - 1)(1 + A(N'_p))^{\alpha - 1}}\right)}{x^{\alpha - 1}}, \label{eq:non-stationary-bandit-concentration-proof-eq}
\end{align}
\endgroup
where we followed similar steps as those in the proof of Theo 3 in \cite{shah_2020}. Letting $c_1 = 2 R' K \left( \frac{2}{\Delta_\text{min} } \theta^{1/\xi} \right)^\frac{1}{1-\eta}$ and $z = c_1 x$ we have, for any $n \ge N'_p$ and $z \ge 1$, that 
$$\mathbb{P}\left[ n \bar{\rho}_n - n\mu^* \ge n^{\frac{\alpha}{\xi (1-\eta)}} z\right] \le \frac{2 c_1^{\alpha - 1} \max\left(\theta, \frac{2(K-1)}{(\alpha - 1)(1 + A(N'_p))^{\alpha - 1}}\right)}{z^{\alpha - 1}}. $$
The inequality holds because of \eqref{eq:non-stationary-bandit-concentration-proof-eq} in case $z > c_1$; if $1 \le z \le c_1$, we have that the right-hand side of the equation above is at least one and thus the inequality trivially holds. All that is left is to refine the constants such that we have an inequality that holds for any $n \in \mathbb{N}$, and not only for $n \ge N'_p$. The proof can be completed by following similar steps as those considered by \cite{shah_2020} and by noting that
$$|n \bar{\rho}_n - n \mu^*| \le |n \bar{\rho}_n| + | n \mu^*| = n|\bar{\rho}_n| + n| \mu^*| \le 2Rn, $$
where the last inequality follows from Lemma~\ref{lemma:estimator_properties}. In particular, letting $c_2 = 2R (N'_p - 1)^{1-\frac{\alpha}{\xi(1-\eta)}}$, it holds that
$$\mathbb{P}\left[ n \bar{\rho}_n - n\mu^* \ge n^{\eta'} z\right] \le \frac{\theta'}{z^{\xi'}},$$
with 
\begingroup
\allowdisplaybreaks
\begin{align*}
\eta' &= \frac{\alpha}{\xi(1-\eta)} \\
\xi' &= \alpha - 1 \\
\theta' &= \max\left\{c_2, 2 c_1^{\alpha - 1} \max\left(\theta, \frac{2(K-1)}{(\alpha - 1)(1 + A(N'_p))^{\alpha - 1}}\right) \right\}.
\end{align*}
\endgroup

\noindent The other direction follows similarly.

\section{Proof of Theorem \ref{theo:non-stationary-bandit-results-stream-estimator}}
\label{appendix:non-stationary-bandit-results-stream-estimator}

\noindent We begin with the following Lemma.
\begin{lemma} \label{lemma:non-optimal-arm-concentration-bound}
    For any non-optimal arm $i \neq i^*$, $x > 0$ and $n \ge N_p$, where $N_p$ is defined in \eqref{eq:Np_definition}, it holds that
    $$\mathbb{P}\left[T_i(n) \ge n^{\eta'} x \right] \le \max\left\{ \left( \frac{2}{\Delta_\text{min}} \theta^{1/\xi} \right)^{\frac{1}{1-\eta}} + 4, \frac{2^\alpha}{\alpha-1} \right\} x^{1-\alpha}, $$
    where $\eta' = \frac{\alpha}{\xi (1-\eta)}$ satisfying $1/2 \le \eta' < 1$.
\end{lemma}
\begin{proof}
    For any fixed $n \ge N_p$, we split the proof into two cases: (i) the case where $x \ge (A(n) + 3)/n^{\eta'}$; and (ii) the case where $x < (A(n) + 3)/n^{\eta'}$. \\

    \noindent For case (i), i.e., if $x \ge (A(n) + 3)/n^{\eta'}$, we know from \eqref{eq:non_stat_bandit_proof_eq5} that it holds, for any $x' \ge 1$ and $n \ge N_p$,
    \begin{equation}
        \mathbb{P}\left[T_i(n) > \lfloor (1+ A(n))x'\rfloor + 1 \right] \le \frac{2\left((1+A(n))x'\right)^{1-\alpha}}{\alpha - 1} \label{eq:appendix_T_i_upper_bound_eq} .
    \end{equation}
    Let $x^* = \frac{\lceil n^{\eta'} x \rceil - 2}{1 + A(n)}$. Since $x \ge (A(n) + 3)/n^{\eta'}$ implies that $\lceil n^{\eta'} x \rceil - 2 \ge 1 + A(n)$ it holds that $x^* \ge 1$. Hence, we have that
    \begingroup
    \allowdisplaybreaks
    \begin{align*}
        \mathbb{P}\left[T_i(n) \ge n^{\eta'} x \right] &= \mathbb{P}\left[T_i(n) \ge \lceil n^{\eta'} x \rceil \right] \\
        &\overset{(a)}{\le} \mathbb{P}\left[T_i(n) > \lfloor (1+ A(n))x^*\rfloor + 1 \right] \\
        &\overset{(b)}{\le} \frac{2\left((1+A(n))x^*\right)^{1-\alpha}}{\alpha - 1} \\
        &= \frac{2\left(\lceil n^{\eta'} x \rceil - 2\right)^{1-\alpha}}{\alpha - 1}.
    \end{align*}
    \endgroup
    Step (a) is true since $\mathbb{P}[\rvar{x} \ge a] \le \mathbb{P}[\rvar{x} > b] = \mathbb{P}[\rvar{x} \ge b + 1]$ if $a \ge b+1$ for integer-valued thresholds $a$ and $b$, and integer-valued random variable $\rvar{x}$ (from the monotonicity of tail probabilities). Letting $\rvar{x} = T_i(n)$, $a = \lceil n^\eta x \rceil$ and $b=\lfloor (1+ A(n))x^*\rfloor + 1$, it can be verified that $a \ge b+1$ (more precisely, it holds that $a = b+1$) and, hence, the step above is valid. Step (b) follows from \eqref{eq:appendix_T_i_upper_bound_eq} (noting that $x^* \ge 1$). Finally, since $\frac{2\left(\lceil y \rceil - 2\right)^{1-\alpha}}{\alpha - 1} \le \frac{2^\alpha}{\alpha-1} y^{1-\alpha}$ for any $y > 3$ (note that $n^{\eta'} x \ge A(n) + 3 > 3$, since $A(n) > 0$) we have that
    \begin{equation*}
        \frac{2\left(\lceil n^{\eta'} x \rceil - 2\right)^{1-\alpha}}{\alpha - 1} \le \frac{2^\alpha}{\alpha-1} (n^{\eta'} x)^{1-\alpha} \le \frac{2^\alpha}{\alpha-1} x^{1-\alpha}.
    \end{equation*}

    \noindent For case (ii), i.e., if $x < (A(n) + 3)/n^{\eta'}$, it holds that $x^{1-\alpha} \ge \left(\frac{A(n)+3}{n^{\eta'}}\right)^{1-\alpha}$, equivalently $x^{1-\alpha} \left(\frac{A(n)+3}{n^{\eta'}}\right)^{\alpha - 1} \ge 1$, since $\alpha > 2$. Therefore, for any $x < (A(n) + 3)/n^{\eta'}$ we have that
    \begin{equation*}
        \mathbb{P}\left[T_i(n) \ge n^{\eta'} x \right] \le 1 \le \left(\frac{A(n)+3}{n^{\eta'}}\right)^{\alpha - 1} x^{1-\alpha} \overset{(a)}{\le} \left[ \left( \frac{2}{\Delta_\text{min}} \theta^{1/\xi} \right)^{\frac{1}{1-\eta}} + 4 \right] x^{1-\alpha}.
    \end{equation*}
    where (a) follows from the fact that, given $\eta' = \frac{\alpha}{\xi (1-\eta)}$ and the definition of $A(n)$ in \eqref{eq:A(t)_def},
    \begin{align*}
        \frac{A(n)+3}{n^{\eta'}} &= \frac{\left\lceil \left( \frac{2}{\Delta_\text{min}} \theta^{1/\xi} \right)^{\frac{1}{1-\eta}} n^{{\eta'}} \right\rceil + 3}{n^{\eta'}}\\
        &\le \frac{ \left( \frac{2}{\Delta_\text{min}} \theta^{1/\xi} \right)^{\frac{1}{1-\eta}} n^{{\eta'}} + 1 + 3}{n^{\eta'}} \\
        &= \left( \frac{2}{\Delta_\text{min}} \theta^{1/\xi} \right)^{\frac{1}{1-\eta}} + \frac{4}{n^{\eta'}} \\
        &\le \left( \frac{2}{\Delta_\text{min}} \theta^{1/\xi} \right)^{\frac{1}{1-\eta}} + 4.
    \end{align*}

    \noindent To conclude the proof, we merge cases (i) and (ii), yielding for any $x > 0$ and $n \ge N_p$ that
    \begin{equation*}
        \mathbb{P}\left[T_i(n) \ge n^{\eta'} x \right] \le  \max\left\{ \left( \frac{2}{\Delta_\text{min}} \theta^{1/\xi} \right)^{\frac{1}{1-\eta}} + 4, \frac{2^\alpha}{\alpha-1} \right\} x^{1-\alpha}.
    \end{equation*}
\end{proof}

\noindent Now, we prove Theo.~\ref{theo:non-stationary-bandit-results-stream-estimator}. We start by proving the concentration result (part 2) and then, building on this result, we prove the convergence result (part 1). \\

\noindent \textbf{Proof of part 2 of Theorem~\ref{theo:non-stationary-bandit-results-stream-estimator} (concentration result):} \\ 

\begin{proof}
    Let $\eta'' = \frac{\alpha}{\xi (1-\eta)}$. Since $\xi\eta(1-\eta) \le \alpha < \xi (1-\eta)$, we have that $1/2 \le \eta'' < 1$. \\
    
    \noindent For any $n \in \mathbb{N}$, and $z \ge 1$ we have that
    \begin{equation}
    \mathbb{P}\left[n\hat{\rho}_n^\text{stream} - n\mu^* \ge n^{\eta''} z \right] \le \mathbb{P}\left[ n\bar{\rho}_n - n\mu^* \ge \frac{1}{2} n^{\eta''} z \right] + \mathbb{P}\left[ n\hat{\rho}_n^\text{stream}- n\bar{\rho}_n \ge \frac{1}{2} n^{\eta''} z \right]. \label{eq:appendix_proposition_rho_stream_eq_1}
    \end{equation}

    \noindent For any $z \ge 1$ and $n \in \mathbb{N}$, the first term in \eqref{eq:appendix_proposition_rho_stream_eq_1} can be upper bounded using the result from Theo.~\ref{theo:non-stationary-bandit-results}. In particular, 
    \begin{equation*}
    \mathbb{P}\left[ n\bar{\rho}_n - n\mu^* \ge \frac{1}{2} n^{\eta''} z \right] \le \frac{\theta' 2^{\xi'}}{z^{\xi'}} \overset{(a)}{=} \frac{c_1}{z^{\alpha-1}},
    \end{equation*}
    where $\xi' = \alpha - 1$ and $\theta'$ is defined in Theo.~\ref{theo:non-stationary-bandit-results}. In step (a), we let $c_1 = \theta' 2^{\xi'}$. We make two remarks. First, we let $\eta'' = \frac{\alpha}{\xi (1-\eta)}$ in this proof, hence, the result from Theo.~\ref{theo:non-stationary-bandit-results} is directly applicable. Second, to derive the result above, we invoke Theo~\ref{theo:non-stationary-bandit-results} with $z = z' = z/2$, yielding $\mathbb{P}\left[ n\bar{\rho}_n - n\mu^* \ge n^{\eta''} z' \right] \le \frac{\theta'}{(z')^{\xi'}}$. We note that Theo.~\ref{theo:non-stationary-bandit-results} is only valid for $z' \ge 1$ (since Theo.~\ref{theo:non-stationary-bandit-results} requires $z \ge 1$); however, if $z' < 1$ (i.e., $z < 2$) then $1 \le \frac{\theta'}{(z')^{\xi'}}$ and the upper bound is also valid (since any probability value is trivially upper bounded by one). \\
    
    \noindent To upper bound the second term in \eqref{eq:appendix_proposition_rho_stream_eq_1}, letting $i^* = \argmin_{i \in \{1, \ldots, K\}} \mu_i$, we start by noting that
    \begingroup
    \allowdisplaybreaks
    \begin{align*}
         n (\hat{\rho}_n^\text{stream} - \bar{\rho}_n) &= \frac{n}{\beta} \ln\left( \frac{1}{n} \sum_{i=1}^K T_i(n) \exp(\beta \hat{\rho}_{i, T_i(n)})\right) - \sum_{i=1}^K T_i(n) \hat{\rho}_{i, T_i(n)} \\
         &\overset{(a)}{=}  \frac{n}{\beta} \ln\left( \exp(\beta \hat{\rho}_{i^*, T_{i^*}(n)})\left( \frac{T_{i^*}(n)}{n} + \sum_{i \neq i^*} \frac{T_i(n)}{n} \exp(\beta(\hat{\rho}_{i, T_i(n)} - \hat{\rho}_{i^*, T_{i^*}(n)})) \right) \right) \\
         & \quad \quad - \sum_{i=1}^K T_i(n) \hat{\rho}_{i, T_i(n)} \\
         &= n \hat{\rho}_{i^*, T_{i^*}} + \frac{n}{\beta} \ln\left( \frac{T_{i^*}(n)}{n} + \sum_{i \neq i^*} \frac{T_i(n)}{n} \exp(\beta(\hat{\rho}_{i, T_i(n)} - \hat{\rho}_{i^*, T_{i^*}(n)})) \right) \\
         & \quad \quad - \sum_{i=1}^K T_i(n) \hat{\rho}_{i, T_i(n)} \\
         &= \sum_{i=1}^K T_i(n) \left( \hat{\rho}_{i^*, T_{i^*}(n)} - \hat{\rho}_{i, T_i(n)} \right) + \frac{n}{\beta} \ln\left( \frac{T_{i^*}(n)}{n} + \sum_{i \neq i^*} \frac{T_i(n)}{n} \exp(\beta(\hat{\rho}_{i, T_i(n)} - \hat{\rho}_{i^*, T_{i^*}(n)})) \right) \\
         &= \sum_{i \neq i^*} T_i(n) \left( \hat{\rho}_{i^*, T_{i^*}(n)} - \hat{\rho}_{i, T_i(n)} \right) + \frac{n}{\beta} \ln\left( \frac{T_{i^*}(n)}{n} + \sum_{i \neq i^*} \frac{T_i(n)}{n} \exp(\beta(\hat{\rho}_{i, T_i(n)} - \hat{\rho}_{i^*, T_{i^*}(n)})) \right) \\
         &\overset{(b)}{=} \sum_{i \neq i^*} T_i(n) \left( \hat{\rho}_{i^*, T_{i^*}(n)} - \hat{\rho}_{i, T_i(n)} \right) + \frac{n}{\beta} \ln \left( 1 + \sum_{i \neq i^*} \frac{T_i(n)}{n} \left( \exp(\beta(\hat{\rho}_{i, T_i(n)} - \hat{\rho}_{i^*, T_{i^*}(n)})) - 1\right) \right) \\
         &\overset{(c)}{\le} \sum_{i \neq i^*} T_i(n) \left( \hat{\rho}_{i^*, T_{i^*}(n)} - \hat{\rho}_{i, T_i(n)} \right) + \frac{n}{\beta} \sum_{i \neq i^*} \frac{T_i(n)}{n} \left( \exp(\beta(\hat{\rho}_{i, T_i(n)} - \hat{\rho}_{i^*, T_{i^*}(n)})) - 1\right) \\
         &= \sum_{i \neq i^*} T_i(n) \left[ -\left( \hat{\rho}_{i, T_i(n)} - \hat{\rho}_{i^*, T_{i^*}(n)} \right) + \frac{1}{\beta} \left( \exp(\beta(\hat{\rho}_{i, T_i(n)} - \hat{\rho}_{i^*, T_{i^*}(n)})) - 1\right) \right] \\
         &\overset{(d)}{\le} \sum_{i \neq i^*} T_i(n) \frac{\exp(2\beta R) - 2 \beta R - 1}{\beta},
    \end{align*}
    \endgroup
    where (a) holds by isolating term $i^*$ in the sum; (b) holds by noting that $\sum_{i \neq i^*} T_i(n)/n + T_{i^*}(n)/n = 1$; (c) holds since $\ln(1+x) \le x$ for any $x > 0$. For step (d), first note that $ \hat{\rho}_{i, T_i(n)} - \hat{\rho}_{i^*, T_{i^*}(n)} \in [-2R,2R]$ since $\hat{\rho}_{i, T_{i}(n)} \in [-R,R]$ for any arm $i$ (from Lemma~\ref{lemma:estimator_properties}). Let $f(x) =  -x + (1/\beta) (\exp(\beta x) - 1)$, with $x \in [-2R,2R]$. It holds that the maximum of $f(x)$ over the $[-2R,2R]$ equals $f(2R) = \frac{\exp(2\beta R) - 2 \beta R - 1}{\beta}$ since $f$ is convex and the maximum is attained at the right endpoint $x=2R$. We also note that $\frac{\exp(2\beta R) - 2 \beta R - 1}{\beta} > 0$ for any $\beta > 0$. Therefore, it holds that the second term in \eqref{eq:appendix_proposition_rho_stream_eq_1} can be upper bounded as
    \begin{align*}
        \mathbb{P}\left[ n\hat{\rho}_n^\text{stream}- n\bar{\rho}_n \ge \frac{1}{2} n^{\eta''} z \right] &\le \mathbb{P}\left[ \sum_{i \neq i^*} T_i(n) \ge \frac{\beta}{2 \left( \exp(2\beta R) - 2 \beta R - 1 \right)} n^{\eta''} z \right] \\
        &\overset{(a)}{\le} \sum_{i \neq i^*} \mathbb{P}\left[ T_i(n) \ge \frac{\beta}{2 (K-1) \left( \exp(2\beta R) - 2 \beta R - 1 \right)} n^{\eta''} z \right] \\
        &\overset{(b)}{\le} \sum_{i \neq i^*} \max\left\{ \left( \frac{2}{\Delta_\text{min}} \theta^{1/\xi} \right)^{\frac{1}{1-\eta}} + 4, \frac{2^\alpha}{\alpha-1} \right\} \\
        &\quad \quad \cdot \left( \frac{\beta}{2 (K-1) \left( \exp(2\beta R) - 2 \beta R - 1 \right)}\right)^{1-\alpha} z^{1-\alpha} \\
        &\overset{(c)}{=} \frac{c_2}{z^{\alpha - 1}} 
    \end{align*}
    where (a) follows from a union bound. Step (b) follows from Lemma~\ref{lemma:non-optimal-arm-concentration-bound} by noting that $\eta'' = \frac{\alpha}{\xi (1-\eta)}$ and setting $x=\frac{\beta}{2 (K-1) \left( \exp(2\beta R) - 2 \beta R - 1 \right)} z$ in the Lemma. In (c), we let
    $$c_2 = (K-1) \max\left\{ \left( \frac{2}{\Delta_\text{min}} \theta^{1/\xi} \right)^{\frac{1}{1-\eta}} + 4, \frac{2^\alpha}{\alpha-1} \right\} \left( \frac{\beta}{2 (K-1) \left( \exp(2\beta R) - 2 \beta R - 1 \right)}\right)^{1-\alpha}.$$
    
    \noindent We note, however, that the bound above only holds for $n \ge N_p$, where $N_p$ is defined in \eqref{eq:Np_definition}.
    To find a bound that holds for any $n\ge1$ we first observe that $|n\rho^{\text{stream}}_n - n\bar{\rho}_n| \le 2Rn$. Moreover, for each $1\le n <N_p$, it exists a $\bar{z}(n)$ such that $ n^{\eta''}\bar{z}(n)/2 = 2Rn$. That is 
    \begin{equation*}
        \bar{z}(n) = 4Rn^{1-\eta''}.
    \end{equation*}

    \noindent This implies that for each $1 \le n < N_p, z \ge 1$, the following inequality holds:
    \begin{equation*}
        \mathbb{P}\left[ n\hat{\rho}_n^\text{stream}- n\bar{\rho}_n \ge \frac{1}{2} n^{\eta''} z \right] \le \frac{\bar{z}(n)^{\alpha-1}}{z^{\alpha-1}}
    \end{equation*}
    This is trivially true for each $1 \le n < N_p$ since: (i) if $z > \bar{z}(n)$, it implies $\eta'' z / 2 > 2R$ and the above probability is $0$ and any positive number in the right hand side of the equation above makes the inequality true; and (ii) if $1 \le z \le \bar{z}(n)$, it implies that the right hand side is at least $1$, making the inequality hold. \\

    \noindent We now define
    \begin{equation*}
        c_3=\max_{n\in [1,N_p)} \bar{z}(n) = 4 R (N_p-1)^{1-\eta''}
    \end{equation*}
    Then, we can conclude that, for each $n\ge1$ and each $z \ge 1$, we have
    \begin{equation*}
        \mathbb{P}\left[ n\hat{\rho}_n^\text{stream}- n\bar{\rho}_n \ge \frac{1}{2} n^{\eta''} z \right] \le \frac{\max\{c_2, c_3\}}{z^{\alpha-1}}
    \end{equation*}

    \noindent Therefore, for any $n \in \mathbb{N}$, and $z \ge 1$, it holds 
    \begin{equation*}
    \mathbb{P}\left[n\hat{\rho}_n^\text{stream} - n\mu^* \ge n^{\eta''} z \right] \le \frac{c_1 + \max\{c_2, c_3\}}{z^{\alpha-1}}.
    \end{equation*}

    \noindent For the other direction of the bound, letting $\eta'' = \frac{\alpha}{\xi (1-\eta)}$, it holds for any $z \ge 1$ and $n \in \mathbb{N}$ that
    \begin{equation*}
        \mathbb{P}\left[n\hat{\rho}_n^\text{stream} - n\mu^* \le -n^{\eta''} z \right] \overset{(a)}{\le} \mathbb{P}\left[\bar{\rho}_n - n\mu^* \le -n^{\eta''} z \right] \overset{(b)}{\le} \frac{\theta'}{z^{\xi'}} = \frac{\theta'}{z^{\alpha - 1}},
    \end{equation*}
    where (a) follows from the fact that $\bar{\rho}_n \le \hat{\rho}_n^\text{stream}$ since (from the convexity of $\exp$ and Jensen's inequality)
    $$\bar{\rho}_n = \frac{1}{\beta} \ln \left( \exp\left( \sum_{i=1}^K \frac{T_i(n)}{n} \beta \hat{\rho}_{i, T_i(n)} \right) \right) \le \frac{1}{\beta} \ln \left( \sum_{i=1}^K \frac{T_i(n)}{n} \exp\left(\beta \hat{\rho}_{i, T_i(n)} \right) \right) = \hat{\rho}_n^\text{stream}.$$
    Step (b) follows from Theo.~\ref{theo:non-stationary-bandit-results} where $\xi' = \alpha - 1$. \\

    \noindent To conclude, we find a single $\theta''$ value that works for both directions by taking the maximum between $c_1 + \max\{c_2, c_3\}$ and $\theta'$, yielding
    $$\mathbb{P}\left[n\hat{\rho}_n^\text{stream} - n\mu^* \ge n^{\eta''} z \right] \le \frac{\theta''}{z^{\xi''}}, \quad \quad \mathbb{P}\left[n\hat{\rho}_n^\text{stream} - n\mu^* \le -n^{\eta''} z \right] \le \frac{\theta''}{z^{\xi''}}$$
    with 
    \begingroup
    \allowdisplaybreaks
    \begin{align*}
    \eta'' &= \frac{\alpha}{\xi(1-\eta)} \\
    \xi'' &= \alpha - 1 \\
    \theta'' &= \max\left\{c_1 + \max\{c_2, c_3\}, \theta' \right\}.
    \end{align*}
    \endgroup
\end{proof}

\paragraph{Proof of part 1 of Theorem~\ref{theo:non-stationary-bandit-results-stream-estimator} (convergence result):}
\begin{proof}
    Regarding the convergence in expectation, i.e., the fact that $\lim_{n \rightarrow \infty} \mathbb{E}\left[\hat{\rho}_{n}^\text{stream}\right] = \mu^*$, we start by noting that, for any $\varepsilon >0$, we have that
    \begingroup
    \allowdisplaybreaks
    \begin{align*}
        \mathbb{P}[|\hat{\rho}_{n}^\text{stream} - \mu^*| \ge \varepsilon] &= \mathbb{P}[n|\hat{\rho}_{n}^\text{stream} - \mu^*| \ge n\varepsilon] \\
        &\le \mathbb{P}[n(\hat{\rho}_{n}^\text{stream}-\mu^*) \ge n\varepsilon] + \mathbb{P}[n(\hat{\rho}_{n}^\text{stream}-\mu^*) \le -n\varepsilon] \\
        &\overset{(a)}{\le} 2 \frac{\theta''}{(n^{1-\eta''} \varepsilon)^{\xi''}},
    \end{align*}
    \endgroup
    where (a) follows from the already proven concentration bound by noting that, if $z=n^{1-\eta''} \varepsilon$ and $z\ge 1$, then $\mathbb{P}[n(\hat{\rho}_{n}^\text{stream}-\mu^*) \ge n^{\eta''} z] = \mathbb{P}[n(\hat{\rho}_{n}^\text{stream}-\mu^*) \ge n\varepsilon]$ and we apply our derived concentration result; otherwise ($z < 1$), the upper bound is greater than one and the result also holds. We also note that $|\hat{\rho}_{n}^\text{stream}-\mu^*| \le 2R$ since both $|\hat{\rho}_{n}^\text{stream}| \le R$ and $|\mu^*|\le R$. Hence, for any $\varepsilon > 0$,
    \begingroup
    \allowdisplaybreaks
    \begin{align*}
        \mathbb{E}\big[|\hat{\rho}_{n}^\text{stream}-\mu^*|\big] &= \mathbb{E}\big[|\hat{\rho}_{n}^\text{stream}-\mu^*|\mathbf{1}_{|\hat{\rho}_{n}^\text{stream}-\mu^*|\le\varepsilon}\big] + \mathbb{E}\big[|\hat{\rho}_{n}^\text{stream}-\mu^*|\mathbf{1}_{|\hat{\rho}_{n}^\text{stream}-\mu^*|>\varepsilon}\big] \\
        &\le \varepsilon + 2R\mathbb{P}[|\hat{\rho}_{n}^\text{stream}-\mu^*|>\varepsilon]\\
        &\le \varepsilon + 4R \frac{\theta''}{(n^{1-\eta''} \varepsilon)^{\xi''}},
    \end{align*}
    \endgroup
    and therefore, for any $\varepsilon > 0$, taking the limit $ n \rightarrow \infty$ yields the desired result. \\

    \noindent To get an explicit convergence rate we note that
    \begingroup
    \allowdisplaybreaks
    \begin{align*}
        |\EE[]{\hat{\rho}_{n}^\text{stream}} - \mu^*| &\le \EE[]{|\hat{\rho}_{n}^\text{stream} - \mu^*|} \\
        &\overset{(a)}{=} \int_{0}^\infty \mathbb{P}\left[ |\hat{\rho}_{n}^\text{stream} - \mu^* | > u \right] du \\
        &= \int_{0}^{n^{\eta''-1}} \mathbb{P}\left[ |\hat{\rho}_{n}^\text{stream} - \mu^* | > u \right] du + \int_{n^{\eta''-1}}^\infty \mathbb{P}\left[ |\hat{\rho}_{n}^\text{stream} - \mu^* | > u \right] du \\
        &\overset{(b)}{\le} \int_{0}^{n^{\eta''-1}} 1 du + n^{\eta''-1} \int_{1}^\infty \mathbb{P}\left[ |\hat{\rho}_{n}^\text{stream} - \mu^* | > n^{\eta''-1} z \right] dz \\
        &\overset{(c)}{\le} n^{\eta''-1} + n^{\eta''-1}  2 \theta'' \int_{1}^\infty \frac{1}{z^{\xi''}} dz \\
        &= n^{\eta''-1} + n^{\eta''-1}  2 \theta'' \frac{1}{\xi'' - 1}\\
        &= \left( 1 + \frac{2 \theta''}{\xi'' - 1} \right) n^{\eta''-1},
    \end{align*}
    \endgroup
    where (a) is true since $|\hat{\rho}_{n}^\text{stream} - \mu^*|$ is non-negative; (b) follows from the change of variables $u = n^{\eta''-1} z$; and (c) follows from the concentration result. Since $\eta'' = \frac{\alpha}{\xi(1-\eta)}$ we get that $|\EE[]{\hat{\rho}_{n}^\text{stream}} - \mu^*| = \O\left( \theta'' n^{\frac{\alpha}{\xi(1-\eta)}-1}\right)$.
\end{proof}

\section{Proof of Lemma~\ref{lemma:non-deterministic-bandit-convergence-and-concentration}} \label{appendix:proof-non-deterministic-bandit-convergence-and-concentration}

To prove Lemma~\ref{lemma:non-deterministic-bandit-convergence-and-concentration}, we start by proving an auxiliary result (Lemma~\ref{lemma:non-deterministic-bandit-convergence-and-concentration-auxiliary}) in Sec.~\ref{appendix:proof-non-deterministic-bandit-convergence-and-concentration:auxiliary-result}. Lemma~\ref{lemma:non-deterministic-bandit-convergence-and-concentration-auxiliary} essentially provides convergence and concentration results for estimator $\hat{\rho}_{i,n}^{\text{int.}}$, which is similar to $\hat{\rho}_{i,n}$ (as introduced in \eqref{eq:entropic_risk_estimator_non_det_bandit}) but does not consider intermediate costs and discounting. Then, in Sec.~\ref{appendix:proof-non-deterministic-bandit-convergence-and-concentration:main-result}, we exploit Lemma~\ref{lemma:non-deterministic-bandit-convergence-and-concentration-auxiliary} and prove Lemma~\ref{lemma:non-deterministic-bandit-convergence-and-concentration}, which provides convergence and concentration results for estimator $\hat{\rho}_{i,n}$.

\subsection{Auxiliary result}
\label{appendix:proof-non-deterministic-bandit-convergence-and-concentration:auxiliary-result}

\begin{lemma}
    \label{lemma:non-deterministic-bandit-convergence-and-concentration-auxiliary}
    For any arm $i \in \{1, \ldots, K\}$, let
    \begin{equation}
    \hat{\rho}_{i,n}^{\text{int.}} = \frac{1}{\beta} \ln \left( \frac{1}{n} \sum_{s' \in \S} \sum_{t=1}^{T_i^{s'}(n)} \exp(\beta \rvar{x}_{i,t}^{s'}) \right) \label{eq:rho_intermediate_def}    
    \end{equation}
    and
    $$\tilde{\mu}_i = \text{ERM}_\beta(\mu_i^{\rvar{s}'}) = \frac{1}{\beta} \ln \left(\mathbb{E}_{\rvar{s}' \sim P^i(\cdot)} \left[ \exp(\beta \mu_i^{\rvar{s}'}) \right] \right).$$
    Then, under Assumption~\ref{assumption:non-deterministic-bandit-assumption}, it holds for any arm $i$ that
    \begin{enumerate}
        \item $\lim_{n \rightarrow \infty} \mathbb{E}[\hat{\rho}_{i,n}^{\text{int.}}] = \tilde{\mu}_i$.
        \item there exist constants $\theta^L > 1$, $\xi^L > 1 $, and $1/2 \le \eta^L < 1$ such that, for any $z \ge 1$ and $n \in \mathbb{N}$,
            \begin{align*}
            \mathbb{P}\left[n\hat{\rho}_{i,n}^{\text{int.}} - n\tilde{\mu}_i \ge n^{\eta^L} z \right] &\le \frac{\theta^L}{z^{\xi^L}}, \\
            \mathbb{P}\left[n\hat{\rho}_{i,n}^{\text{int.}} - n\tilde{\mu}_i \le - n^{\eta^L} z \right] &\le \frac{\theta^L}{z^{\xi^L}},
            \end{align*}
        with $\theta^L = 2^{\xi +1} |\S|^{\xi + 1} \theta \left( R^\xi + \exp(4 \beta R)^\xi + \frac{\exp(2 \beta R)^\xi}{\beta^\xi}\right)$, $\xi^L = \xi$, and $\eta^L = \eta$.
    \end{enumerate}
\end{lemma}

\subsubsection{Proof of Lemma~\ref{lemma:non-deterministic-bandit-convergence-and-concentration-auxiliary}}
We state the following result, which is due to \cite{comer_2025}.

\begin{lemma} \label{lemma:comer_2025_lemma}
    For every arm $i \in \{1, \ldots K\}$, $\eta \in [1/2, 1)$ and any finite $\xi > 1$ there exists $\theta^T > 1$ such that for all $ n \in \mathbb{N}$, $z \in \mathbb{R}$ with $n,z \ge 1$ and $s' \in \S$ we have
    $$\mathbb{P}[T_i^{s'}(n) - n P_i(s') \ge n^\eta z] \le \frac{\theta^T}{z^\xi},$$
    $$\mathbb{P}[T_i^{s'}(n) - n P_i(s') \le -n^\eta z] \le \frac{\theta^T}{z^\xi}.$$
\end{lemma}

\noindent Now, we provide the proof of Lemma~\ref{lemma:non-deterministic-bandit-convergence-and-concentration-auxiliary}. We start by proving the concentration result (part 2) and then, building on this result, we prove the convergence result (part 1). \\

\noindent \textbf{Proof of part 2 of Lemma~\ref{lemma:non-deterministic-bandit-convergence-and-concentration-auxiliary} (concentration result):} 

\begin{proof}
    We note that
    \begin{equation*}
    \hat{\rho}_{i,n}^{\text{int.}} = \frac{1}{\beta} \ln \left( \frac{1}{n} \sum_{s' \in \S} \sum_{t=1}^{T_i^{s'}(n)} \exp(\beta \rvar{x}_{i,t}^{s'}) \right) = \frac{1}{\beta} \ln \left( \frac{1}{n} \sum_{s' \in \S} T_i^{s'}(n) \exp(\beta \hat{\rho}_{i, T_i^{s'}(n)}^{s'}) \right).
    \end{equation*}
    and that
    \begin{equation*}
    \tilde{\mu}_i = \text{ERM}_\beta(\mu_i^{\rvar{s}'}) = \frac{1}{\beta} \ln \left(\mathbb{E}_{\rvar{s}' \sim P^i(\cdot)} \left[ \exp(\beta \mu_i^{\rvar{s}'}) \right] \right) = \frac{1}{\beta} \ln \left(\sum_{s' \in \S} P^i(s') \exp(\beta \mu_i^{s'}) \right).
    \end{equation*}
    We start by following similar steps to the proof of Theorem 3 in \cite{comer_2025}:
    \begingroup
    \allowdisplaybreaks
    \begin{align*}
         &\mathbb{P}\left[n\hat{\rho}_{i,n}^{\text{int.}} - n\tilde{\mu}_i \ge n^{\eta} z \right] = \\
         &\qquad \mathbb{P}\left[\underbrace{n\frac{1}{\beta} \ln \left( \frac{1}{n} \sum_{s' \in \S} T_i^{s'}(n) \exp(\beta \hat{\rho}_{i, T_i^{s'}(n)}^{s'}) \right) - n \frac{1}{\beta} \ln \left(\sum_{s' \in \S} P^i(s') \exp(\beta \mu_i^{s'}) \right) \ge n^{\eta} z }_{(A)} \right] \\
        &\le \mathbb{P}\left[ \underbrace{(A) \cap \bigcup_{s' \in \S} \left\{ |T_i^{s'}(n) - nP_i(s')| \ge \frac{n^\eta z}{2R|\S|} \right\}}_{\text{(I)}}\right] + \mathbb{P}\left[ \underbrace{(A) \cap \bigcap_{s' \in \S} \left\{ |T_i^{s'}(n) - nP_i(s')| \le \frac{n^\eta z}{2R|\S|} \right\} }_{\text{(II)}} \right].
    \end{align*}
    \endgroup
    We now bound each term separately. Starting with $\mathbb{P}[\text{(I)}]$ we have that
    \begingroup
    \allowdisplaybreaks
    \begin{align*}
         \mathbb{P}[\text{(I)}] &\le \sum_{s' \in \S} \mathbb{P}\left[ (A) \cap \left\{ |T_i^{s'}(n) - nP_i(s')| \ge \frac{n^\eta z}{2R|\S|} \right\} \right] \\
         &\le \sum_{s' \in \S} \mathbb{P}\left[ |T_i^{s'}(n) - nP_i(s')| \ge \frac{n^\eta z}{2R|\S|} \right] \\
         &\overset{(a)}{\le} 2 |\S| \frac{\theta 2^\xi R^\xi |\S|^\xi}{z^\xi}
    \end{align*}
    \endgroup
    where (a) follows from the fact that, from Lemma~\ref{lemma:comer_2025_lemma}, we have that for any $s' \in \S$ there exists $\theta^T$ such that $\mathbb{P}[T_i^{s'}(n) - n P_i(s') \ge n^\eta \frac{z}{2R|\S|}] \le \frac{\theta^T}{(\frac{z}{2R|\S|})^\xi} = \frac{\theta^T 2^\xi R^\xi |\S|^\xi}{z^\xi}$. I.e., the bound in Lemma~\ref{lemma:comer_2025_lemma} holds by letting $z' = z/(2R|\S|)$ (where we denote with $z'$ the $z$ variable in Lemma~\ref{lemma:comer_2025_lemma}) since if $z\ge 2R|\S|$ the result follows directly from the Lemma, otherwise $\frac{\theta^T}{(\frac{z}{2R|\S|})^\xi} \ge 1$ and the bound holds trivially. We follow similar steps to bound the other direction and merge the two results via a union bound. Similar to \cite{comer_2025}, we drop the subscript $T$ from $\theta^T$ since we can set $\theta$ to be the largest constant among all the concentration bounds for $T_i^{s'}(n)$ (from Lemma~\ref{lemma:comer_2025_lemma}) and $\hat{\rho}_{i,n}^{s'}$ (from Assumption~\ref{assumption:non-deterministic-bandit-assumption}).

    \noindent Regarding term $\mathbb{P}[\text{(II)}]$, we start by noting that event
    \begingroup
    \allowdisplaybreaks
    \begin{align*}
    (A) &= n \underbrace{\left( \frac{1}{\beta} \ln \left( \frac{1}{n} \sum_{s' \in \S} T_i^{s'}(n) \exp(\beta \hat{\rho}_{i, T_i^{s'}(n)}^{s'})\right) - \frac{1}{\beta} \ln \left( \frac{1}{n} \sum_{s' \in \S} T_i^{s'}(n) \exp(\beta \mu_i^{s'})\right) \right)}_{\text{(C)}}\\
    &\quad + n \underbrace{\left( \frac{1}{\beta} \ln \left( \frac{1}{n} \sum_{s' \in \S} T_i^{s'}(n) \exp(\beta \mu_i^{s'})\right) - \frac{1}{\beta} \ln \left(\sum_{s' \in \S} P^i(s') \exp(\beta \mu_i^{s'}) \right) \right)}_{\text{(D)}} \\
    &= n(C) + n (D).
    \end{align*}
    \endgroup
    \noindent Hence,
    \begingroup
    \allowdisplaybreaks
    \begin{align*}
    \mathbb{P}[\text{(II)}] &= \mathbb{P}\left[ (A) \cap \underbrace{\bigcap_{s' \in \S} \left\{ |T_i^{s'}(n) - nP_i(s')| \le \frac{n^\eta z}{2R|\S|} \right\}}_{\E}\right]\\
    &= \mathbb{P}[\left\{n(C) + n (D) \ge n^\eta z \right\} \cap \E] \\
    &\overset{(a)}{\le} \mathbb{P}\left[\left( \left\{n(C) \ge \frac{n^\eta z}{2} \right\} \cup \left\{n(D) \ge \frac{n^\eta z}{2} \right\}\right) \cap \E\right] \\
    &\overset{(b)}{\le} \mathbb{P}\left[\left\{n(C) \ge \frac{n^\eta z}{2}\right\} \cap \E\right] + \mathbb{P}\left[\left\{n(D) \ge \frac{n^\eta z}{2}\right\} \cap \E\right],
    \end{align*}
    \endgroup
    where (a) follows from the fact that $\left\{n(C) + n (D) \ge n^\eta z \right\} \subseteq \left\{n(C) \ge \frac{n^\eta z}{2} \right\} \cup \left\{n(D) \ge \frac{n^\eta z}{2} \right\}$; and (b) from a union bound. \\

    \noindent We now focus our attention on the term $\mathbb{P}\left[\left\{n(C) \ge \frac{n^\eta z}{2}\right\} \cap \E\right]$. We have that
    \begingroup
    \allowdisplaybreaks
    \begin{align*}
         (C) &= \frac{1}{\beta} \ln \left( \frac{1}{n} \sum_{s' \in \S} T_i^{s'}(n) \exp(\beta ( \hat{\rho}_{i, T_i^{s'}(n)}^{s'} + \mu_i^{s'} - \mu_i^{s'} ) )\right) - \frac{1}{\beta} \ln \left( \frac{1}{n} \sum_{s' \in \S} T_i^{s'}(n) \exp(\beta \mu_i^{s'})\right) \\
         &= \frac{1}{\beta} \ln \left( \sum_{s' \in \S} \frac{(T_i^{s'}(n)/n) \exp(\beta \mu_i^{s'})}{\sum_{u \in \S} (T_i^{u}(n)/n) \exp(\beta \mu_i^{u})} \exp\left(\beta (\hat{\rho}_{i, T_i^{s'}(n)}^{s'} - \mu_i^{s'}) \right) \right) \\
         &\overset{(a)}{=} \frac{1}{\beta} \ln \left( 1 + \sum_{s' \in \S} \frac{(T_i^{s'}(n)/n) \exp(\beta \mu_i^{s'})}{\sum_{u \in \S} (T_i^{u}(n)/n) \exp(\beta \mu_i^{u})} \left( \exp\left(\beta (\hat{\rho}_{i, T_i^{s'}(n)}^{s'} - \mu_i^{s'}) \right) - 1 \right)\right) \\
         &\overset{(b)}{\le} \frac{1}{\beta} \sum_{s' \in \S} \frac{(T_i^{s'}(n)/n) \exp(\beta \mu_i^{s'})}{\sum_{u \in \S} (T_i^{u}(n)/n) \exp(\beta \mu_i^{u})} \left( \exp\left(\beta (\hat{\rho}_{i, T_i^{s'}(n)}^{s'} - \mu_i^{s'}) \right) - 1 \right) \\
         &\le \frac{1}{\beta} \sum_{s' \in \S} \frac{(T_i^{s'}(n)/n) \exp(\beta \mu_i^{s'})}{\sum_{u \in \S} (T_i^{u}(n)/n) \exp(\beta \mu_i^{u})} \left| \exp\left(\beta (\hat{\rho}_{i, T_i^{s'}(n)}^{s'} - \mu_i^{s'}) \right) - 1 \right| \\
         &\overset{(c)}{\le} \sum_{s' \in \S} \frac{(T_i^{s'}(n)/n) \exp(\beta \mu_i^{s'})}{\sum_{u \in \S} (T_i^{u}(n)/n) \exp(\beta \mu_i^{u})} \exp\left(\beta |\hat{\rho}_{i, T_i^{s'}(n)}^{s'} - \mu_i^{s'}| \right) |\hat{\rho}_{i, T_i^{s'}(n)}^{s'} - \mu_i^{s'}| \\
        &\overset{(d)}{\le} \exp\left(2 \beta R \right) \sum_{s' \in \S} \frac{(T_i^{s'}(n)/n) \exp(\beta \mu_i^{s'})}{\sum_{u \in \S} (T_i^{u}(n)/n) \exp(\beta \mu_i^{u})} |\hat{\rho}_{i, T_i^{s'}(n)}^{s'} - \mu_i^{s'}|
    \end{align*}
    \endgroup
    where (a) follows from the fact that $\sum_{s' \in \S} \frac{(T_i^{s'}(n)/n) \exp(\beta \mu_i^{s'})}{\sum_{u \in \S} (T_i^{u}(n)/n) \exp(\beta \mu_i^{u})} = 1$; (b) from the fact that, for any $x > -1$, we have $\ln(1+x) \le x$; (c) from the fact that $|\exp(\beta x) - 1| \le \beta \exp(\beta |x|) |x|$; and (d) from the fact that $|\hat{\rho}_{i, T_i^{s'}(n)}^{s'} - \mu_i^{s'}| \le 2R$ from Lemma~\ref{lemma:estimator_properties}. It also holds that, 
    \begingroup
    \allowdisplaybreaks
    \begin{align*}
    -(C) &= \frac{1}{\beta} \ln \left( \left( \sum_{s' \in \S} \frac{(T_i^{s'}(n)/n) \exp(\beta \mu_i^{s'})}{\sum_{u \in \S} (T_i^{u}(n)/n) \exp(\beta \mu_i^{u})} \exp\left(\beta (\hat{\rho}_{i, T_i^{s'}(n)}^{s'} - \mu_i^{s'}) \right) \right)^{-1}\right) \\
    &\overset{(a)}{\le} \frac{1}{\beta} \ln \left( \sum_{s' \in \S} \frac{(T_i^{s'}(n)/n) \exp(\beta \mu_i^{s'})}{\sum_{u \in \S} (T_i^{u}(n)/n) \exp(\beta \mu_i^{u})} \frac{1}{\exp\left(\beta (\hat{\rho}_{i, T_i^{s'}(n)}^{s'} - \mu_i^{s'}) \right)}\right) \\
    &= \frac{1}{\beta} \ln \left( \sum_{s' \in \S} \frac{(T_i^{s'}(n)/n) \exp(\beta \mu_i^{s'})}{\sum_{u \in \S} (T_i^{u}(n)/n) \exp(\beta \mu_i^{u})} \exp\left(\beta (\mu_i^{s'} - \hat{\rho}_{i, T_i^{s'}(n)}^{s'}) \right)\right) \\
    &\overset{(b)}{\le} \exp\left(2 \beta R \right) \sum_{s' \in \S} \frac{(T_i^{s'}(n)/n) \exp(\beta \mu_i^{s'})}{\sum_{u \in \S} (T_i^{u}(n)/n) \exp(\beta \mu_i^{u})} |\hat{\rho}_{i, T_i^{s'}(n)}^{s'} - \mu_i^{s'}|
    \end{align*}
    \endgroup
    where (a) follows from Jensen's inequality for the function $f(x) = 1/x$; and (b) follows from similar steps as those we followed when upper-bounding (C). Combining both results, we have that
    \begingroup
    \allowdisplaybreaks
    \begin{align}
    n |(C)| &\le n \exp\left(2 \beta R \right) \sum_{s' \in \S} \frac{(T_i^{s'}(n)/n) \exp(\beta \mu_i^{s'})}{\sum_{u \in \S} (T_i^{u}(n)/n) \exp(\beta \mu_i^{u})} |\hat{\rho}_{i, T_i^{s'}(n)}^{s'} - \mu_i^{s'}| \nonumber \\
    &= \exp\left(2 \beta R \right) \sum_{s' \in \S} \frac{\exp(\beta \mu_i^{s'})}{\sum_{u \in \S} (T_i^{u}(n)/n) \exp(\beta \mu_i^{u})} T_i^{s'}(n) |\hat{\rho}_{i, T_i^{s'}(n)}^{s'} - \mu_i^{s'}| \nonumber \\
    &\overset{(a)}{\le} \exp\left(4 \beta R \right) \sum_{s' \in \S}  T_i^{s'}(n) |\hat{\rho}_{i, T_i^{s'}(n)}^{s'} - \mu_i^{s'}| \label{eq:lemma_non_stat_non_det_proof_eq1}
    \end{align}
    \endgroup
    where (a) follows from the fact that $\exp(-2 \beta R) \le \frac{\exp(\beta \mu_i^{s'})}{\sum_{u \in \S} (T_i^{u}(n)/n) \exp(\beta \mu_i^{u})} \le \exp(2 \beta R)$. \\

    \noindent Hence,
    \begingroup
    \allowdisplaybreaks
    \begin{align*}
    \mathbb{P}\left[\left\{n(C) \ge \frac{n^\eta z}{2}\right\} \cap \E\right] &\le \mathbb{P}\left[\left\{n|(C)| \ge \frac{n^\eta z}{2}\right\} \cap \E\right] \\
    &\overset{(a)}{\le} \mathbb{P}\left[\left\{\exp\left(4 \beta R \right) \sum_{s' \in \S}  T_i^{s'}(n) |\hat{\rho}_{i, T_i^{s'}(n)}^{s'} - \mu_i^{s'}| \ge \frac{n^\eta z}{2}\right\} \cap \E\right] \\
    &= \mathbb{P}\left[\left\{ \sum_{s' \in \S}  T_i^{s'}(n) |\hat{\rho}_{i, T_i^{s'}(n)}^{s'} - \mu_i^{s'}| \ge \frac{n^\eta z}{2 \exp\left(4 \beta R \right)}\right\} \cap \E\right] \\
    &\overset{(b)}{\le} \mathbb{P}\left[ \bigcup_{s' \in \S} \left\{ T_i^{s'}(n) |\hat{\rho}_{i, T_i^{s'}(n)}^{s'} - \mu_i^{s'}| \ge  \frac{n^\eta z}{2|\S| \exp\left(4 \beta R \right)}\right\} \cap \E\right] \\
    &\overset{(c)}{\le} \sum_{s' \in \S} \mathbb{P}\left[ \left\{ T_i^{s'}(n) |\hat{\rho}_{i, T_i^{s'}(n)}^{s'} - \mu_i^{s'}| \ge \frac{n^\eta z}{2|\S| \exp\left(4 \beta R \right)}\right\} \cap \E\right] \\
    &\le \sum_{s' \in \S} \mathbb{P}\left[ T_i^{s'}(n) |\hat{\rho}_{i, T_i^{s'}(n)}^{s'} - \mu_i^{s'}| \ge \frac{n^\eta z}{2|\S| \exp\left(4 \beta R \right)}\right] \\
    &\overset{(d)}{\le} \sum_{s' \in \S} 2\theta \left(2|\S| \exp\left(4 \beta R \right)\right)^\xi z^{-\xi} \\
    &= 2 |\S| \theta \left(2|\S| \exp\left(4 \beta R \right)\right)^\xi z^{-\xi}
    \end{align*}
    \endgroup
    where (a) follows from \eqref{eq:lemma_non_stat_non_det_proof_eq1}; (b) since 
    $\left\{ \sum_{s' \in \S}  T_i^{s'}(n) |\hat{\rho}_{i, T_i^{s'}(n)}^{s'} - \mu_i^{s'}| \ge \exp\left(-4 \beta R \right) \frac{n^\eta z}{2}\right\} \subseteq \bigcup_{s' \in \S} \left\{ T_i^{s'}(n) |\hat{\rho}_{i, T_i^{s'}(n)}^{s'} - \mu_i^{s'}| \ge \exp\left(-4 \beta R \right) \frac{n^\eta z}{2|\S|}\right\};$
    (c) from a union bound; and (d) by noting that
    \begingroup
    \allowdisplaybreaks
    \begin{align*}
    \mathbb{P}\bigg[ T_i^{s'}(n) |\hat{\rho}_{i, T_i^{s'}(n)}^{s'} - \mu_i^{s'}| \ge &\frac{n^\eta z}{2|\S| \exp\left(4 \beta R \right)}\bigg] \le \mathbb{P}\left[ T_i^{s'}(n) \left(\hat{\rho}_{i, T_i^{s'}(n)}^{s'} - \mu_i^{s'}\right) \ge \frac{n^\eta z}{2|\S| \exp\left(4 \beta R \right)}\right] \\
    &\quad + \mathbb{P}\left[ T_i^{s'}(n) \left(\mu_i^{s'} - \hat{\rho}_{i, T_i^{s'}(n)}^{s'}\right) \ge \frac{n^\eta z}{2|\S| \exp\left(4 \beta R \right)}\right] \\
    &= \mathbb{P}\left[ T_i^{s'}(n) \left(\hat{\rho}_{i, T_i^{s'}(n)}^{s'} - \mu_i^{s'}\right) \ge T_i^{s'}(n)^\eta \frac{n^\eta z}{2|\S| \exp\left(4 \beta R \right)T_i^{s'}(n)^\eta}\right] \\
    &\quad + \mathbb{P}\left[ T_i^{s'}(n) \left(\mu_i^{s'} - \hat{\rho}_{i, T_i^{s'}(n)}^{s'}\right) \ge T_i^{s'}(n)^\eta \frac{n^\eta z}{2|\S| \exp\left(4 \beta R \right)T_i^{s'}(n)^\eta}\right] \\
    &\overset{(a)}{\le} 2\theta / \left(\frac{n^\eta z}{2|\S| \exp\left(4 \beta R \right)T_i^{s'}(n)^\eta}\right)^\xi \\
    &= 2\theta \left(\frac{2|\S| \exp\left(4 \beta R \right)}{ z }\right)^\xi \left(\frac{T_i^{s'}(n)}{n}\right)^{\eta \xi} \\
    &\le 2\theta \left(\frac{2|\S| \exp\left(4 \beta R \right)}{ z }\right)^\xi \\
    &= 2\theta \left(2|\S| \exp\left(4 \beta R \right)\right)^\xi z^{-\xi},
    \end{align*}
    \endgroup
    where step (a) is true from Assumption~\ref{assumption:non-deterministic-bandit-assumption}. To see why, we recall that Assumption~\ref{assumption:non-deterministic-bandit-assumption} states that there exist constants $\theta > 1$, $\xi > 1 $, and $1/2 \le \eta < 1$ such that, for any $z' \ge 1$ and $n \in \mathbb{N}$, $\mathbb{P}\left[n\hat{\rho}_{i,n}^{s'} - n\mu_i^{s'} \ge n^{\eta} z' \right] \le \frac{\theta}{(z')^\xi}$. Letting $z' = \frac{n^\eta z}{2|\S| \exp\left(4 \beta R \right)T_i^{s'}(n)^\xi}$, if $z \ge 2|\S| \exp(4 \beta R)$ (implying that $z' \ge 1$) then the result follows directly from Assumption~\ref{assumption:non-deterministic-bandit-assumption}. If $z$ is such that $z' < 1$, then a similar bound holds trivially because $\mathbb{P}\left[n\hat{\rho}_{i,n}^{s'} - n\mu_i^{s'} \ge n^{\eta} z' \right] \le 1 \le \frac{\theta}{(z')^\xi}$. We follow the same line of reasoning to upper-bound the other term. \\

    \noindent We now focus our attention on the term $\mathbb{P}\left[\left\{n(D) \ge \frac{n^\eta z}{2}\right\} \cap \E\right]$. We have that
    \begingroup
    \allowdisplaybreaks
    \begin{align*}
         (D) &= \frac{1}{\beta} \ln \left( \frac{1}{n} \sum_{s' \in \S} T_i^{s'}(n) \exp(\beta \mu_i^{s'})\right) - \frac{1}{\beta} \ln \left(\sum_{s' \in \S} P^i(s') \exp(\beta \mu_i^{s'}) \right) \\
         &= \frac{1}{\beta} \ln \left( \frac{ \sum_{s' \in \S} (T_i^{s'}(n)/n) \exp(\beta \mu_i^{s'})}{\sum_{s' \in \S} P^i(s') \exp(\beta \mu_i^{s'})} \right) \\
         &= \frac{1}{\beta} \ln \left( 1 + \frac{ \sum_{s' \in \S} ((T_i^{s'}(n)/n) - P^i(s')) \exp(\beta \mu_i^{s'})}{\sum_{s' \in \S} P^i(s') \exp(\beta \mu_i^{s'})} \right) \\
         &\le \frac{1}{\beta} \frac{ \sum_{s' \in \S} ((T_i^{s'}(n)/n) - P^i(s')) \exp(\beta \mu_i^{s'})}{\sum_{s' \in \S} P^i(s') \exp(\beta \mu_i^{s'})} \\
         &\le \frac{1}{\beta} \frac{ \sum_{s' \in \S} |(T_i^{s'}(n)/n) - P^i(s')| \exp(\beta \mu_i^{s'})}{\sum_{s' \in \S} P^i(s') \exp(\beta \mu_i^{s'})} \\
        &\le \frac{1}{\beta} \frac{ \sum_{s' \in \S} |(T_i^{s'}(n)/n) - P^i(s')| \exp(\beta R)}{\sum_{s' \in \S} P^i(s') \exp(-\beta R)} \\
        &\le \frac{\exp(2\beta R)}{\beta} \sum_{s' \in \S} |(T_i^{s'}(n)/n) - P^i(s')|.
    \end{align*}
    \endgroup
    It also holds, following similar steps, that $-(D) \le \frac{\exp(2\beta R)}{\beta} \sum_{s' \in \S} |P^i(s') - (T_i^{s'}(n)/n)|$. Combining both results,
    $$n|(B)| \le n \frac{\exp(2\beta R)}{\beta} \sum_{s' \in \S} |(T_i^{s'}(n)/n) - P^i(s')|.$$
    Hence,
    \begingroup
    \allowdisplaybreaks
    \begin{align*}
    \mathbb{P}\left[\left\{n(D) \ge \frac{n^\eta z}{2}\right\} \cap \E\right] &\le \mathbb{P}\left[\left\{n|(D)| \ge \frac{n^\eta z}{2}\right\} \cap \E\right] \\
    &\le \mathbb{P}\left[\left\{n \frac{\exp(2\beta R)}{\beta} \sum_{s' \in \S} |(T_i^{s'}(n)/n) - P^i(s')| \ge \frac{n^\eta z}{2}\right\} \cap \E\right] \\
    &= \mathbb{P}\left[\left\{ \sum_{s' \in \S} |T_i^{s'}(n) - P^i(s')/n| \ge \frac{\beta}{2 \exp(2 \beta R)} n^\eta z \right\} \cap \E\right] \\
    &\le \mathbb{P}\left[ \bigcup_{s' \in \S} \left\{ |T_i^{s'}(n) - P^i(s')/n| \ge \frac{\beta}{2 |\S| \exp(2 \beta R)} n^\eta z \right\} \cap \E\right] \\
    &\le \sum_{s' \in \S} \mathbb{P}\left[ \left\{ |T_i^{s'}(n) - P^i(s')/n| \ge \frac{\beta}{2 |\S| \exp(2 \beta R)} n^\eta z \right\} \cap \E\right] \\
    &\le \sum_{s' \in \S} \mathbb{P}\left[ \left\{ |T_i^{s'}(n) - P^i(s')/n| \ge \frac{\beta}{2 |\S| \exp(2 \beta R)} n^\eta z \right\} \right] \\
    &\overset{(a)}{\le} 2 |\S| \frac{\theta 2^\xi \exp(2 \beta R)^\xi |\S|^\xi}{\beta^\xi} z^{-\xi}
    \end{align*}
    \endgroup
    where (a) follows from the fact that, from Lemma~\ref{lemma:comer_2025_lemma}, we have that for any $s' \in \S$ there exists $\theta^T$ such that $\mathbb{P}[T_i^{s'}(n) - n P_i(s') \ge n^\eta \frac{z \beta}{2 |\S| \exp(2 \beta R)}] \le \frac{\theta^T}{(\frac{z \beta}{2 |\S| \exp(2 \beta R)})^\xi} = \frac{\theta^T 2^\xi \exp(2 \beta R)^\xi |\S|^\xi}{z^\xi \beta^\xi}$. I.e., the bound in Lemma~\ref{lemma:comer_2025_lemma} holds by letting $z' =  \frac{z \beta}{2 |\S| \exp(2 \beta R)}$ (where we denote with $z'$ the $z$ variable from Lemma~\ref{lemma:comer_2025_lemma}) since if $z \beta \ge 2 |\S| \exp(2 \beta R)$ the result follows directly from the Lemma ($z' \ge 1)$, otherwise $\frac{\theta^T}{(\frac{z \beta}{2 |\S| \exp(2 \beta R)})^\xi} \ge 1$ and the bound holds trivially. We follow similar steps to bound the other direction and merge the two results via a union bound. Similar to \cite{comer_2025}, we drop the subscript $T$ from $\theta^T$ since we can set $\theta$ to be the largest constant among all the concentration bounds for $T_i^{s'}(n)$ (from Lemma~\ref{lemma:comer_2025_lemma}) and $\hat{\rho}_{i,n}^{s'}$ (from Assumption~\ref{assumption:non-deterministic-bandit-assumption}).\\

    \noindent Summarizing, we have that: (i) $\mathbb{P}\left[\left\{n(C) \ge \frac{n^\eta z}{2}\right\} \cap \E\right] \le 2 |\S| \theta \left(2|\S| \exp\left(4 \beta R \right)\right)^\xi z^{-\xi}$; and (ii) $\mathbb{P}\left[\left\{n(D) \ge \frac{n^\eta z}{2}\right\} \cap \E\right] \le 2 |\S| \frac{\theta 2^\xi \exp(2 \beta R)^\xi |\S|^\xi}{ \beta^\xi} z^{-\xi}$. Hence,
    \begingroup
    \allowdisplaybreaks
    \begin{align*}
    \mathbb{P}[(II)] &\le \mathbb{P}\left[\left\{n(C) \ge \frac{n^\eta z}{2}\right\} \cap \E\right] + \mathbb{P}\left[\left\{n(D) \ge \frac{n^\eta z}{2}\right\} \cap \E\right]\\
    &\le 2 |\S| \theta \left(2|\S| \exp\left(4 \beta R \right)\right)^\xi z^{-\xi} + 2 |\S| \frac{\theta 2^\xi \exp(2 \beta R)^\xi |\S|^\xi}{\beta^\xi} z^{-\xi}.
    \end{align*}
    \endgroup

    \noindent Combining all results,
    \begingroup
    \allowdisplaybreaks
    \begin{align*}
         \mathbb{P}\left[n\hat{\rho}_{i,n}^{\text{int.}} - n\tilde{\mu}_i \ge n^{\eta} z \right] &\le \mathbb{P}[\text{(I)}] + \mathbb{P}[\text{(II)}] \\
         &\le 2 |\S| \theta 2^\xi R^\xi |\S|^\xi z^{-\xi} + 2 |\S| \theta 2^\xi |\S|^\xi \exp\left(4 \beta R \right)^\xi z^{-\xi} \\
         &\quad \quad + 2 |\S| \frac{\theta 2^\xi \exp(2 \beta R)^\xi |\S|^\xi}{\beta^\xi} z^{-\xi}\\
         &= \left( 2^{\xi +1} |\S|^{\xi + 1} \theta \left( R^\xi + \exp(4 \beta R)^\xi + \frac{\exp(2 \beta R)^\xi}{\beta^\xi}\right)\right) z^{-\xi}\\
         &= \theta^L z^{-\xi^L}
    \end{align*}
    \endgroup
    where $\theta^L = 2^{\xi +1} |\S|^{\xi + 1} \theta \left( R^\xi + \exp(4 \beta R)^\xi + \frac{\exp(2 \beta R)^\xi}{\beta^\xi}\right)$, $\xi^L = \xi$, and $\eta^L = \eta$. We can follow similar steps to upper bound $\mathbb{P}\left[n\hat{\rho}_{i,n}^{\text{int.}} - n\tilde{\mu}_i \le -n^{\eta} z \right]$.
\end{proof}

\noindent \textbf{Proof of part 1 of Lemma~\ref{lemma:non-deterministic-bandit-convergence-and-concentration-auxiliary} (convergence result):}

\begin{proof}
    Regarding the convergence in expectation, i.e., the fact that $\lim_{n \rightarrow \infty} \mathbb{E}\left[\hat{\rho}_{i,n}^{\text{int.}}\right] = \tilde{\mu}_i$, we start by noting that, for any $\varepsilon >0$, we have that
    \begingroup
    \allowdisplaybreaks
    \begin{align*}
        \mathbb{P}[|\hat{\rho}_{i,n}^{\text{int.}}-\tilde{\mu}_i| \ge \varepsilon] &= \mathbb{P}[n|\hat{\rho}_{i,n}^{\text{int.}}-\tilde{\mu}_i| \ge n\varepsilon] \\
        &\le \mathbb{P}[n(\hat{\rho}_{i,n}^{\text{int.}}-\tilde{\mu}_i) \ge n\varepsilon] + \mathbb{P}[n(\hat{\rho}_{i,n}^{\text{int.}}-\tilde{\mu}_i) \le -n\varepsilon] \\
        &\overset{(a)}{\le} 2 \frac{\theta^L}{(n^{1-\eta^L} \varepsilon)^\xi},
    \end{align*}
    \endgroup
    where (a) follows from the already proven concentration bound by noting that, if $z=n^{1-\eta^L} \varepsilon$ and $z\ge 1$, then $\mathbb{P}[n(\hat{\rho}_{i,n}^{\text{int.}}-\tilde{\mu}_i) \ge n^{\eta^L} z] = \mathbb{P}[n(\hat{\rho}_{i,n}^{\text{int.}}-\tilde{\mu}_i) \ge n\varepsilon]$ and we apply the already derived concentration result; otherwise $z<1$, the upper bound is greater than one and the result also holds. We also note that $|\hat{\rho}_{i,n}^{\text{int.}}-\tilde{\mu}_i| \le 2R$ since both $|\hat{\rho}_{i,n}^{\text{int.}}| \le R$ and $|\tilde{\mu}_i|\le R$. Hence, for any $\varepsilon > 0$,
    \begingroup
    \allowdisplaybreaks
    \begin{align*}
        \mathbb{E}\big[|\hat{\rho}_{i,n}^{\text{int.}}-\tilde{\mu}_i|\big] &= \mathbb{E}\big[|\hat{\rho}_{i,n}^{\text{int.}}-\tilde{\mu}_i|\mathbf{1}_{|\hat{\rho}_{i,n}^{\text{int.}}-\tilde{\mu}_i|\le\varepsilon}\big] +\mathbb{E}\big[|\hat{\rho}_{i,n}^{\text{int.}}-\tilde{\mu}_i|\mathbf{1}_{|\hat{\rho}_{i,n}^{\text{int.}}-\tilde{\mu}_i|>\varepsilon}\big]  \\
        &\le \varepsilon + 2R\mathbb{P}[|\hat{\rho}_{i,n}^{\text{int.}}-\tilde{\mu}_i|>\varepsilon]\\
        &\le \varepsilon + 4R \frac{\theta^L}{(n^{1-\eta^L} \varepsilon)^\xi},
    \end{align*}
    \endgroup
    and therefore, for any $\varepsilon > 0$, taking the limit $ n \rightarrow \infty$ yields the desired result.
\end{proof}

\subsection{Main result}
\label{appendix:proof-non-deterministic-bandit-convergence-and-concentration:main-result}

\begin{proof}
    We start by noting that Lemma~\ref{lemma:non-deterministic-bandit-convergence-and-concentration-auxiliary} holds for any $\beta$ value. In the context of this proof, we will invoke Lemma~\ref{lemma:non-deterministic-bandit-convergence-and-concentration-auxiliary} with $\beta = \beta \gamma$. \\

    \noindent It holds that
    \begin{align*}
    \hat{\rho}_{i,n} &=\frac{1}{\beta} \ln \left( \frac{1}{n} \sum_{s' \in \S} \sum_{t=1}^{T_i^{s'}(n)} \exp(\beta ( c_i + \gamma \rvar{x}_{i,t}^{s'}) ) \right) \\
    &= c_i + \frac{\gamma}{\beta \gamma } \ln \left( \frac{1}{n} \sum_{s' \in \S} \sum_{t=1}^{T_i^{s'}(n)} \exp\left(\beta \gamma \rvar{x}_{i,t}^{s'} \right) \right) \\
    &= c_i + \gamma \hat{\rho}_{i,n}^{\text{int.}},
    \end{align*}
    where above we consider estimator $\hat{\rho}_{i,n}^{\text{int.}}$ as defined in \eqref{eq:rho_intermediate_def} but with $\beta = \beta \gamma$. \\

    \noindent Exploiting Lemma~\ref{lemma:non-deterministic-bandit-convergence-and-concentration-auxiliary}, for $\mu_{i,n} = \mathbb{E}[\hat{\rho}_{i,n}]$ as introduced in Sec.~\ref{sec:bandits:non-stat-non-det-risk-aware-bandit}, we have that
    \begin{align*}
    \lim_{n \rightarrow \infty} \mu_{i,n} &= \lim_{n \rightarrow \infty} \mathbb{E}[\hat{\rho}_{i,n}] \\
    &= \lim_{n \rightarrow \infty} \mathbb{E}[c_i + \gamma \hat{\rho}_{i,n}^{\text{int.}}] \\
    &= c_i + \gamma \lim_{n \rightarrow \infty} \mathbb{E}[\hat{\rho}_{i,n}^{\text{int.}}] \\
    &\overset{(a)}{=} c_i + \gamma \text{ERM}_{\beta\gamma}(\mu_i^{\rvar{s}'}) \\
    &\overset{(b)}{=} \text{ERM}_{\beta}\left(c_i + \gamma \mu_i^{\rvar{s}'} \right) \\
    &\overset{(c)}{=} \mu_i
    \end{align*}
    where (a) follows from Lemma~\ref{lemma:non-deterministic-bandit-convergence-and-concentration-auxiliary} and (b) holds since
    $$c + \gamma \text{ERM}_{\beta\gamma}( \rvar{z}) = c + \text{ERM}_{\beta}(\gamma \rvar{z}) = \text{ERM}_{\beta}(c + \gamma \rvar{z}),$$
    where $c$ and $\gamma$ are positive constants and $\rvar{z}$ is a random variable. Step (c) holds by definition. Hence, we have just showed that $\lim_{n \rightarrow \infty} \mu_{i,n} = \mu_i$. \\

    \noindent Regarding the concentration properties, estimator $\hat{\rho}_{i,n}$ satisfies similar concentration properties as those of Lemma~\ref{lemma:non-deterministic-bandit-convergence-and-concentration-auxiliary}. In particular, we note that there exist constants $\theta^L > 1$, $\xi^L > 1 $, and $1/2 \le \eta^L < 1$ such that, for any $z \ge 1$ and $n \in \mathbb{N}$
    \begin{align*}
        \mathbb{P}\left[n\hat{\rho}_{i,n} - n\mu_i \ge n^{\eta^L} z \right] &= \mathbb{P}\Big[n \Big( c_i + \gamma \hat{\rho}_{i,n}^{\text{int.}} \Big) \\
        &\quad \quad - n\Big( c_i + \gamma \text{ERM}_{\beta\gamma}(\mu_i^{\rvar{s}'}) \Big) \ge n^{\eta^L} z \Big] \\
        &= \mathbb{P}\left[n \left( \hat{\rho}_{i,n}^{\text{int.}} - \text{ERM}_{\beta\gamma}(\mu_i^{\rvar{s}'}) \right) \ge n^{\eta^L} \frac{z}{\gamma} \right] \\
        &\overset{(a)}{\le} \frac{\theta^L}{z^{\xi^L}},
    \end{align*}
    where (a) follows from Lemma~\ref{lemma:non-deterministic-bandit-convergence-and-concentration-auxiliary} by noting that from the Lemma we have
    $$\mathbb{P}\left[n \left( \hat{\rho}_{i,n}^{\text{int.}} - \text{ERM}_{\beta\gamma}(\mu_i^{\rvar{s}'}) \right) \ge n^{\eta^L} z' \right] \le \frac{\theta^L}{(z')^{\xi^L}}$$
    for any $z' \ge 1$ and thus it also holds by letting $z' = z/\gamma$ for any $z\ge 1$ and $\gamma \in (0,1]$. We recall that $\theta^L$ is defined in Lemma~\ref{lemma:non-deterministic-bandit-convergence-and-concentration-auxiliary}, $\eta^L = \eta$, and $\xi^L = \xi$. We can follow similar steps to upper bound the other direction.
\end{proof}

\section{Proof of Theorem \ref{theo:mcts-result}}
\label{appendix:proof_of_mcts_result}
The proof of Theo. \ref{theo:mcts-result} follows by successively applying Theo.~\ref{theo:non-stationary-non-deterministic-bandit-results-stream-estimator} to trees. In particular, we first show that the assumptions of Theo.~\ref{theo:non-stationary-non-deterministic-bandit-results-stream-estimator} are satisfied at the last level of the tree (depth $H$). Then, by successive application of Theo.~\ref{theo:non-stationary-non-deterministic-bandit-results-stream-estimator}, we establish that polynomial concentration is attained throughout the tree up to the root node.

\subsection{Base case: Analyzing leaf level $H$}
We start by showing that Assumption~\ref{assumption:non-deterministic-bandit-assumption} is satisfied at the leaf nodes (depth $H$) of the tree. At level $H$, for any state $s_{H} \in \S$, we observe a sequence of independent and identically distributed random costs $(\rvar{x}_1, \rvar{x}_2, \ldots)$ taking values in $[-R,R]$ sampled from the terminal cost function $c_H$\footnote{We assume here the most general case where the terminal costs can be random. We note, however, that the case where the costs are deterministic directly follows as a particular case.}. For any $\beta > 0$, let $\hat{\rho}_{s_H, n} = \frac{1}{\beta_{H}} \ln \left(\frac{1}{n}\sum_{t=1}^n \exp(\beta_{H} \rvar{x}_t) \right)$, be the empirical ERM estimator for the stream of costs associated with the terminal state $s_H \in \S$. Let also $\mu_{s_H,n} = \mathbb{E}[\hat{\rho}_{s_H,n}]$. It holds that
\begin{enumerate}
    \item $\lim_{n \rightarrow \infty} \mu_{s_H,n} = \text{ERM}_{\beta_H}(c_H(s_H))$. We let $\mu^*_{s_H} = \text{ERM}_{\beta_H}(c_H(s_H))$ by definition.
    \item there exist constants $\theta_{H} > 1$, $\xi_{H} > 1 $, and $1/2 \le \eta_{H} < 1$ such that, for any $z \ge 1$ and $n \in \mathbb{N}$,
    \begin{equation*}
        \mathbb{P}\left[n\hat{\rho}_{s_H, n} - n\mu^*_{s_H} \ge n^{\eta_H} z \right] \le \frac{\theta_H}{z^{\xi_H}}, \quad \mathbb{P}\left[n\hat{\rho}_{s_H, n} - n\mu^*_{s_H} \le - n^{\eta_H} z \right] \le \frac{\theta_H}{z^{\xi_H}}.
    \end{equation*}
\end{enumerate}
The first point above follows from the discussion in Sec.~\ref{appendix:ERM-estimator-properties} (the paragraph before Lemma~\ref{lemma:polynomial-concentration-erm-estimator}) and the second point above follows from Lemma~\ref{lemma:polynomial-concentration-erm-estimator}. Hence, Assumption~\ref{assumption:non-deterministic-bandit-assumption} is satisfied for any terminal node in the tree. We also emphasize that, according to Lemma~\ref{lemma:polynomial-concentration-erm-estimator}, the polynomial bound remains valid for any $\xi_H > 1$ and $1/2 \le \eta_H < 1$, and any sufficiently high $\theta_H$. In other words, parameters $\xi_H$, $\eta_H$, and $\theta_H$ can be freely chosen at the leaf nodes (as far as $\theta_H$ is sufficiently high).

\subsection{Induction step: From level $h+1$ to level $h$}
We focus our attention to a given node $s_{h} \in \S$ at depth $h \in \{0, \ldots, H-1\}$ in the tree. For each action $a \in \A$ and next state $s_{h+1} \in \S$, we let, similar to Sec.~\ref{sec:bandits:non-stat-non-det-risk-aware-bandit}, 
\begin{equation}
    \hat{\rho}_{(s_{h},a),n}^{s_{h+1}} = \frac{1}{\beta_{h+1}} \ln \left(\frac{1}{n}\sum_{t=1}^n \exp(\beta_{h+1} \rvar{x}_{(s_{h},a),t}^{s_{h+1}}) \right), \label{eq:mcts_next_state_estimator}
\end{equation}
where $(\rvar{x}_{(s_{h},a),1}^{s_{h+1}}, \rvar{x}_{(s_{h},a),2}^{s_{h+1}}, \ldots)$ is the sequence of discounted cumulative costs starting at next state $s_{h+1}$. Let also $\mu_{(s_{h},a),n}^{s_{h+1}} = \mathbb{E}[\hat{\rho}_{(s_{h},a),n}^{s_{h+1}}]$. \\

\noindent As an induction hypothesis, we assume that layer $h+1$ in the tree satisfies Assumption~\ref{assumption:non-deterministic-bandit-assumption}, i.e., for any action $a \in \A$ and next state $s_{h+1} \in \S$ it holds that: (i) $\lim_{n \rightarrow \infty} \mu_{(s_{h},a),n}^{s_{h+1}} = \mu_{(s_{h},a)}^{s_{h+1}}$; and (ii) there exist constants $\alpha_{h+1}$, $\xi_{h+1}$, $\eta_{h+1}$, and $\theta_{h+1}$ such that the ERM estimator $\hat{\rho}_{(s_{h},a),n}^{s_{h+1}}$ attains polynomial concentration with respect to its limiting value $\mu_{(s_{h},a)}^{s_{h+1}}$. For $h = H-1$, we note that in the base case above we already showed that (i) and (ii) are verified at the leaf nodes (depth $H$). \\

\noindent Estimator $\hat{\rho}_{(s_h,a),n}$, as introduced in \eqref{eq:mcts_state_node_erm_estimator}, satisfies 
\begin{align*}
    \hat{\rho}_{(s_h,a),n} &= \frac{1}{\beta_h} \ln \left( \frac{1}{n} \sum_{t=1}^{n} \exp\left(\beta_h ( c_{h}(s_h,a) + \gamma \rvar{x}_{(\rvar{s}_h,a),t}^{\rvar{s}_{h+1}} ) \right) \right) \\
    &\overset{(a)}{=} \frac{1}{\beta_h} \ln \left( \frac{1}{n} \sum_{t=1}^{n} \exp\left(\beta_h ( c_{h}(s_h,a) + \gamma \rvar{x}_{(\rvar{s}_h,a),t}^{\rvar{s}'_a(t)} ) \right) \right) \\
    &= \frac{1}{\beta_h} \ln \left( \frac{1}{n} \sum_{s' \in \S} \sum_{t=1}^{T_a^{s'}(n)} \exp\left(\beta_h ( c_{h}(s_h,a) + \gamma \rvar{x}_{(\rvar{s}_h,a),t}^{s'} ) \right) \right)
\end{align*}
where in (a) $\rvar{s}'_a(t)$ is a random function mapping the timestep $t$ to a random next state $\rvar{s}_{h+1} \in \S$. We also highlight that $n$ above corresponds to the number of times action $a$ was selected in state $s_h$. Therefore, estimator \eqref{eq:mcts_state_node_erm_estimator} is exactly of the same form as estimator \eqref{eq:entropic_risk_estimator_non_det_bandit}, as introduced in Sec.~\ref{sec:bandits:non-stat-non-det-risk-aware-bandit}. \\

\noindent Furthermore, from the perspective of node $s_h$ in the search tree, the action selection mechanism prescribed by the risk-aware MCTS algorithm, i.e., \eqref{eq:mcts_action-selection}, is equivalent to the action selection mechanism we present in Sec.~\ref{sec:bandits:non-stat-non-det-risk-aware-bandit} in the context of a non-stationary non-deterministic bandits, i.e., \eqref{eq:ucb_erm_algorithm_2}. The key difference is that we replace estimator $\hat{\rho}_{i,n}$, as defined in \eqref{eq:entropic_risk_estimator_non_det_bandit}, with an equivalent estimator $\hat{\rho}_{(s_h,a),n}$, as defined in \eqref{eq:mcts_state_node_erm_estimator}. Also, (i) we replace the timestep $t$ in \eqref{eq:ucb_erm_algorithm_2} with the number of visits to the state, $N(s_h)$, to obtain \eqref{eq:mcts_action-selection}; (ii) we replace $T_i(n)$ in \eqref{eq:ucb_erm_algorithm_2} with the number of times each action $a$ has been selected while in state $s_h$, $N(s_h,a)$, to obtain \eqref{eq:mcts_action-selection}; and (iii) parameters $\theta^L$, $\xi^L = \xi$, $\eta^L =\eta$, and $\alpha$ in \eqref{eq:ucb_erm_algorithm_2} are now indexed by the depth in \eqref{eq:mcts_action-selection}, becoming $\theta^L_{h+1}$, $\xi_{h+1}$, $\eta_{h+1}$, and $\alpha_{h+1}$. \\

\noindent Similar to Sec.~\ref{sec:bandits:non-stat-non-det-risk-aware-bandit}, we define the following quantities:
\begingroup
\allowdisplaybreaks
\begin{align*}
    \mu_{(s_h,a),n} &= \mathbb{E}[\hat{\rho}_{(s_h,a),n}], \\
    \mu_{(s_h,a)} &= \text{ERM}_{\beta_{h}}\left(c_{h}(s_h,a) + \gamma \mu_{(s_{h},a)}^{\rvar{s}_{h+1}}\right)\\
    &= \frac{1}{\beta_h} \ln \left( \mathbb{E}_{\rvar{s}_{h+1} \sim P^a(s_{h}, \cdot)}\left[ \exp(\beta_h (c_h(s_h,a) + \gamma \mu_{(s_{h},a)}^{\rvar{s}_{h+1}}) )\right] \right), \\
    \mu^*_{s_{h}} &= \min_{a \in \A} \mu_{(s_{h},a)}, \\
    a^*_{s_h} &= \argmin_{a \in \A} \mu_{(s_{h},a)}, \\
    \Delta_{(s_h,a)} &= \mu_{(s_h,a)} - \mu^*_{s_{h}}, \\
    \Delta_{\text{min}}^{s_h} &= \min_{a \in \A \setminus a^*_{s_h}} \Delta_{(s_h,a)}.
\end{align*}
\endgroup

\noindent We also define $R_h = (\gamma^{1+(H-h)} - 1) R / (\gamma - 1)$. Since the costs are bounded given Assumption~\ref{assumption:mdp_assumptions}, it holds that $c_{h}(s_h,a) + \gamma \rvar{x}_{(s_{h},a)}^{s_{h+1}} \in [-R_h,R_h]$ and $\hat{\rho}_{(s_h,a),n} \in [-R_h,R_h]$ for any $s_h, s_{h+1} \in \S$ and $a \in \A$.

\noindent Letting $\bar{\rho}_{s_h,n} = \frac{1}{n} \sum_{a \in \A} T_a(n) \hat{\rho}_{i, T_a(n)}$, where $n$ refers to the number of times node $s_h \in \S$ has been visited and $T_a(n)$ to the number of times action $a \in \A$ was selected once decision node $s_h$ has been visited $n$ times, we can invoke Theo.~\ref{theo:non-stationary-non-deterministic-bandit-results} with the choice of variables above defined, establishing convergence and concentration of $\bar{\rho}_{s_h,n}$ with respect to its limiting value $\mu^*_{s_{h}}$. However, we are particularly interested in invoking Theo.~\ref{theo:non-stationary-non-deterministic-bandit-results-stream-estimator} instead since it provides convergence and concentration results for the stream of costs at decision node $s_h$, instead of $\bar{\rho}_{s_h,n}$ as considered by Theo.~\ref{theo:non-stationary-non-deterministic-bandit-results}, allowing us to recursively apply the convergence and concentration results throughout the tree up until the root node. Therefore, let
\begin{equation}
    \hat{\rho}_{s_h,n}^\text{stream} = \frac{1}{\beta_h} \ln\left( \frac{1}{n} \sum_{a \in \A} \sum_{s' \in \S} \sum_{t=1}^{T_a^{s'}(T_a(n))} \exp(\beta_h ( c_{h}(s_h,a) + \gamma \rvar{x}_{(\rvar{s}_h,a),t}^{s'} ))\right) \label{eq:mcts_stream_estimator}
\end{equation}
where $n$ above refers to the number of times decision node $s_h$ has been visited. Estimator \eqref{eq:mcts_stream_estimator} has the same form as estimator \eqref{eq:entropic_risk_estimator_non_det_bandit_stream_estimator}. Thus, we can directly invoke Theo.~\ref{theo:non-stationary-non-deterministic-bandit-results-stream-estimator} with the variables above defined to establish the convergence and concentration of $\hat{\rho}_{s_h,n}^\text{stream}$. For convenience, we restate below Theo.~\ref{theo:non-stationary-non-deterministic-bandit-results-stream-estimator} in terms of the variables above defined.

\begin{theorem} \label{theo:appendix-mcts-recursive-theorem}
    Let $\hat{\rho}_n^\text{stream}$, as defined in \eqref{eq:mcts_stream_estimator}, be the empirical $\text{ERM}_{\beta_h}$ of the stream of costs obtained at the decision node $s_h$ in the tree when running action-selection algorithm \eqref{eq:mcts_action-selection} under Assumption~\ref{assumption:non-deterministic-bandit-assumption}. Furthermore, for any $1/2 \le \eta_{h+1} < 1$, assume Assumption~\ref{assumption:non-deterministic-bandit-assumption} holds with a sufficiently large $\xi_{h+1} > 1$ value such that there exists $\alpha_{h+1} > 2$ satisfying $\xi_{h+1}\eta_{h+1}(1-\eta_{h+1}) \le \alpha_{h+1} < \xi_{h+1} (1-\eta_{h+1})$. Then,
    \begin{enumerate}
        \item it holds that $|\mathbb{E}[\hat{\rho}_{s_h,n}^\text{stream}] - \mu^*_{s_h}| = \O\left(n^{\frac{\alpha_{h+1}}{\xi_{h+1} (1-\eta_{h+1})} -1}\right)$, i.e., $\lim_{n \rightarrow \infty} \mathbb{E}[\hat{\rho}_{s_h,n}^\text{stream}] = \mu^*_{s_h}$.
        \item there exist constants $\theta_h >1$, $\xi_h > 1$, and $1/2 \le \eta_h < 1$ such that, for every $n \in \mathbb{N}$ and every $z \ge 1$, it holds that
    \begin{equation*}
        \mathbb{P}\left[n\hat{\rho}_{s_h,n}^\text{stream} - n\mu^*_{s_h} \ge n^{\eta_h} z \right] \le \frac{\theta_h}{z^{\xi_h}}, \quad \text{ and } \quad \mathbb{P}\left[n\hat{\rho}_{s_h,n}^\text{stream} - n\mu^*_{s_h} \le - n^{\eta_h} z \right] \le \frac{\theta_h}{z^{\xi_h}},
    \end{equation*}
    where $\eta_h = \frac{\alpha_{h+1}}{\xi_{h+1} (1-\eta_{h+1})}$, $\xi_h = \alpha_{h+1} - 1$, and $\theta_h$ depends on $R_h$, $K=|\A|$, $\Delta_{\text{min}}^{s_h}$, $\beta_h$, $\theta_{h+1}$, $\theta_{h+1}^L$, $\xi_{h+1}$, $\alpha_{h+1}$, and $\eta_{h+1}$.
    \end{enumerate}
\end{theorem}

\noindent Therefore, from the theorem above, we know that the empirical $\text{ERM}_{\beta_h}$ for the stream of costs obtained at node $s_h$, $\hat{\rho}_{s_h,n}^\text{stream}$, satisfies the convergence and concentration properties above. This is precisely the assumption needed to apply Theo.~\ref{theo:appendix-mcts-recursive-theorem} from the perspective of a node at depth $h-1$. This is because estimator $\hat{\rho}_{s_h,n}^\text{stream}$, as defined in \eqref{eq:mcts_stream_estimator}, now corresponds to estimator $\hat{\rho}_{(s_{h-1},a),n}^{s_{h}}$, as defined in \eqref{eq:mcts_next_state_estimator}, from the perspective of a decision node $s_{h-1}$ above in the tree for a particular action $a \in \A$. Thus, we can recursively apply Theo.~\ref{theo:appendix-mcts-recursive-theorem} to obtain the convergence of the empirical $\text{ERM}_{\beta_h}$ to the optimal ERM throughout the tree until the root node.

\noindent We also highlight that, at each decision node, the limiting empirical $\text{ERM}_{\beta_h}$ satisfies the Bellman optimality equations as introduced in \eqref{eq:erm_bellman_optimality equations} and, thus, the actions selected by the algorithm are optimal. In particular, at the root node $s_0 \in \S$, $\mu^*_{s_0}$ corresponds precisely to the optimal risk-aware value function $V^*_0(s_0)$ for the initial state $s_0$. Also, at the root node $s_0 \in \S$, $\hat{\rho}_{s_0,n}^\text{stream}$ corresponds to the empirical $\text{ERM}_\beta$ for the stream of costs obtained at the root and be equivalently written as 
\begin{align*}
    \hat{\rho}_{s_h,n}^\text{stream} &= \frac{1}{\beta_h} \ln\left( \frac{1}{n} \sum_{a \in \A} \sum_{s' \in \S} \sum_{t=1}^{T_a^{s'}(T_a(n))} \exp(\beta_h ( c_{h}(s_h,a) + \gamma \rvar{x}_{(\rvar{s}_h,a),t}^{s'} ))\right) \\
    &= \frac{1}{\beta} \ln \left( \frac{1}{n} \sum_{a \in \A} \sum_{t=1}^{T_a(n)} \exp\left(\beta \rvar{x}_{(s_0,a),t}\right) \right) \\
    &\overset{(a)}{=} \hat{V}_{n}(s_0),
\end{align*}
where (a) holds by definition. \\

\noindent All that remains is to investigate the existence of parameters
$$\{\alpha_h\}_{h \in \{1, \ldots, H\}}, \quad \{\theta_h\}_{h \in \{1, \ldots, H\}}, \quad \{\eta_h\}_{h \in \{1, \ldots, H\}}, \quad \text{ and } \quad \{\xi_h\}_{h \in \{1, \ldots, H\}}$$
such that the conditions required to apply Theo.~\ref{theo:appendix-mcts-recursive-theorem} at each level of the tree are satisfied. We address this in the next section.

\subsection{Defining parameters $\{\alpha_h\}$, $\{\theta_h\}$, $\{\eta_h\}$, and $\{\xi_h\}$}
Regarding parameters $\{\theta_h\}_{\{1, \ldots, H-1\}}$, at depth $H$ it holds from Lemma~\ref{lemma:polynomial-concentration-erm-estimator} that the $\text{ERM}_\beta$ estimator satisfies polynomial concentration with a given $\theta_H$ value. Then, Theo.~\ref{theo:appendix-mcts-recursive-theorem} recursively defines parameters $\theta_h$ along the tree starting from the leaves up until the root node. The only condition we require on these parameters is that $\theta_h > 1$ at every depth $h \in \{1, \ldots, H\}$. This condition can be easily enforced by increasing the value of $\theta_h$ at any depth (if needed) so as to meet $\theta_h > 1$ (replacing $\theta_h$ with any higher value yields a valid bound).

We now turn our attention to the remaining parameters $\{\xi_h\}_{\{1, \ldots, H\}}$, $\{\eta_h\}_{\{1, \ldots, H\}}$, and $\{\alpha_h\}_{\{1, \ldots, H\}}$. We start by noting that, at depth $H$, parameters $\xi_H > 1$ and $1/2 \le \eta_H < 1$ can be freely chosen given Lemma~\ref{lemma:polynomial-concentration-erm-estimator}. Parameters $\{\xi_h\}_{\{1, \ldots, H-1\}}$, $\{\eta_h\}_{\{1, \ldots, H-1\}}$, and $\{\alpha_h\}_{\{1, \ldots, H-1\}}$ must be chosen in such a way that
\begin{align*}
    \xi_h &= \alpha_{h+1} - 1, \; \forall h \in \{1, \ldots, H-1\}, \\
    \eta_h &= \frac{\alpha_{h+1}}{\xi_{h+1}(1-\eta_{h+1})}, \; \forall h \in \{1, \ldots, H-1\}.
\end{align*}
The conditions above arise from Theo.~\ref{theo:appendix-mcts-recursive-theorem}. Furthermore, we need to satisfy, for every depth $h \in \{1, \ldots, H\}$, that $\xi_h > 1$, and that there exists $\alpha_h > 2$ satisfying $\xi_{h}\eta_{h}(1-\eta_{h}) \le \alpha_{h} < \xi_{h} (1-\eta_{h})$. Similar to \cite{shah_2020}, for any $1/2 \le \eta < 1$, it can be verified that
\begin{align*}
    \eta_h &= \eta, \; \forall h \in \{1, \ldots H\} \\
    \xi_h &= \alpha_{h+1} - 1, \; \forall h \in \{1, \ldots, H-1\}, \text{ and } \xi_{H} > 1 \text{ is arbitrary}, \\
    \alpha_h &= \eta(1-\eta)\xi_h, \; \forall h \in \{1, \ldots, H\},
\end{align*}
yields a valid set of parameters verifying the desired constraints, as far as $\xi_H$ is sufficiently high. \\

\noindent We illustrate the relationship between these different parameters in Fig.~\ref{fig:mcts-recursive-params}.

\begin{figure}[h]
    \centering
    \includegraphics[width=0.95\textwidth]{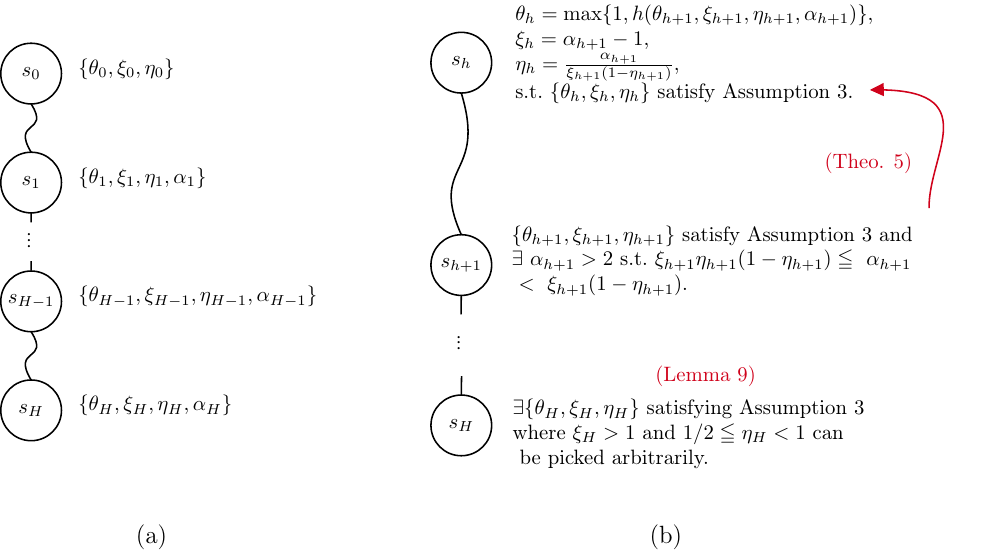}
    \caption{Illustration of the recursive relation between parameters $\{\alpha_h\}$, $\{\theta_h\}$, $\{\eta_h\}$, and $\{\xi_h\}$.}
    \label{fig:mcts-recursive-params}
\end{figure}

\clearpage
\section{Risk-aware MCTS pseudocode}
\label{appendix:mcts_pseudocode}
\begin{figure}[h]
    \centering
    \includegraphics[width=0.95\textwidth]{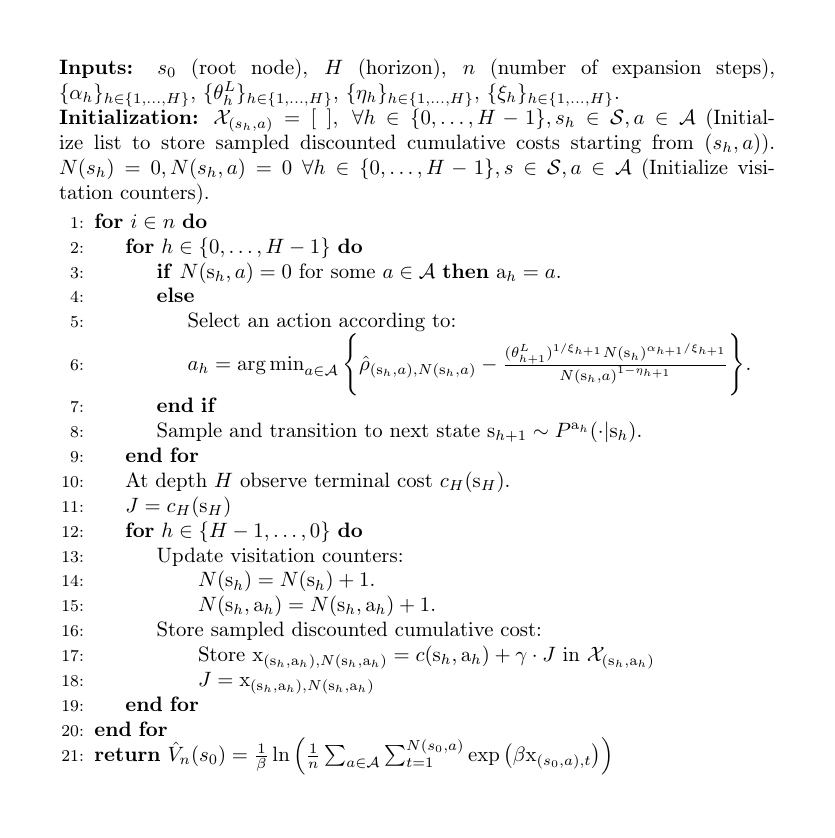}
    \caption{Risk-aware MCTS pseudocode.}
    \label{fig:risk-aware-mcts-pseudocode}
\end{figure}

\clearpage
\section{Empirical results}
\label{appendix:empirical-results}

\subsection{Environments}

\paragraph{MDP-4}
Environment MDP-4 consists of a four-state MDP where the agent needs to tradeoff, at the initial state ($s_0$), between a safe action $a_1$ that leads to a medium cost state ($s_1$) and a risky action $a_0$ that can lead to both low ($s_2$) or high ($s_3$) cost states. When picking action $a_0$ at state $s_0$, with probability $1-\epsilon$ the agent transitions to state $s_2$ and with probability $\epsilon$ the agent transitions to state $s_3$. States $s_1$, $s_2$, and $s_3$ are absorbing, but we reset the agent to the initial state with $10\%$ probability at every timestep independent of the chosen action. The state-dependent cost function is: $c(s_0) = 0, \; c(s_1) = 5, c(s_2) = 1, \; c(s_3) = 20$. We refer to Fig.~\ref{fig:mdp-4-illustration} for an illustration of the environment.

\begin{figure}[h]
    \centering
    \includegraphics[width=0.25\textwidth]{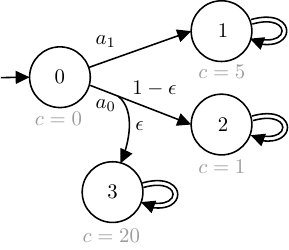}
    \caption{MDP-4 environment illustration.}
    \label{fig:mdp-4-illustration}
\end{figure}

\paragraph{Grid-MDP}
Environment Grid-MDP consists of a $5\times3$ grid, and an additional absorbing pitfall state $(s_{\text{pitfall}})$. 
The agent starts in state $(4,0)$ and aims to reach the target absorbing state $(s_{\text{target}})=(0,0)$. From the start state, there are two distinct paths to the target, characterized by different levels of risk. 
Risk is modeled through \textit{slippery states} located along the two possible paths. When the agent takes an action in a slippery state, it transitions to $(s_{\text{pitfall}})$ with probability $\epsilon=0.01$, where for each action the cost is $c(s_{\text{pitfall}}) = 5$. 
The safer path contains a single slippery state, whereas the riskier path contains three. 
Consequently, the safer path has a lower probability of failure, while the riskier path offers a lower expected cumulative cost but a higher likelihood of transitioning to the pitfall state.
Once the agent reaches the target state, the cost of each action is $c(s_{\text{target}})=0$.

\paragraph{Bin-Packing environment}
We introduce a stochastic sequential bin-packing environment that models the task of
a robot packing grocery items into bags at a checkout counter.
The environment is designed to exhibit meaningful tail risk: most packing sequences
incur only mild costs from packing inefficiency, but a small fraction result in
high-cost accident events (crushed fragile items, liquid items spilled).

An episode consists of $H$ steps, one per item to be packed.
At each timestep $t \in \{0, \dots, H{-}1\}$, the agent observes the current item
(drawn uniformly at random from the item catalogue) together with the current state
of all bags, and selects a bag into which to place the item.
The catalogue contains 23 grocery items drawn from common household shopping lists.
Each item $i$ is characterised by six attributes:
\begin{itemize}
    \item $w_i \in \mathbb{R}_{>0}$: estimated weight (grams).
    \item $v_i \in \mathbb{R}_{>0}$: estimated volume (cm$^3$).
    \item $\kappa_i \in \{1,\dots,10\}$: crush score (higher = more fragile).
    \item $\phi_i \in \{0,1\}$: spill risk flag (e.g.\ open liquids).
    \item $\psi_i \in \{0,1\}$: spill vulnerability flag (item is damaged by spillage).
    \item $o_i \in \{0,1\}$: boolean flag that states if the item is sensitive to orientation changes. If true, it makes spilling events more likely. 
\end{itemize}

The state of each bag $b$ comprises three groups of quantities:
\begin{itemize}
    \item \textit{Capacity}: the volume and weight capacities $V^b_{\max}$ and $W^b_{\max}$,
    and the cumulative volume $v^b_{\mathrm{tot}}$ and weight $w^b_{\mathrm{tot}}$ of the items
    it currently contains.
    \item \textit{Composition}: the number of fragile items ($\kappa_i \ge 8$), spill-risk
    items ($\phi_i = 1$), and spill-vulnerable items ($\psi_i = 1$) in the bag.
    \item \textit{Risk and damage}: the total risk $R^b = \sum_{i \in b} r_i$, where $r_i$ is
    the per-item risk factor defined below, and the number of items in the bag that have
    been crushed or have spilled.
\end{itemize}
Each fragile or spill-risk item $i$ carries a risk factor $r_i$, initialised to zero,
which grows with the weight subsequently placed on top of it.
When a new item $j$ is placed in bag $b$, the risk factor of every intact fragile item
$i \in b$ is updated as
\[
    r_i \leftarrow r_i + \alpha_j\, w_j,
    \qquad
    \alpha_j =
    \begin{cases}
        1/500 & \text{if } \kappa_j \le 2 \text{ (hard item)},\\
        1/1000 & \text{otherwise},
    \end{cases}
\]
after which item $i$ is crushed independently with probability $\min(0.05\, r_i,\, 1)$.
Spill-risk items are updated analogously, with the increment additionally scaled by the
bag's squeeze ratio $z^b = v^b_{\mathrm{tot}} / V^b_{\max}$, so that spills become more likely
as the bag fills up:
\[
    r_i \leftarrow r_i + \alpha_i\, w_j\, z^b,
    \qquad
    \alpha_i =
    \begin{cases}
        1/500 & \text{if } o_i = 1 \text{ (orientation-sensitive)},\\
        1/1000 & \text{otherwise}.
    \end{cases}
\]

At each step the agent incurs a cost
\[
    c_t = 10 \cdot \mathbf{1}[\text{new bag opened}]
        + 50 \cdot N^{\mathrm{crush}}_t
        + 40 \cdot N^{\mathrm{spill}}_t
        + 10 \cdot N^{\mathrm{cont}}_t
        + 100 \cdot\mathbf{1}[\text{illegal move}],
\]
where $N^{\mathrm{crush}}_t$ and $N^{\mathrm{spill}}_t$ are the numbers of crush and spill
events triggered by the placement, and $N^{\mathrm{cont}}_t$ is the number of
spill-vulnerable items contaminated by those spills.
A move is illegal if placing the item would exceed the bag's volume or weight capacity.
The bag-opening cost captures packing inefficiency and accrues steadily over an episode.
Crush and spill events are rare but costly, and they produce the heavy right tail of the
return distribution.

\clearpage
\subsection{Complete experimental results}

\subsubsection{MDP-4}

\textbf{$K_\text{UCB} = 0.3$ (additional results):}

\begin{table}[h]
    \centering
    \resizebox{0.45\columnwidth}{!}{%
    \begin{tabular}{|c|c||c|c|}
    \hline
    $\beta$ & \begin{tabular}[c]{@{}c@{}} \textsc{ERM-BI} \\ (Oracle) \end{tabular} & \textsc{Acc-MCTS}                                                 & \begin{tabular}[c]{@{}c@{}} \textsc{ERM-MCTS} \\ (Ours)\end{tabular} \\ \hline \hline
    0.01    & \begin{tabular}[c]{@{}c@{}}0.1704\\ (-0.0168, +0.0174) \end{tabular}  & \begin{tabular}[c]{@{}c@{}}\textbf{0.1704}\\ (-0.0169, +0.0172) \end{tabular} & \begin{tabular}[c]{@{}c@{}}\textbf{0.1704}\\ (-0.0171, +0.0175)\end{tabular}  \\ \hline
    5       & \begin{tabular}[c]{@{}c@{}}0.2155\\ (-0.0018, +0.0018) \end{tabular}  & \begin{tabular}[c]{@{}c@{}}0.2424\\ (-0.0233, +0.0269)\end{tabular} & \begin{tabular}[c]{@{}c@{}}\textbf{0.2151}\\ (-0.0017, +0.0017)\end{tabular}  \\ \hline
    25      & \begin{tabular}[c]{@{}c@{}}0.2188\\ (-0.0010, +0.0010)\end{tabular}  & \begin{tabular}[c]{@{}c@{}}\textbf{0.2863}\\ (-0.0678, +0.0376)\end{tabular} & \begin{tabular}[c]{@{}c@{}}\textbf{0.2188}\\ (-0.0010, +0.0010)\end{tabular}  \\ \hline
    \end{tabular}
    }
    \caption{(MDP-4) $\text{ERM}_{\beta}$ of the distribution of discounted costs obtained by different algorithms and $\beta$ values for the MDP-4 environment. Results for $n=5\;000$ MCTS iterations and $K_\text{UCB} = 0.3$. Lower is better. We highlight values whose confidence intervals overlap with the Oracle baseline confidence intervals.}
    \label{table:exp-results-MDP-4}
\end{table}

\begin{figure}[h]
    \centering
    \includegraphics[width=0.8\textwidth]{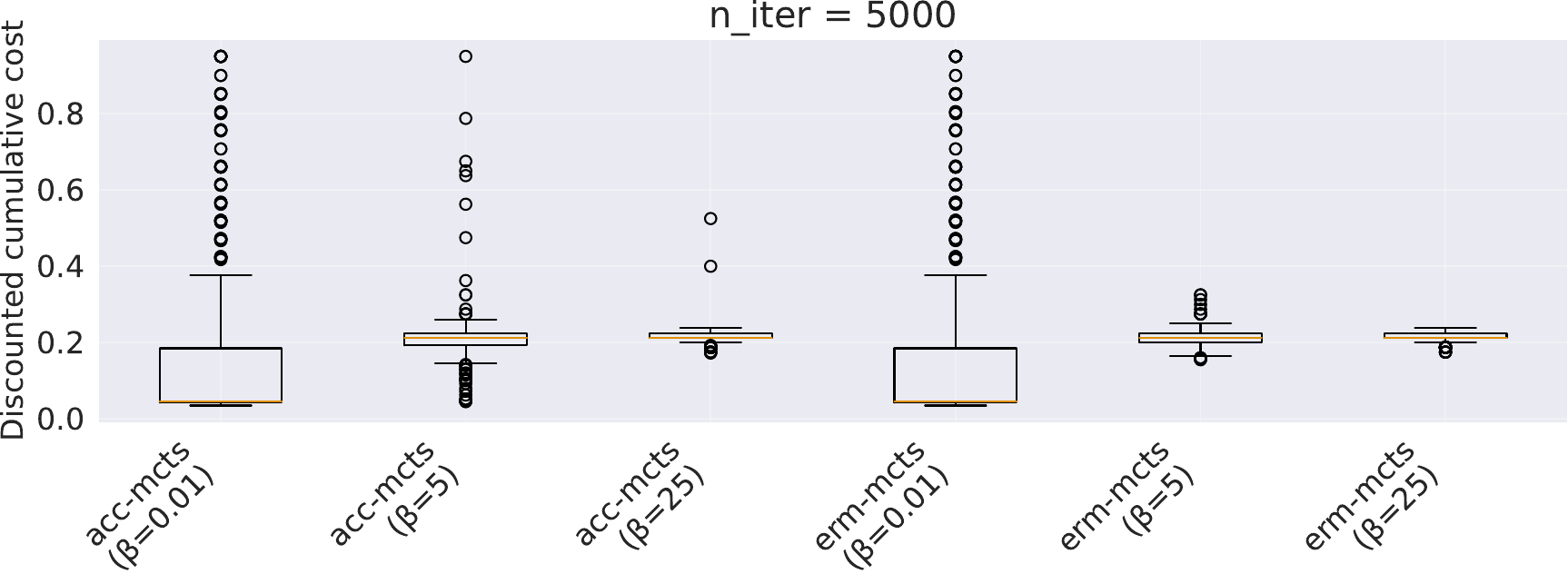}
    \caption{(MDP-4) Comparison of the discounted cumulative cost distributions obtained by different approaches under the MDP-4 environment for multiple $\beta$ values. Results for $n=5\;000$ MCTS iterations and $K_\text{UCB} = 0.3$. Lower is better.}
    \label{fig:dists-comparison-MDP-4}
\end{figure}

\clearpage
\textbf{Other $K_\text{UCB}$ values:}

\begin{figure}[h]
    \centering

    \begin{minipage}{0.32\linewidth}
    \centering
    \includegraphics[height=3.5cm]{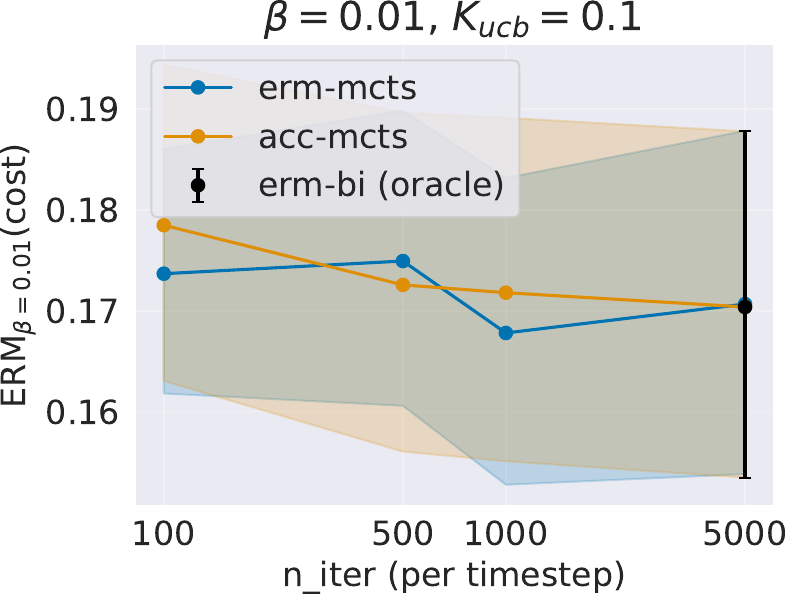}
    \end{minipage}
    \hfill
    \begin{minipage}{0.32\linewidth}
    \centering
    \includegraphics[height=3.5cm]{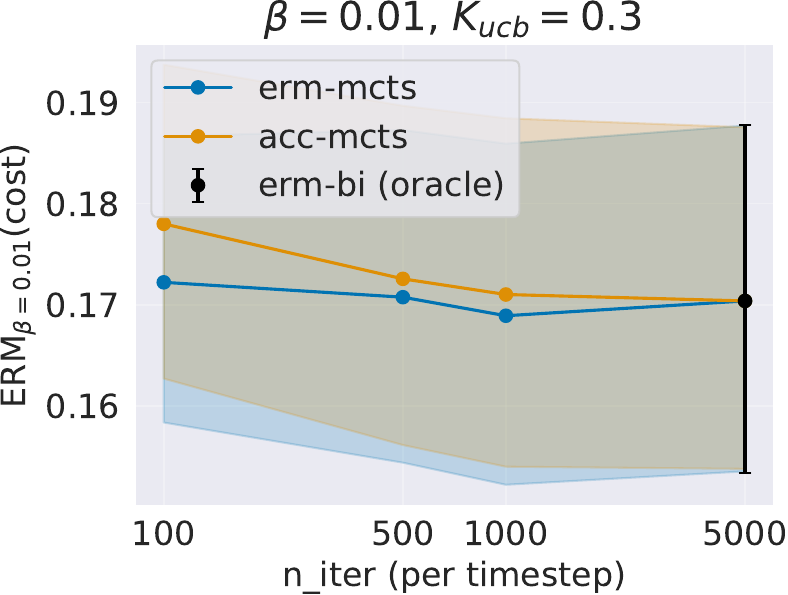}
    \end{minipage}
    \hfill
    \begin{minipage}{0.32\linewidth}
    \centering
    \includegraphics[height=3.5cm]{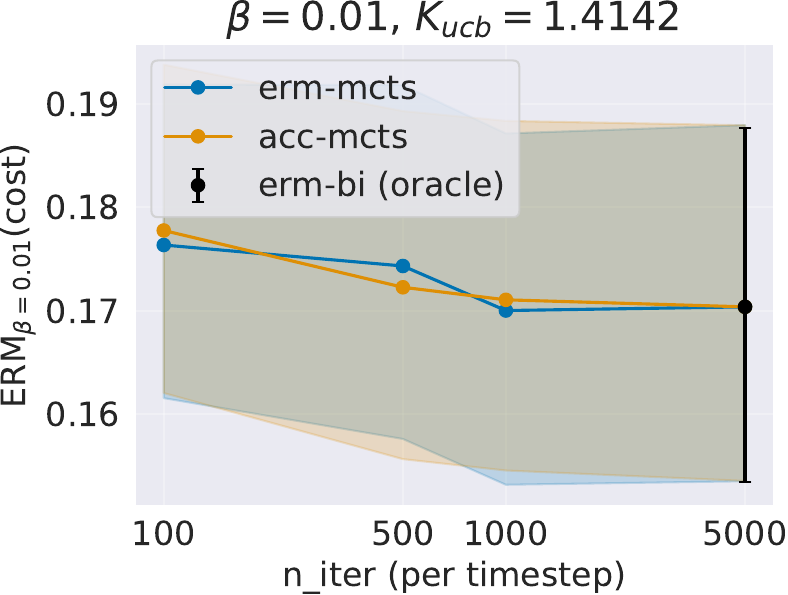}
    \end{minipage}

    \begin{minipage}{0.32\linewidth}
    \centering
    \includegraphics[height=3.5cm]{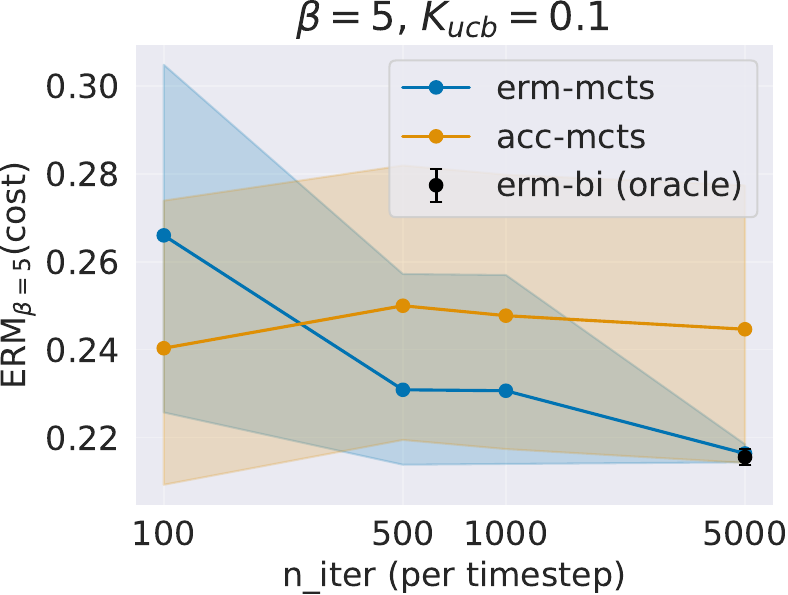}
    \end{minipage}
    \hfill
    \begin{minipage}{0.32\linewidth}
    \centering
    \includegraphics[height=3.5cm]{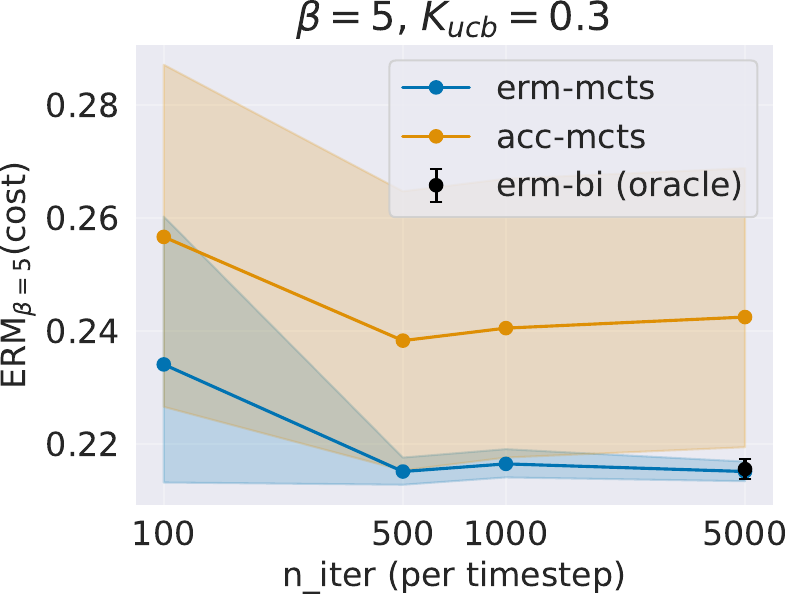}
    \end{minipage}
    \hfill
    \begin{minipage}{0.32\linewidth}
    \centering
    \includegraphics[height=3.5cm]{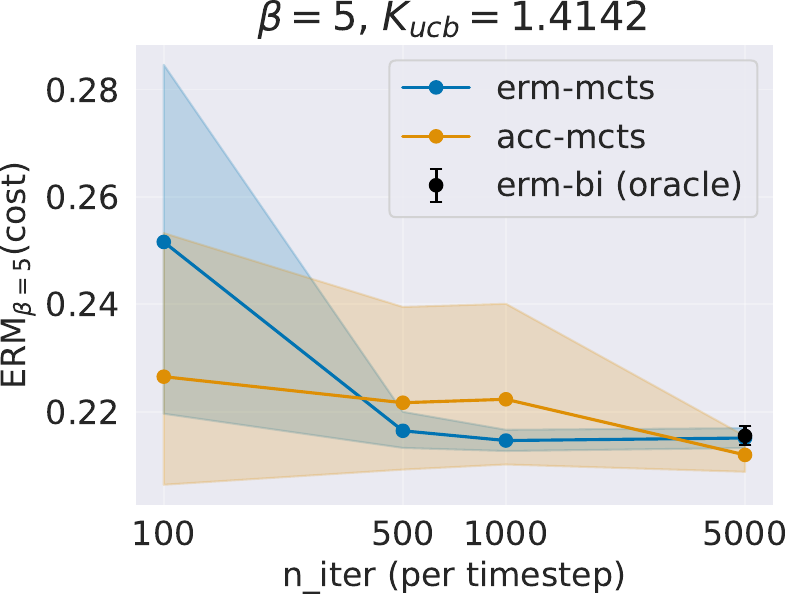}
    \end{minipage}

    \begin{minipage}{0.32\linewidth}
    \centering
    \includegraphics[height=3.5cm]{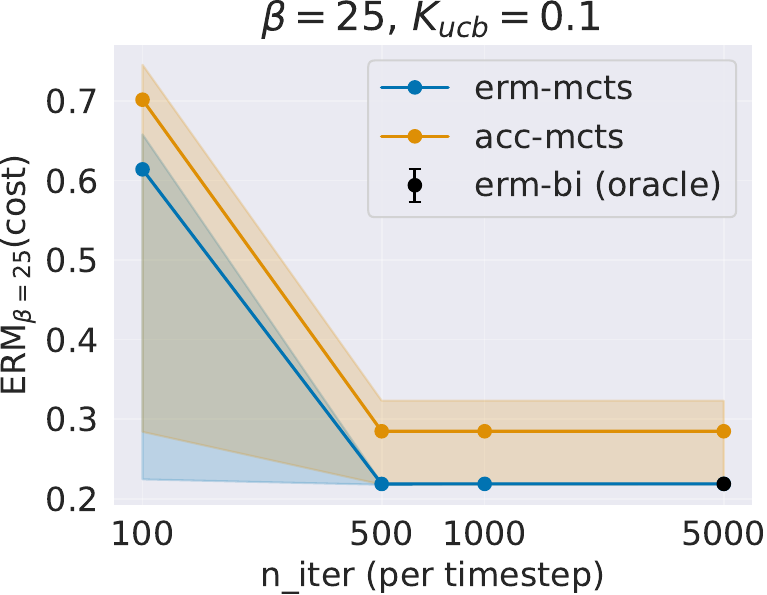}
    \end{minipage}
    \hfill
    \begin{minipage}{0.32\linewidth}
    \centering
    \includegraphics[height=3.5cm]{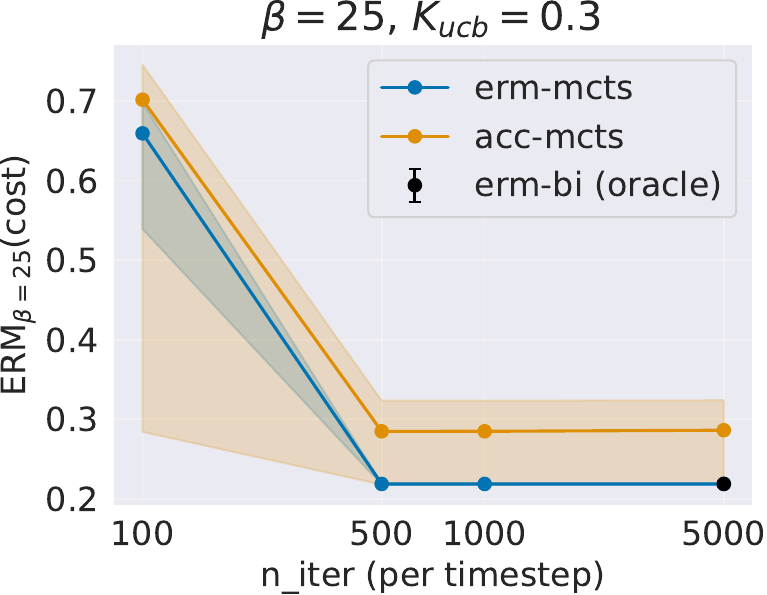}
    \end{minipage}
    \hfill
    \begin{minipage}{0.32\linewidth}
    \centering
    \includegraphics[height=3.5cm]{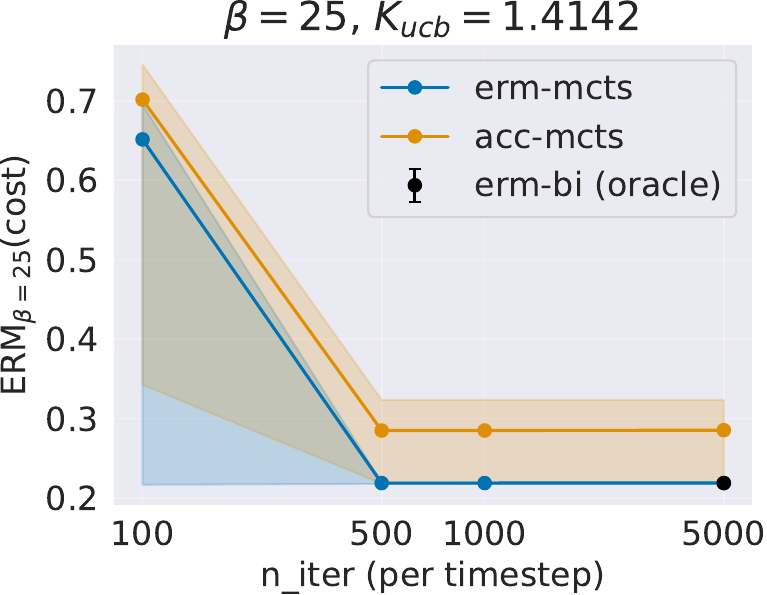}
    \end{minipage}
    
    \caption{(MDP-4) $\text{ERM}_\beta$ for different $\beta$ values, $K_\text{UCB}$ values, and algorithms, as a function of the number of MCTS iterations ($n$).}
\end{figure}

\clearpage
\subsubsection{Grid-MDP}

\textbf{$K_\text{UCB} = 0.3$ (additional results):}

\begin{table}[h]
    \centering
    \resizebox{0.45\columnwidth}{!}{%
    \begin{tabular}{|c|c||c|c|}
    \hline
    $\beta$ & \begin{tabular}[c]{@{}c@{}} \textsc{ERM-BI} \\ (Oracle) \end{tabular} & \textsc{Acc-MCTS}                                                 & \begin{tabular}[c]{@{}c@{}} \textsc{ERM-MCTS} \\ (Ours)\end{tabular} \\ \hline \hline
    1       & \begin{tabular}[c]{@{}c@{}}0.0821\\ (-0.0124, +0.0139) \end{tabular}  & \begin{tabular}[c]{@{}c@{}}\textbf{0.0805}\\ (-0.0115, +0.0127) \end{tabular} & \begin{tabular}[c]{@{}c@{}}\textbf{0.0862}\\ (-0.0121, +0.0131)\end{tabular}  \\ \hline
    5       & \begin{tabular}[c]{@{}c@{}}0.1245\\ (-0.0105, +0.0132) \end{tabular}  & \begin{tabular}[c]{@{}c@{}}\textbf{0.1509}\\ (-0.0159, +0.0167)\end{tabular} & \begin{tabular}[c]{@{}c@{}}\textbf{0.1252}\\ (-0.0106, +0.0131)\end{tabular}  \\ \hline
    25      & \begin{tabular}[c]{@{}c@{}}0.1829\\ (-0.0055, +0.0063)\end{tabular}  & \begin{tabular}[c]{@{}c@{}}\textbf{0.1864}\\ (-0.0051, +0.0052)\end{tabular} & \begin{tabular}[c]{@{}c@{}}\textbf{0.1830}\\ (-0.0037, +0.0043)\end{tabular}  \\ \hline
    \end{tabular}
    }
    \caption{(Grid-MDP) $\text{ERM}_{\beta}$ of the distribution of discounted costs obtained by different algorithms and $\beta$ values for the Grid-MDP environment. Results for $n=5\;000$ MCTS iterations and $K_\text{UCB} = 0.3$. Lower is better. We highlight values whose confidence intervals overlap with the Oracle baseline confidence intervals.}
    \label{table:exp-results-Grid-MDP}
\end{table}

\begin{figure}[h]
    \centering
    \includegraphics[width=0.8\textwidth]{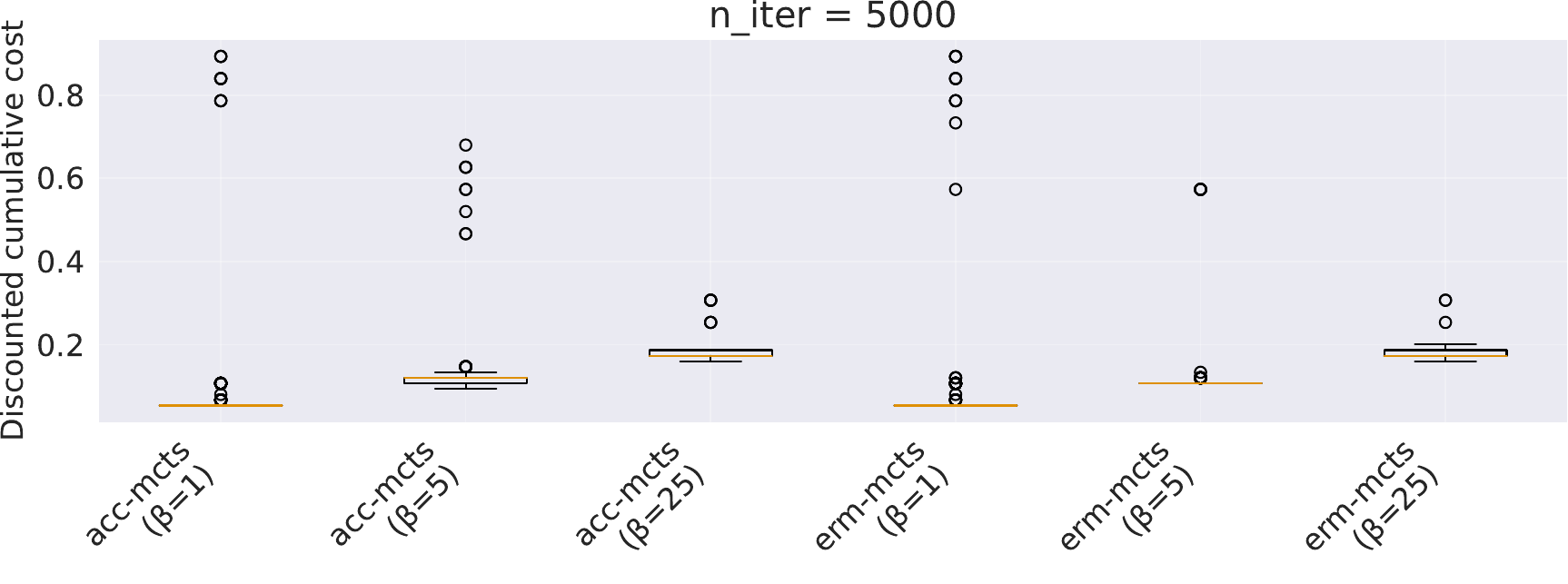}
    \caption{(Grid-MDP) Comparison of the discounted cumulative cost distributions obtained by different approaches under the Grid-MDP environment for two $\beta$ values. Results for $n=5\;000$ MCTS iterations and $K_\text{UCB} = 0.3$. Lower is better.}
    \label{fig:dists-comparison-Grid-MDP}
\end{figure}

\clearpage
\textbf{Other $K_\text{UCB}$ values:}

\begin{figure}[h]
    \centering

    \begin{minipage}{0.32\linewidth}
    \centering
    \includegraphics[height=3.5cm]{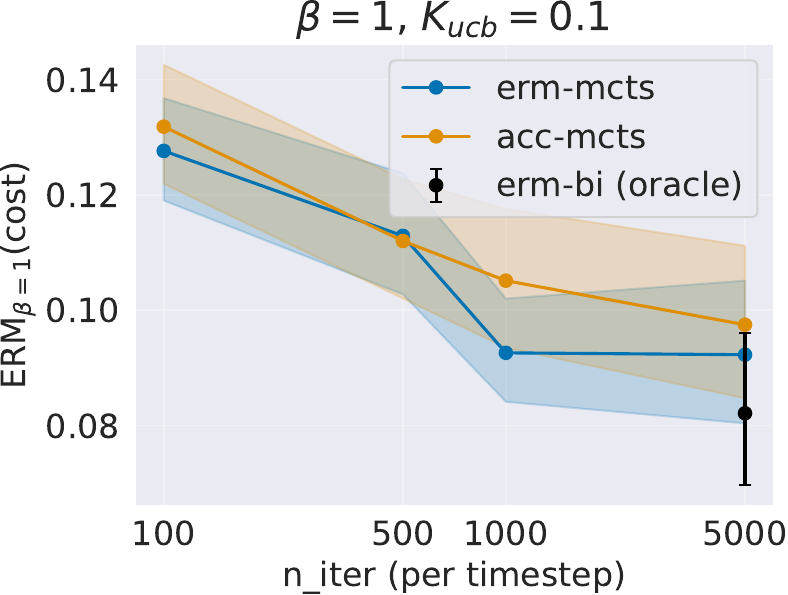}
    \end{minipage}
    \hfill
    \begin{minipage}{0.32\linewidth}
    \centering
    \includegraphics[height=3.5cm]{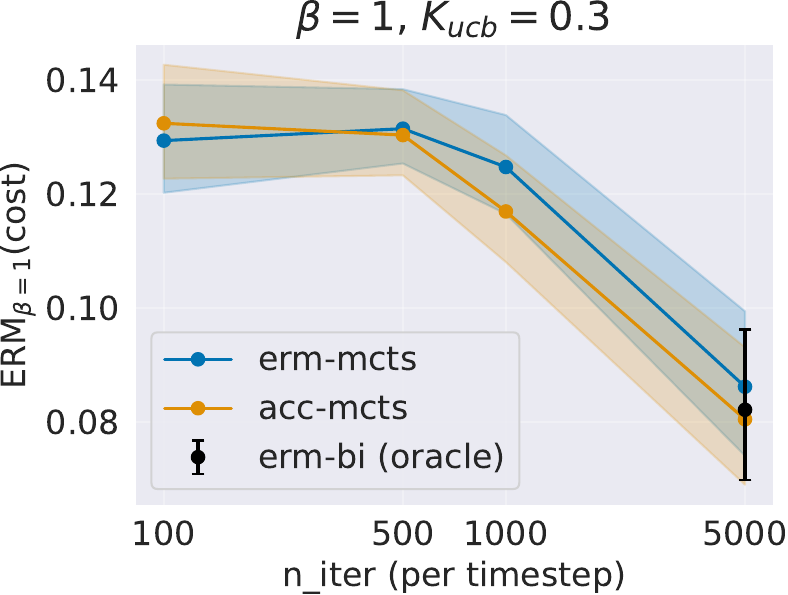}
    \end{minipage}
    \hfill
    \begin{minipage}{0.32\linewidth}
    \centering
    \includegraphics[height=3.5cm]{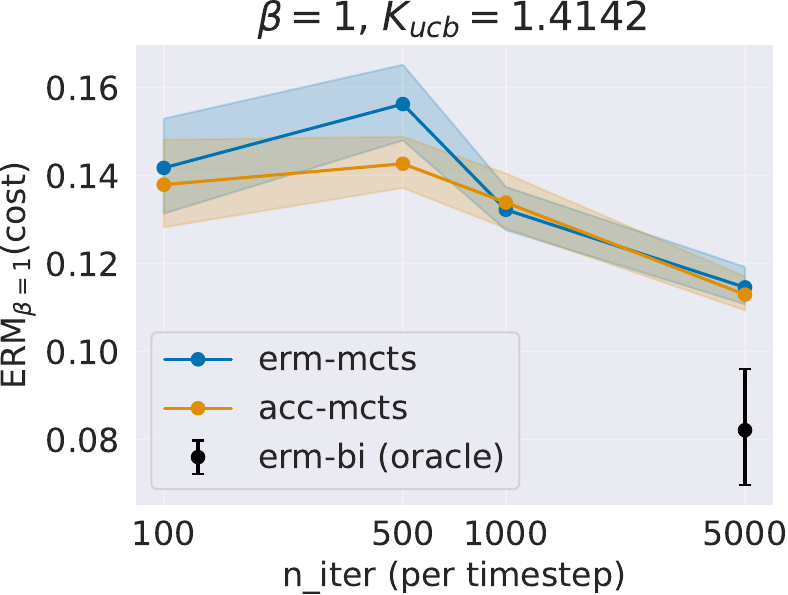}
    \end{minipage}

    \begin{minipage}{0.32\linewidth}
    \centering
    \includegraphics[height=3.5cm]{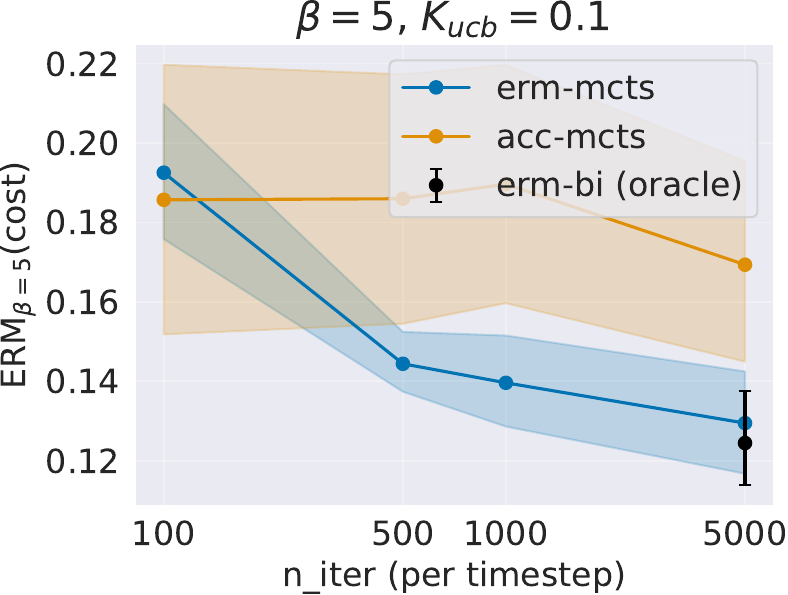}
    \end{minipage}
    \hfill
    \begin{minipage}{0.32\linewidth}
    \centering
    \includegraphics[height=3.5cm]{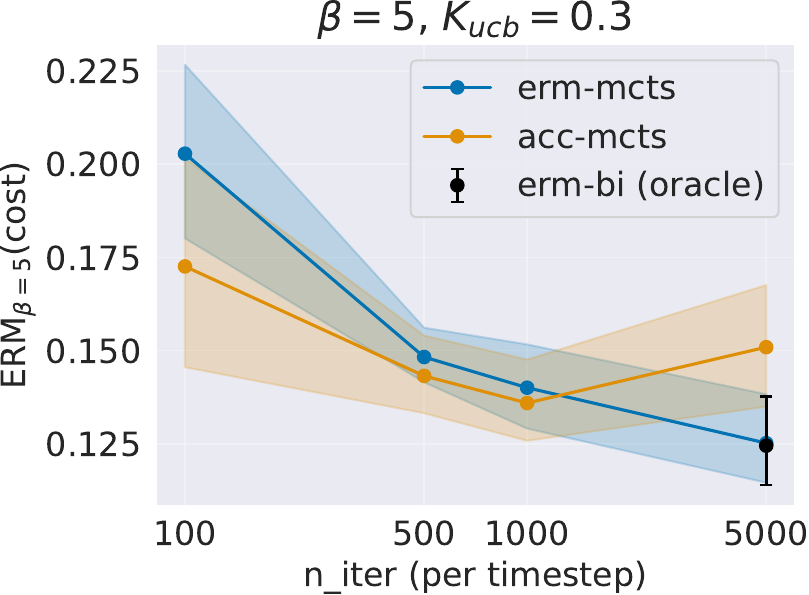}
    \end{minipage}
    \hfill
    \begin{minipage}{0.32\linewidth}
    \centering
    \includegraphics[height=3.5cm]{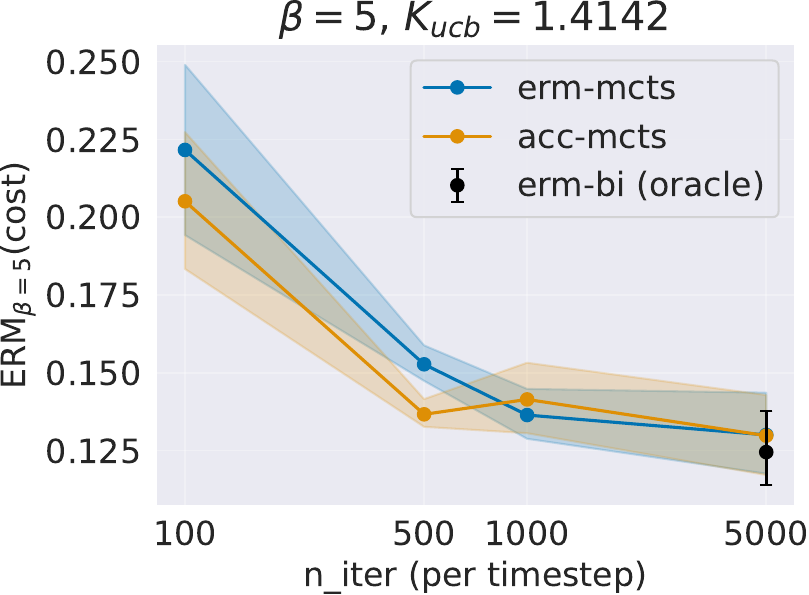}
    \end{minipage}

    \begin{minipage}{0.32\linewidth}
    \centering
    \includegraphics[height=3.5cm]{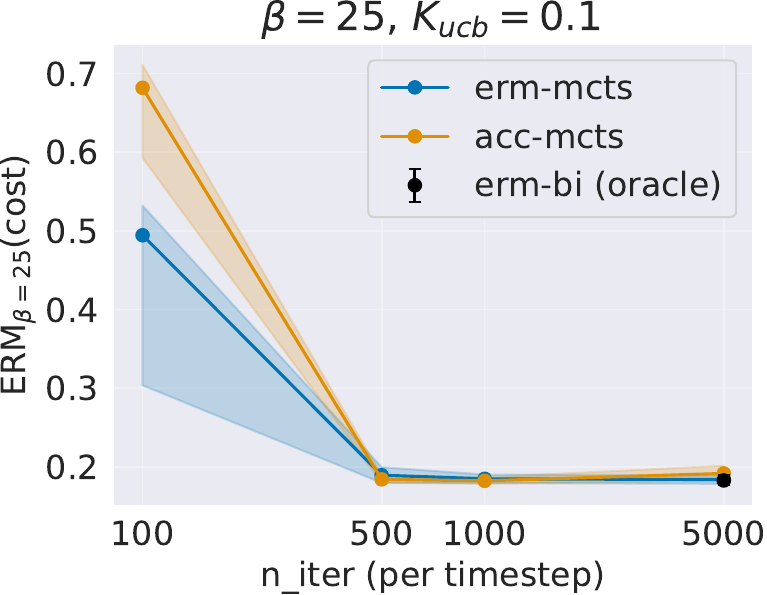}
    \end{minipage}
    \hfill
    \begin{minipage}{0.32\linewidth}
    \centering
    \includegraphics[height=3.5cm]{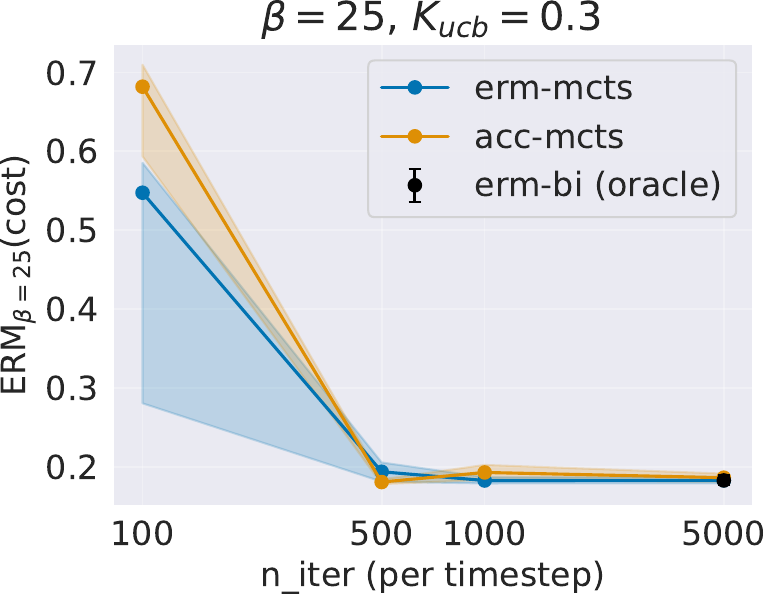}
    \end{minipage}
    \hfill
    \begin{minipage}{0.32\linewidth}
    \centering
    \includegraphics[height=3.5cm]{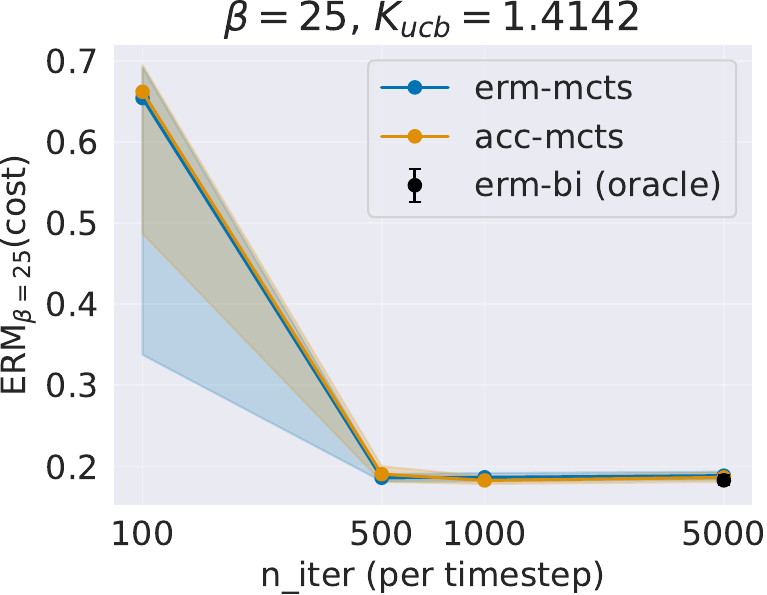}
    \end{minipage}
    
    \caption{(Grid-MDP) $\text{ERM}_\beta$ for different $\beta$ values, $K_\text{UCB}$ values, and algorithms, as a function of the number of MCTS iterations ($n$).}
\end{figure}

\clearpage
\subsubsection{Bin-Packing}

\textbf{$K_\text{UCB} = 0.3$ (additional results):}

\begin{table}[h]
    \centering
    \resizebox{0.45\columnwidth}{!}{%
    \begin{tabular}{|c|c|c|}
    \hline
    $\beta$ & \textsc{Acc-MCTS}                                                 & \begin{tabular}[c]{@{}c@{}} \textsc{ERM-MCTS} \\ (Ours)\end{tabular} \\ \hline \hline
    1       & \begin{tabular}[c]{@{}c@{}}\textbf{0.0532}\\ (-0.0020, +0.0020) \end{tabular} & \begin{tabular}[c]{@{}c@{}}\textbf{0.0534}\\ (-0.0019, +0.0020)\end{tabular}  \\ \hline
    25      & \begin{tabular}[c]{@{}c@{}}0.0718\\ (-0.0015, +0.0016)\end{tabular} & \begin{tabular}[c]{@{}c@{}}\textbf{0.0683}\\ (-0.0019, +0.0019)\end{tabular}  \\ \hline
    \end{tabular}
    }
    \caption{(Bin-Packing) $\text{ERM}_{\beta}$ of the distribution of discounted costs obtained by different algorithms and $\beta$ values for the Bin-Packing environment. Results for $n=5000$ MCTS iterations and $K_\text{UCB} = 0.3$. Lower is better. Best values are highlighted.}
    \label{table:exp-results-binpacking}
\end{table}

\begin{figure}[h]
    \centering
    \includegraphics[width=0.6\textwidth]{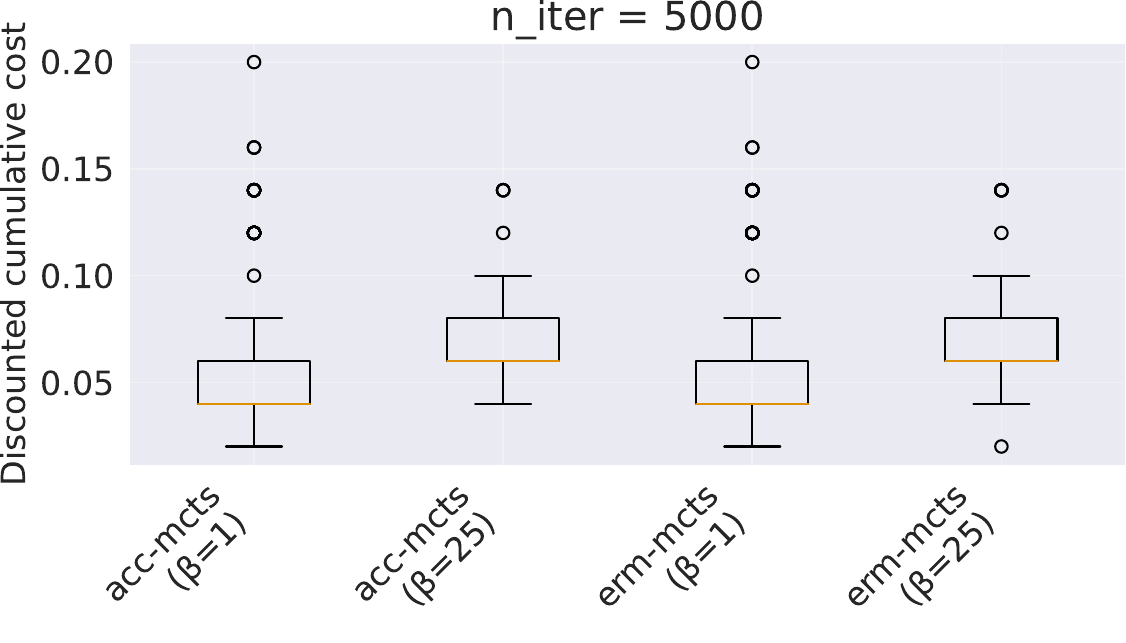}
    \caption{(Bin-Packing) Comparison of the discounted cumulative cost distributions obtained by different approaches under the Bin-Packing environment for multiple $\beta$ values. Results for $n=5000$ MCTS iterations and $K_\text{UCB} = 0.3$. Lower is better.}
    \label{fig:dists-comparison-binpacking}
\end{figure}

\clearpage
\textbf{Other $K_\text{UCB}$ values:}

\begin{figure}[h]
    \centering

    \begin{minipage}{0.32\linewidth}
    \centering
    \includegraphics[height=3.5cm]{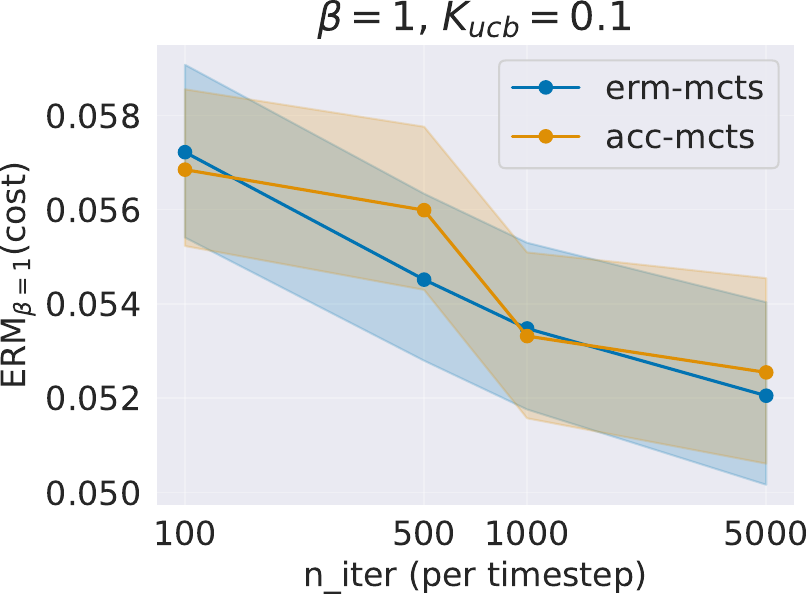}
    \end{minipage}
    \hfill
    \begin{minipage}{0.32\linewidth}
    \centering
    \includegraphics[height=3.5cm]{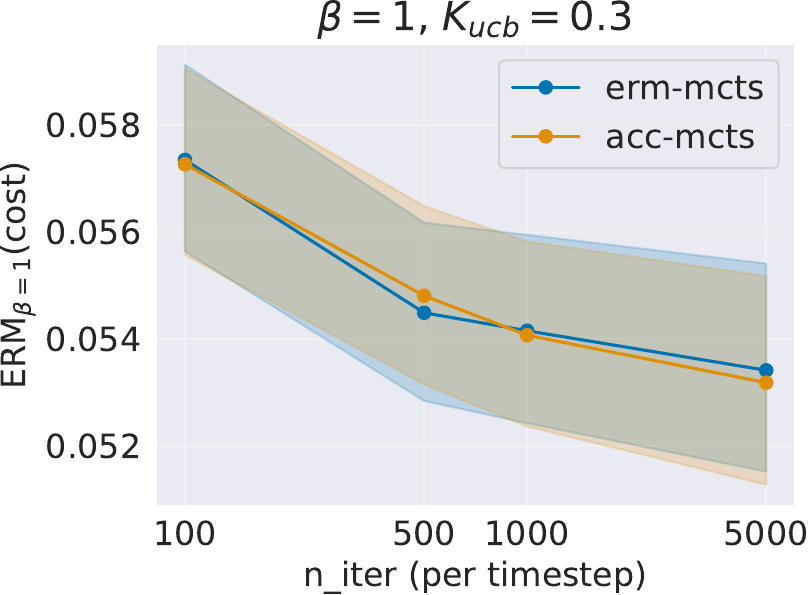}
    \end{minipage}
    \hfill
    \begin{minipage}{0.32\linewidth}
    \centering
    \includegraphics[height=3.5cm]{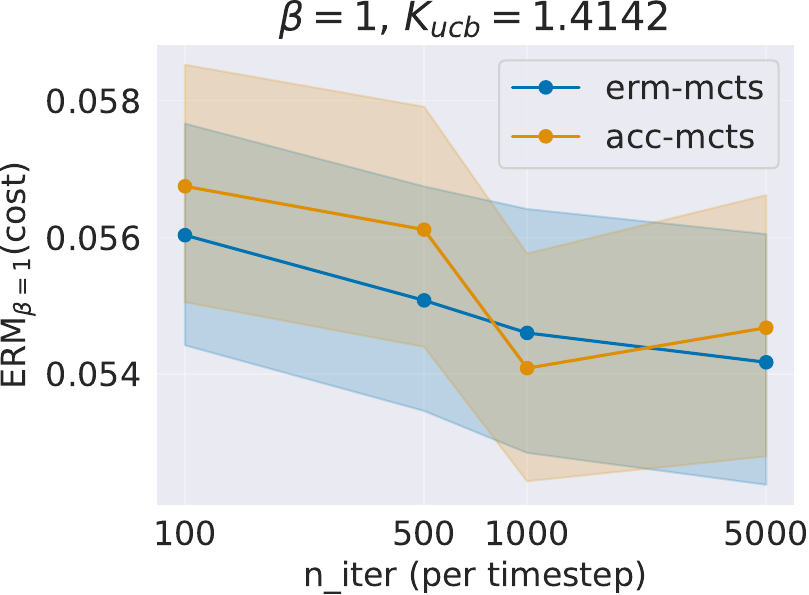}
    \end{minipage}

    \begin{minipage}{0.32\linewidth}
    \centering
    \includegraphics[height=3.5cm]{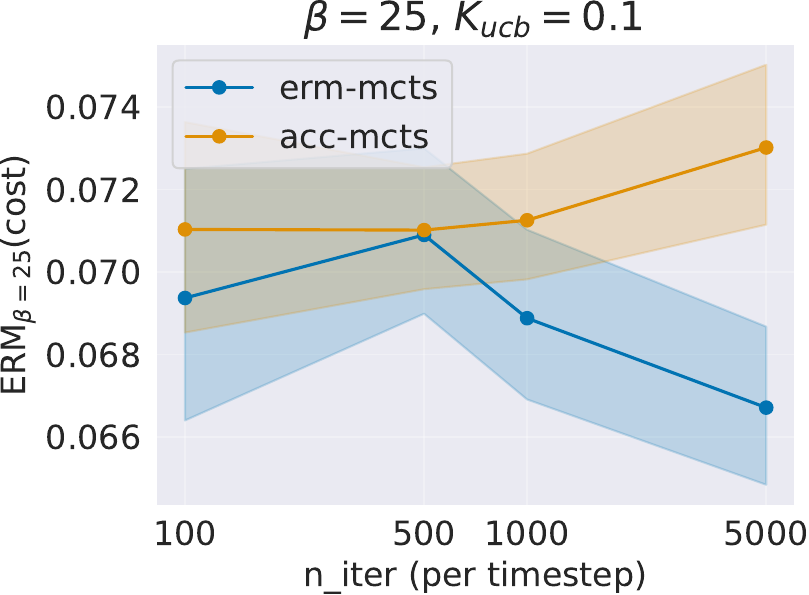}
    \end{minipage}
    \hfill
    \begin{minipage}{0.32\linewidth}
    \centering
    \includegraphics[height=3.5cm]{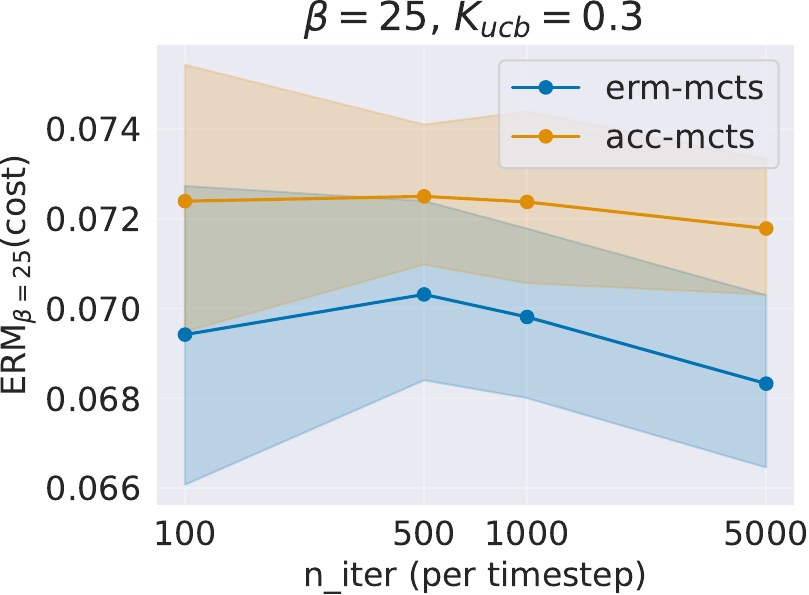}
    \end{minipage}
    \hfill
    \begin{minipage}{0.32\linewidth}
    \centering
    \includegraphics[height=3.5cm]{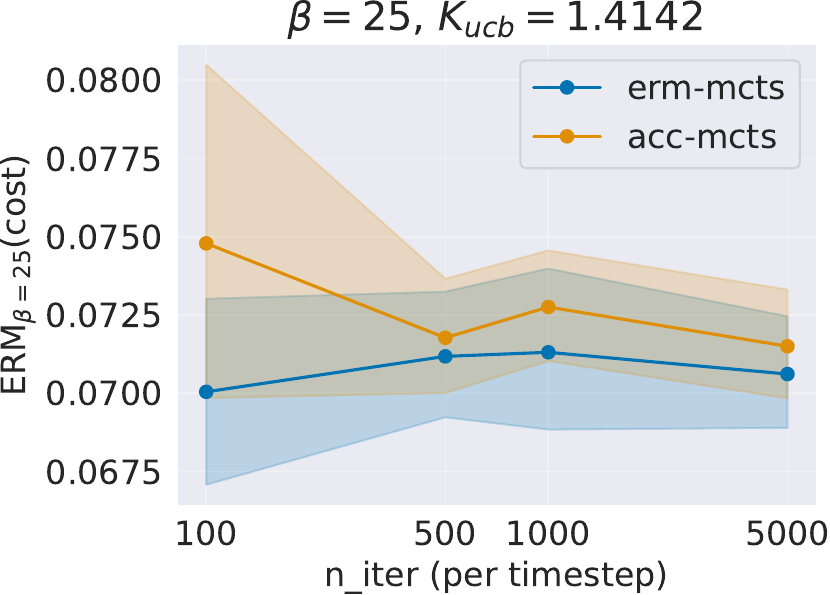}
    \end{minipage}
    
    \caption{(Bin-Packing) $\text{ERM}_\beta$ for different $\beta$ values, $K_\text{UCB}$ values, and algorithms, as a function of the number of MCTS iterations ($n$).}
\end{figure}

\end{document}